\documentclass{article}

\PassOptionsToPackage{numbers, compress, sort}{natbib}

\usepackage[final]{neurips_2025}

\usepackage{amsthm}
\usepackage{amsmath}
\usepackage{bbm}
\usepackage[utf8]{inputenc} % allow utf-8 input
\usepackage[T1]{fontenc}    % use 8-bit T1 fonts
\usepackage{hyperref}       % hyperlinks
\usepackage{url}            % simple URL typesetting
\usepackage{booktabs}       % professional-quality tables
\usepackage{amsfonts}       % blackboard math symbols
\usepackage{nicefrac}       % compact symbols for 1/2, etc.
\usepackage[british]{babel}
\usepackage{microtype}      % microtypography
\usepackage[table,xcdraw]{xcolor}
\usepackage{xcolor}         % colors
\usepackage{graphicx}
\usepackage[capitalize, nameinlink]{cleveref} 
\usepackage{paralist}
\usepackage{subcaption}
\usepackage{comment}
\usepackage{wrapfig}
\usepackage{svg}
\usepackage{algorithm}
\usepackage{algorithmic}
\usepackage{enumitem}
\usepackage{dsfont}
\usepackage{thmtools} 
\usepackage{thm-restate}
\usepackage[symbol]{footmisc}
\usepackage{wrapfig}
\usepackage{tcolorbox}
\usepackage[svgnames]{xcolor}

\usepackage{todonotes}
\usepackage{titletoc}

\usepackage[noblocks]{authblk}

\hypersetup{
    colorlinks = true,
    urlcolor   = RoyalBlue,
    linkcolor  = RoyalBlue,
    citecolor  = RoyalBlue
}

\renewcommand\Affilfont{\small\normalfont}

\makeatletter
\renewcommand\AB@affilsepx{\hspace{40pt} \protect\Affilfont}
\makeatother

\crefname{app}{Appendix}{Appendices}
\Crefname{app}{Appendix}{Appendices}

\definecolor{tab10-0}{RGB}{31,119,180}
\definecolor{tab10-1}{RGB}{255,127,14}
\definecolor{tab10-2}{RGB}{44,160,44}
\definecolor{tab10-3}{RGB}{214,39,40}
\definecolor{tab10-4}{RGB}{148,103,189}
\definecolor{tab10-5}{RGB}{140,86,75}
\definecolor{tab10-6}{RGB}{227,119,194}
\definecolor{tab10-7}{RGB}{127,127,127}
\definecolor{tab10-8}{RGB}{188,189,34}
\definecolor{tab10-9}{RGB}{23,190,207}

\newcommand{\kat}[1]{{\bf\textcolor{violet}{ Kat: #1}}}

\newcommand{\norm}[1]{\left\lVert#1\right\rVert}

\title{Geometry-Aware Edge Pooling\\ for Graph Neural Networks}

\renewcommand{\thefootnote}{\fnsymbol{footnote}}

\author[*,1,2]{Katharina Limbeck}
\author[*,3,4]{Lydia Mezrag}
\author[$\dagger$,3,4]{Guy Wolf}
\author[$\dagger$,1,5]{Bastian Rieck}

\affil[1]{Helmholtz Munich}
\affil[2]{Technical University of Munich} 
\affil[3]{Université de Montréal}
\affil[4]{Mila - Quebec AI Institute}
\affil[5]{Université de Fribourg}

\begin{document}

\footnotetext[1]{These authors contributed equally to this work.}
\footnotetext[2]{These authors jointly supervised this work.}

\maketitle
%--------------------------------------------------------------
\newtheorem{thm}{Theorem}
\newtheorem{lem}{Lemma}
\newtheorem{prop}{Proposition}
\newtheorem{cor}{Corollary}
\newtheorem{definition}{Definition}
\newtheorem{example}{Example}
\newtheorem{rk}{Remark}
\newtheorem{claim}{Claim}
%---------------------------------------------------------------
\begin{abstract}

  Graph Neural Networks (GNNs) have shown significant success for graph-based tasks. Motivated by the prevalence of large  datasets in real-world applications, 
  pooling layers are crucial components of GNNs. 
  By reducing the size of input graphs, pooling enables faster training and potentially better generalisation. 
  However, existing pooling operations often optimise for the learning task at the expense of discarding fundamental graph structures, thus reducing interpretability. 
  This leads to unreliable performance across dataset types, downstream tasks and pooling ratios. Addressing these concerns, we propose novel graph pooling layers for structure-aware pooling via edge collapses. Our methods leverage diffusion geometry and iteratively reduce a graph’s size while preserving both its metric structure and its structural diversity. We guide pooling using \emph{magnitude}, an isometry-invariant diversity measure, which permits us to control the fidelity of the pooling process. 
  Further, we use the \emph{spread} of a metric space as a faster and more stable alternative ensuring computational efficiency. 
  Empirical results demonstrate that our methods \begin{inparaenum}[(i)]
      \item achieve top performance compared to alternative pooling layers across a range of diverse graph classification tasks, 
      \item preserve key spectral properties of the input graphs, and
      \item retain high accuracy across varying pooling ratios.
  \end{inparaenum}
\end{abstract}

\setcounter{footnote}{0}
\renewcommand{\thefootnote}{\arabic{footnote}}

\section{Introduction} 

Graph pooling layers are important components of GNN architectures. They are implemented alongside convolutional layers to reduce the size of graph representations during training. Pooling thus enables GNNs to scale to large and complex real-world graphs while regularising the resulting representation. However, the choice of pooling method strongly influences downstream-applications and task-performance. In fact, the question of which graph properties to preserve during pooling, just as the question on the nature and quality of graph datasets in graph learning \citep{coupette2025no}, remains an ongoing debate~\citep{ying2024boosting, liu2023graph, mesquita2020rethinking}.
It is thus crucial to design expressive, efficient, and interpretable pooling layers that are capable of reliably encoding task-relevant information %on a node features as well as graph structure 
while reducing the size of input graphs. Most graph pooling literature takes a node-centric view \citep{liu2023graph}. However, this focus on \emph{node-centric} rather than \emph{edge-centric} pooling  often leads to the loss of important structural information. Common pooling methods either drop nodes or optimise for a node clustering while treating graph connectivities as a secondary objective. As visualised in \Cref{fig:overview_graphs} and further explored in our work, this frequently leads to counter-intuitive pooling decisions that fail to retain key geometric structures in a graph. 
Addressing these concerns, topological and geometric descriptors of graphs are \emph{uniquely} poised to interoperate structural information into graph pooling. 

Throughout this work, we treat graphs as \emph{metric spaces} and assess their geometry via diffusion distances, which naturally work alongside message passing to effectively encode key graph structures. Motivated by ongoing research on novel geometric invariants, %for deep learning applications, 
we find that generalised measures of size and diversity are especially promising candidates for guiding graph pooling. 
%towards better preserving a graph's inherent geometry. 
In particular, we use the \emph{magnitude} of a metric space~\citep{leinster2021entropy}, which measures a graph's structural diversity, to control the loss of structural information during edge pooling.
%Magnitude has %strong mathematical roots, which offer promising insights into its capabilities to summarise the geometry of graphs and metric spaces %roots in 
%Most recently, magnitude has found successful 
Our work is motivated by successful applications of magnitude across a range of machine learning tasks, such as the evaluation of diversity~\citep{limbeck2024metric} for latent spaces, boundary detection for images~\citep{adamer2024magnitude}, and the study of the generalisation behaviour of neural networks~\citep{andreeva2023metric, andreeva2024topological}. Building up on this research, we are the first to propose the use of magnitude in the context of graph learning. We further advance on existing applications by investigating an alternative and closely-related measure, known as the \emph{spread} of a metric space~\citep{willerton2015spread}, to substantially improve the computational efficiency of our methods. 
Our main \textbf{contributions} are as follows:
\begin{itemize}[noitemsep, leftmargin=1em, topsep=0pt]
    \item 
    We propose MagEdgePool and SpreadEdgePool, two novel edge-contraction based pooling layers, which preserve graphs' structural diversity. 
\item 
We investigate the spread of a metric space as a faster and more stable alternative to magnitude  
and ensure that our algorithms can be computed efficiently. 
\item 
We evaluate our methods' capability to preserve key structural properties during pooling. 
\item 
We demonstrate that 
our pooling methods \emph{consistently} perform well in graph classification tasks, achieving top accuracies among other pooling layers across a wide range of experimental setups.
\end{itemize}

\section{Background} 
\label{sec:background}

We briefly provide background information on graph pooling, the magnitude of metric spaces, and diffusion geometry, taking care
to point out related work and how it differs from ours.

\subsection{Graph Pooling for Graph Neural Networks}

\begin{figure}[tbp]
  \centering
  \includegraphics[width=\linewidth]{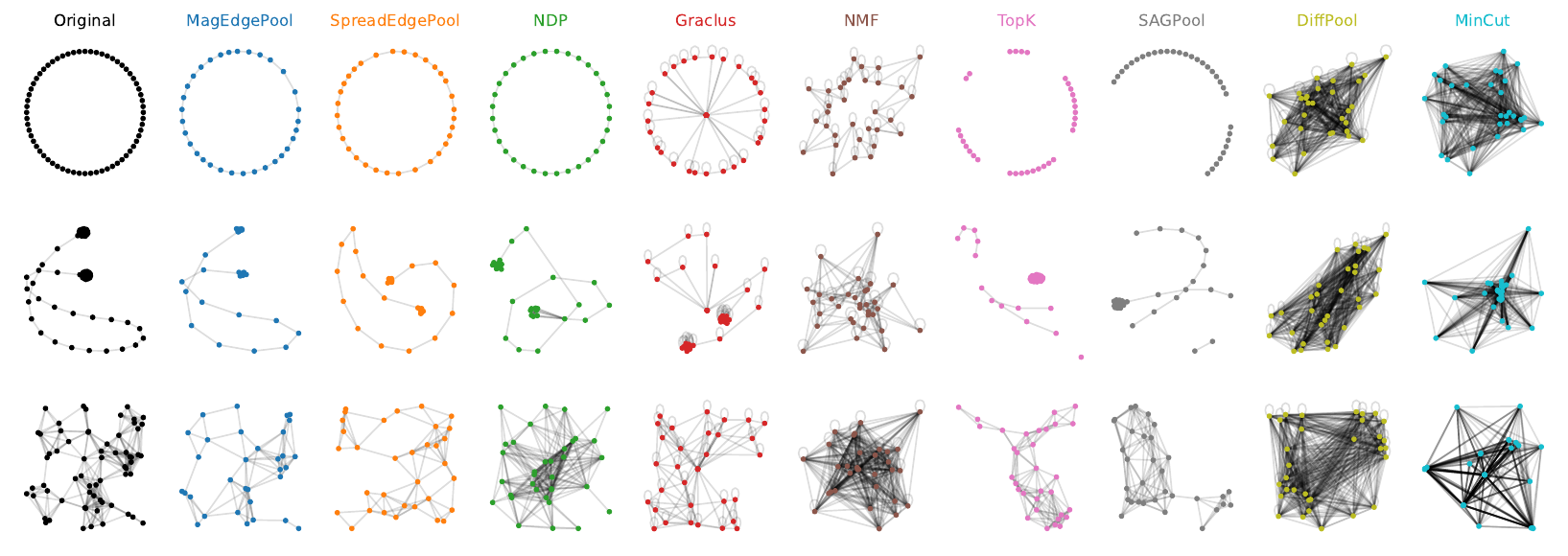}
  \caption{Examples of graphs pooled to approximately half their original size compared across pooling layers. Our proposed  methods, MagEdgePool and SpreadEdgePool, respect the original graphs' geometry during pooling. Alternative approaches tend to obscure adjacency relationships to varying extents by creating counter-intuitive edges (Graclus, NMF), disconnecting entire portions of the graphs (TopK, SAGPool), or returning dense representations that do not preserve any geometric structure (DiffPool, MinCut).}\label{fig:overview_graphs}
\end{figure} %\todo{Mention they are pooled to approx half their size and double-check node numbers.} 

As an ongoing field of interest for graph learning, a wide range of pooling methods has been proposed, which 
can be divided into \emph{global} and \emph{hierarchical} approaches. Global and hierarchical pooling methods fulfil fundamentally different but complementary roles as distinct components of GNN architectures \citep{grattarola2022understanding}. Global or flat pooling methods generate graph-level representations by reducing each graph to a single node and are commonly used as readout operations. Examples include mean-pool or sum-pool, which average or sum all node features, %or more complex global methods such as 
SOPool \citep{wang2024graph}, which uses second-order feature information, or DKEPool \citep{chen2022distribution}, which learns node distributions. However, %by pooling to a single graph-level representation, 
global pooling methods %ignore the inherent hierarchical structure of graphs and 
discard all topological information, which can reduce the expressivity of these models and decrease their performance \citep{knyazev2019understanding, bianchi2024expressive, grattarola2022understanding}. By contrast, hierarchical pooling methods sequentially coarsen graph representations while reducing their size, %. Hierarchical pooling layers 
%output consecutively smaller graphs and 
and 
are frequently used in alternation with intermediate convolutional or message passing layers \citep{liu2023graph}. %usually by using \emph{node-based} pooling via node clustering or node dropping.}   
%Global and hierarchical pooling methods thus fulfil different but complementary roles as distinct components of GNN architectures. %for graph-level tasks. 
%Hierarchical pooling methods 
By reducing both features and adjacencies, hierarchical methods have the ability %and 
%have the ability 
to preserve graphs' inherent geometry allowing GNNs to learn across multiple coarsening levels. %can more expressively encoded. 
%and allow the GNN to learn from across multiple levels of coarsening. %and are capable of preserving its inherent structure. %graphs' inherent geometry and connectivities while coarsening the graphs. 
%We further investigate hierarchical methods in particular motivated by their ability to preserve graphs' inherent geometry.
%} 
%Flat pooling acts as a , while .
%As such, hierarchical pooling methods fulfil a fundamentally different role to global pooling approaches, and are .
%The most successful of these methods often use a hierarchical approach to sequentially coarsen the graph representation while reducing its size \citep{liu2023graph}. \todo{Mention global pooling methods and contrast them to hierarchical pooling.} 

Research has traditionally focused on \emph{node-based} hierarchical pooling via node clustering or node dropping. 
Amongst node drop pooling methods, Node Decimal Pooling (NDP) \citep{bianchi2020hierarchical} is non-trainable, while 
TopK \citep{gao2019graph, cangea2018towards} and SAGPool \citep{lee2019self}, are trainable approaches. However, all these methods  inevitably lose information because they do not select all vertices but remove entire sections of the graph during pooling \citep{liu2023graph}. 
Node clustering approaches are similarly either non-trainable, such as 
Graclus \citep{dhillon2007weighted} and Non-Negative Matrix Factorization (NMF) \citep{bacciu2019non}, or trainable, such as MinCUT \citep{bianchi2020spectral} and DiffPool \citep{ying2018hierarchical}, which output dense graph representations. 
In comparison to node drop pooling, clustering-based methods usually require high memory costs \citep{liu2023graph}. Overall, trainable methods have the potential to better optimise for a specific objective, but  might overfit to the task at hand, especially for smaller datasets. By contrast, non-trainable methods can act as a stronger inductive bias on the underlying graph representation, and do not introduce additional trainable parameters or optimisation objectives into the GNN  \citep{grattarola2022understanding}. 
\emph{Edge-based} pooling methods have been studied less extensively than node-centric pooling 
\citep{diehl2019edge, landolfi2022revisiting} despite their promising potential for encoding connectivities in a more faithful and natural manner. EdgePool \citep{diehl2019edge}, the most successful edge-based method,  uses iterative edge-contraction based on edge scores, which are learnt from features of adjacent nodes. However, EdgePool always pools graphs to half their size, making it less flexible than more adaptive methods. Moreover, learning the edge scores during training is computationally expensive %both in terms of runtimes and memory usage 
\citep{bianchi2024expressive} necessitating the development of faster and more efficient %potentially non-trainable 
edge pooling approaches. Further, EdgePool does not explicitly consider a graph's topology, beyond posing a constraint on not contracting adjacent edges. Thus, selecting edges based on node features can lead to counter-intuitive decisions during pooling, for example by retaining local structures in strongly-connected communities even at high pooling ratios \citep{diehl2019edge}. %\todo{Highlight computational costs and clarify our positioning to EdgePool.}\textcolor{blue}{Moreover, we find that there is no significant improvement in performance when comparing our methods to EdgePool (see \Cref{app:comparison_EdgePool}) and that our methods have lower computational costs during training (see \Cref{app:empirical_efficiency}).} 
%Moreover, 
%Further, there is evidence that the first convolutional layers in GNN already tend to learn homogeneous feature representation, which calls into question the goal of including a local feature-based loss in pooling architectures as \citet{mesquita2020rethinking}
%Further, initial convolutional layers in GNNs can already learn homogenous feature representations %from the input graphs 
%\citep{mesquita2020rethinking}, which reduces the utility of feature-based losses to enforce local clustering, especially on datasets with few or no features. %Rather, 
%To deal with these data, w 
Interpretability %in the pooling operation 
remains a leading concern \citep{liu2023graph}, and non-trainable approaches are capable of being well-performing and traceable  baseline pooling methods \citep{grattarola2022understanding}.
We therefore find that there is strong potential for designing geometry-aware edge pooling operations that make interpretable decisions on which aspects of the graph to retain. 
\Cref{fig:overview_graphs} further illustrates how many of the aforementioned pooling methods inherently fail to preserve key graph structures even for simple toy examples. Trainable methods in this overview were optimised for a spectral loss following \citet{grattarola2022understanding}, but this does \emph{not} ensure  interpretable preservation of graphs' underlying geometry. Addressing these shortcomings, it is of interest to leverage alternative tools from computational geometry, which allow us to quantify the qualitative difference between  graphs. %, out of which many have not yet been explored,
While pooling based on %guiding pooling by %graphs' 
spectral properties has been investigated extensively \citep{bianchi2020hierarchical, grattarola2022understanding}, alternative %topological or 
geometric %information %and alternative geometry properties and 
invariants like  curvature or persistent homology
have only been explored more recently %in the context of graph pooling 
\citep{ying2024boosting, feng2024graph},  and %even less extensively for 
have 
not yet been applied to edge pooling specifically.

\subsection{The Magnitude and Spread of Metric Spaces}
\label{sec:mag_spread}

Magnitude is an invariant of (finite) metric spaces that measures the `effective size' of a space. It is a measure of entropy and diversity \citep{leinster2021entropy} %% 
that has first been proposed in theoretical ecology \citep{solow1994measuring}.
Since its mathematical formalisation by \citet{leinster2013magnitude}, magnitude has been connected to numerous key geometric invariants, such as entropy, curvature, density, volume, and intrinsic dimensionality \citep{leinster2021entropy}.
Because of its intriguing theoretical properties magnitude has received increasing interest for machine learning tasks \citep{limbeck2024metric}. 
However, the magnitude of graphs~\citep{leinster2019magnitude}, despite being a strong graph invariant, has not yet found its way into applications.

%\kat{Let's change this to defining a graph first, then defining magnitude and spread.}

Throughout this paper, we consider an undirected finite graph \(G = (X,E)\) with $n$ nodes as a \emph{finite metric space}, consisting of the node set \(X\) equipped with a \emph{metric} \(d\colon X \times X \rightarrow \mathbb{R}_{\geq0}\). Its similarity matrix \(\zeta_{X} \subseteq \mathbb{R}^{n \times n}\) is defined by \(\zeta_{X}(x,y) = e^{-d(x,y)}\) for \(x,y \in X\). This allows us to introduce similarity-dependent notions of the diversity of metric spaces. 
To this end, we define a \emph{weighting} on the metric space \((X,d)\), which is a vector \(w \in \mathbb{R}^{n}\) such that \(\zeta_{X} w = \mathds{1}\), where $\mathds{1}$ is the column vector of ones. 
Whenever such a weighting exists, the \emph{magnitude} of the metric space \((X,d)\) is \emph{uniquely} defined by \(\text{Mag}(X)= \sum_{i=1}^{n} w(i)\). This is guaranteed if \(\zeta_X\) is positive definite, which essentially means that \(\zeta_X\) is invertible. A finite metric space with positive definite similarity matrix is called \emph{positive definite} \citep{meckes2013positive}. %Metric spaces of negative type \citep{leinster2021entropy} are positive definite, this includes 
Metric spaces of negative type are positive definite \citep{leinster2021entropy}; this includes \(\mathbb{R}^n\) equipped with the Euclidean distance \citep{leinster2013magnitude}, effective resistance distances \citep{devriendt2022discrete}, and diffusion distances~\citep{coifman2006diffusion}. 
Subsequently, we will refer to the magnitude %$\text{Mag}(G)$ 
of a graph \(G\) 
as the magnitude of its associated metric space \((X,d)\). %\footnote{Note that this is not the same definition as in \citet{leinster2019magnitude}.}
That is, we define the \emph{magnitude of a graph} as \begin{equation}
    \text{Mag(G)}= \sum_{x,y \in X} \zeta_{X}^{-1} (x,y).
\end{equation}
Closely related to magnitude, the \emph{spread} of a metric space is another measure of `size' introduced by \citet{willerton2015spread}. Given a metric graph \(G\) with the graph metric \(d\), its \emph{spread} is defined by \begin{equation}\text{Sp}(G) := \sum_{x \in X} \frac{1}{\sum_{y\in X} e^{-d(x,y)}}.
\end{equation}
As diversity measures on graphs, both \emph{magnitude} and \emph{spread} summarise the number of distinct sub-communities in a network based on the distance metric and degree of similarity between nodes. This view on structural diversity naturally aligns with our goal of contracting redundant graph structures during pooling. 
Throughout this work, we %compare both magnitude and spread to 
investigate %the close relationship between them and determine 
to what extent spread is a valid alternative to magnitude.
This is motivated by the fact that computing magnitude in practice
either requires inverting a matrix, solving a system of linear equations~\citep{limbeck2024metric}, or resort to approximations~\citep{andreeva2024approximating}, which can be computationally expensive and numerically unstable. 
Metric-space \emph{spread} in comparison can be computed given \emph{any} distance~(obviating the requirement of metric spaces of negative type), making it much more versatile %and stable than magnitude 
\citep{willerton2015spread}. Moreover, as the sum of reciprocal mean similarities, spread can be calculated or %even 
approximated \citep{dunne2024efficiently} much more efficiently than magnitude and does not require inverting a matrix.
Although spread has been studied less extensively, there are strong reasons to assume that it shares the same advantages as magnitude. %In fact, magnitude and spread are closely related. 
In fact, for a positive definite metric space $X$, we have \(\mathrm{Sp}(X) \leq \mathrm{Mag}(X)\)~\citep[Theorem~2]{willerton2015spread}. 
Moreover, magnitude and spread coincide for finite homogeneous metric spaces~\citep[Theorem~3]{willerton2015spread}, such as the ring graph in \Cref{fig:overview_graphs}. In practice, as we further explore in \Cref{app:mag_spread_corr}, %we find %that for relevant real-wold graph datasets considered throughout this paper, %such as NCI1, a dataset of graphs representing , 
%that 
the magnitude and spread of graphs from 
real-world datasets, such as NCI1, exhibit nigh-perfect correlation when computed based on diffusion distances, underlining the strong connection between the two quantities.

\subsection{Diffusion Geometry on Graphs}

\label{sec:diffusion}
In this work, we use diffusion distances to compute magnitude and spread on graphs. This is motivated by their desirable theoretical properties and the capability of diffusion to aid %...
and act along message passing in GNNs: Diffusion operators are closely associated to random walks and are efficient at identifying important structures in complex geometries while preserving local and non-linear structures~\citep{gasteiger2019diffusion}. The key idea is that the eigenvectors of the Markov matrices can be thought of as coordinates for the   underlying graph structure \citep{coifman2006diffusion}. This provides a vector-space representation of the graph that can be used to %be %subsequently used %for %diverse machine learning tasks.
%to 
assess the dissimilarity between nodes. 

We now briefly detail the type of diffusion distance used throughout this paper.
Consider a graph \(G\) and its adjacency matrix \(A\).
% The latter could be an affinity matrix derived from a kernel.
Let \(D\) be the diagonal degree matrix whose diagonal entries \(D_{ii}=\sum_{j=1}^nA_{ij}\) equal the degree of each vertex. The symmetrically \emph{normalised adjacency matrix} $\hat{A}:= D^{-\frac{1}{2}}AD^{\frac{1}{2}}$ is a \emph{Markov transition matrix} and represents the probability of moving from one vertex to another. 
The \emph{normalised graph Laplacian} is defined as \(\hat{L}= I-\hat{A}=D^{-\frac{1}{2}}(D-A)D^{-\frac{1}{2}}\). Since $\hat{L}$ is symmetric and positive definite, it has positive eigenvalues \(2 \geq \lambda_0 > \lambda_1 > \lambda_2 > \cdots > \lambda_{n-1} \geq 0\) with eigenvectors \(\lbrace \psi_l\rbrace_{l}\).
This provides a natural embedding of the graph \(G\) in Euclidean space given by:
\begin{equation}\label{eq:Phi}
    \Phi(x) = (\lambda_1 \psi_1(x), \cdots, \lambda_{n-1} \psi_{n-1}(x)) \text{ for } x\in X.
\end{equation}
The \emph{diffusion distance} is then defined by the $l^2$-norm, i.e.,
\begin{equation}
    d(x,y) = \norm{\Phi(x)-\Phi(y)}_{2} \text{ for } x,y \in X.
\end{equation}

\begin{restatable}{thm}{diffdist}
    Any finite metric space $(X,d)$ endowed with the diffusion distance is positive definite.
    \label[theorem]{thm:diff_dist_positive}
\end{restatable}
As a consequence of \Cref{thm:diff_dist_positive}, the magnitude of any metric graph %\((G,d)\) 
equipped with this diffusion distance is %is positive definite and its magnitude is 
well defined. 
Leveraging diffusion distances for %computing a distance on a graph 
our methods has further benefits. The normalised graph Laplacian, which works with the relative connectivity between nodes, is robust to varying node degrees and graph sizes. 
This ensures that diffusion distances are on comparable scales across graphs, which  enables us to compute and compare magnitude and spread directly. We note, however, that our pooling methods are flexible and can be applied to a wider range of alternative distances or similarities between nodes, which can be tailored to an application domain.

\section{Methods} 

We first describe our magnitude-guided graph pooling methods in \Cref{sec:mag_pooling} while providing an in-depth theoretical analysis in \Cref{sec:theory}.

\subsection{Magnitude-Guided Graph Pooling}
\label{sec:mag_pooling}

At the heart of our approach, we use magnitude or spread to monitor and control structural changes in the graph during edge contraction pooling. Edge contraction is chosen as a pooling operation because it respects graph connectivity, while outputting a sparsely-connected graph representation. 
Conceptually, edge pooling thus aligns well with the goal of making \emph{minimal} changes to the graph and keeping its diversity and geometry as unchanged as possible. Intuitively, pooled graphs with comparable magnitude will be similar in terms of effective size. That is, the effective number of distinct communities in the graphs are deemed similar based on their diffusion distance, i.e.\ based on the information flow between vertices.

\begin{figure}
  \centering
  \includegraphics[trim={0 0 
1.75cm 0},clip,width=0.97\linewidth]{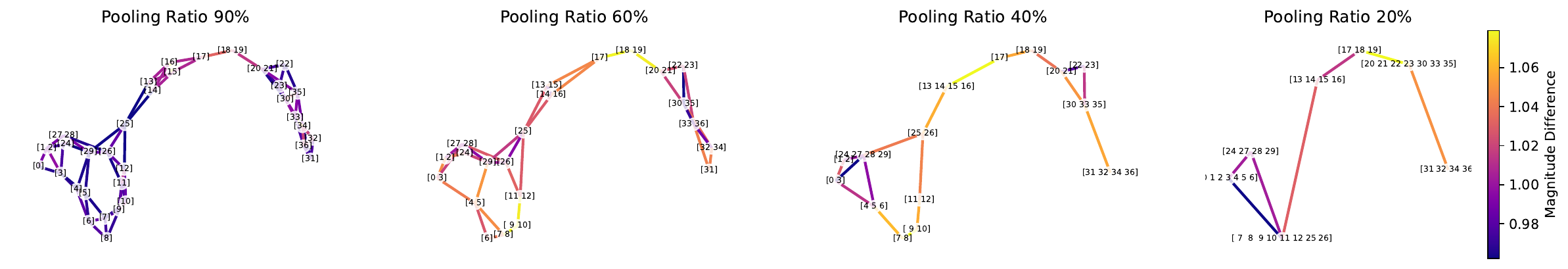}
  \caption{Illustrating our proposed pooling method, MagEdgePool, on a graph from the ENZYMES dataset across varying pooling ratios. Each edge is coloured by its magnitude difference, which measures the impact its contraction would have on the graph's structural diversity. %as computed from  diffusion distances. 
  Edges with low magnitude differences are most redundant for the graph's geometry and are collapsed first.  }\label{fig:example_enzyme}
\end{figure}

Formally, let \(G=(X,E)\) be a graph and denote by \(G/{e}\) the graph resulting from the contraction of the edge \(e \in E\) in $G$. %\todo{Double-check the definition of a graph.}
%$G$ for which the edge \(e \in E\) has been contracted. 
We choose a pooling ratio \(r\in (0, 1]\) and aim to reduce the graph to the corresponding number of nodes i.e.\ \(k= \lfloor r \cdot |X| \rceil \) where \(\lfloor {\,} \rceil\) denotes rounding to the nearest integer. Initially, we set the pooled graph \(G'= (X', E') :=G\). 
To assess which edges to contract first, we determine their importance for the graph's global geometry by computing a \emph{selection score} for each edge, which we  define  as
\begin{equation}
    \label{eq:scores}
    s(e) = |\text{Mag}(G) - \text{Mag}(G / {e})|.
\end{equation}
That is, we calculate the difference in magnitude between the original graph and the graph for which the edge has been collapsed. In this manner, we score an edge's relevance for the graph structure by the impact its collapse would have on the graph's magnitude. %for a graph's 
In each iteration \(i\), we then select the edge with the lowest magnitude difference 
\begin{equation}
    e_{i} \in \arg \min_{e\in E'\setminus E_c}  s(e)
\end{equation}
and assign \(G'=G'/e_{i}\) where \(E_c \subseteq E'\) is the set of all edges that are  adjacent to an already contracted edge. The edge \(e_i\) that we contract is selected \emph{at random} whenever there is more than one valid option. %Whenever multiple edges have the same score, we randomly choose between them. 
%Similar to \citet{diehl2019edge} we restrict our pooling approach to only collapse edges that are not adjacent to any already collapsed edge. 
If all edges that meet this requirement have been collapsed, but the pooling ratio is not reached, we re-compute the edge scores on the new graph and repeat this procedure. 
We stop if the pooling ratio is reached, i.e.\ \(|X'| = \lfloor r \cdot |X|\rceil\), or if there are no edges left in the reduced graph, i.e.\ \(E'= \emptyset\). This approach allows us to flexibly reduce graphs to any desired size. 

Edge pooling then gives a hard assignment of nodes, where each vertex is assigned to a single super-node in the output graph based on the neighbours it has been merged with. 
To compress the node features, we \emph{average} the features of any node that contributed to a pooled super-node. This ensures that information on all nodes' features is preserved during pooling. Note that the feature aggregation function could easily be modified to use sum pooling instead of mean pooling, for example.
Restricting the number of times a vertex can be merged is enforced to prevent our methods from collapsing entire portions of a graph early in the pooling process. This enables a more uniform pooling across the graph, which aids feature preservation and expressivity. 

\begin{tcolorbox}[
  boxsep     = 0.0mm,
  colframe   = black,
  colback    = black!10,
  left       = 2.0mm,
  right      = 2.0mm,
  sharp corners,
]
  \textbf{In a nutshell:}
  Our methods assume % a certain degree of 
%redundancy in the graph and 
that the \emph{most redundant edges} will be those whose removal would change the diversity of the graph the least.
Magnitude %then 
informs us about the global importance of an edge.
This implies that we want to start by merging edges from well-connected communities %with lower diffusion distances and between them 
whenever their contraction does not notably change the graph's effective size. 
\end{tcolorbox}

As illustrated in \Cref{fig:example_enzyme}, our pooling method ensures that globally influential edges will be collapsed \emph{late} in the pooling progress. For example, the edge %connecting the nodes [17] and [18,19], 
which bridges the two parts of this enzyme graph is scored more highly and thus merged later than well-connected cliques. %, for example between vertices [7] and [8]. 
Redundant local structures are collapsed first, and the overall~(diffusion) geometry is respected. For the graph in \Cref{fig:example_enzyme}, this ensures that characteristic features of the enzyme, such as the \emph{cycle}, are preserved across the pooling process. 
We provide a pseudocode implementation of our algorithm in \Cref{app:algo}. Notice that we may use, \emph{mutatis mutandis}, spread in lieu of magnitude. Whenever we use magnitude to compute the edge scores, we denote our algorithm by MagEdgePool, else we use the moniker SpreadEdgePool.
%Conceptually, as we will now motivate in further theoretical and experimental analysis, spread and magnitude are closely related. The choice of which diversity measure and pooling algorithm to use, will therefore rather depend on practical considerations, for which we find that spread gives a substantially more efficient and stable alternative to magnitude. %As a structure-based pooling method, SpreadEdgePool offers a light-weight solution for . 
%In addition, we note that computing each edge-score can be parallelised in application, offering a fast and light-weight approach to sparse, geometry-informed pooling.

%be 
%highly correlated in practical applications 
%Benefits: Intuitive.
%The scoring of each edge can be parallelised.

%\kat{I could add a plot with the correlation between magnitude differences and spread differences per edge somewhere. Here in the methods section, in the motivation or later in the experiments? Actually I also have the correlation with diffusion distances, but I don't know how relevant that is for the main text.}

%\end{figure}

%The normalised Laplacian 
%\begin{figure}
%  \centering
%  \fbox{\rule[-.5cm]{0cm}{4cm} \rule[-.5cm]{4cm}{0cm}}
%  \caption{An overview of our proposed pooling method and analysis pipeline in one picture.}
%\end{figure}

%\newpage

\subsection{Theoretical Analysis}
\label{sec:theory}
\begin{comment}
    
\begin{enumerate}
    \item Spectral distance bound ?
    \item Expressivity (see \cite{feng2024graph,ying2024boosting}).
    \item Time complexity (do this in the end)
   % \item Criteria to decide when to use spread. \kat{maybe always?}\Ly{Yeah, it seems your experiments suggest that, so it's not clear how to justify the use of Mag.} \kat{Let's stick with the motivation from existing literature on theoretical results / practical applications of magnitude. We don't need to show magnitude is better than spread. For this draft it is useful to show they are practically the same.}
    %\item Monotoniciy of magnitude under suitable choice of metric. \Ly{Always true for positive definite metric spaces}
\end{enumerate}

\end{comment}

We now present theoretical properties of our pooling methods. First, we highlight fundamental invariances and properties of magnitude and spread used to design our algorithms. Further, we provide a bound on the difference in magnitude during pooling by the difference in spread, demonstrating the close relationship between the two. %including some fundamental invariances. %. Specifically, we
%\paragraph{Bounding magnitude by spread.} We will now 
For a complete list or theorems and proofs, please refer to \Cref{app:proofs_mag}.

\paragraph{Additivity for disjoint graphs.} An important property of magnitude is that it behaves like the cardinality of sets. This behaviour is  useful when computing the magnitude of a disconnected graph by splitting the problem into calculating the magnitude of the disconnected components.
\begin{restatable}{thm}{disjoint}
%\begin{thm}
\label[theorem]{thm:disjoint}
    Consider a graph $G = G_1 \sqcup G_2$ consisting of the disjoint union of two graphs $G_1$ and $G_2$. Then $\text{Mag}(G) = \text{Mag}(G_1)+ \text{Mag}(G_2)$.  
\end{restatable}
%\vspace*{-12pt}
%\begin{proof} The proof follows from Proposition 1.4.4 by \citet{leinster2013magnitude}. See \Cref{app:proofs_mag}.\end{proof}

\paragraph{Isomorphism invariance.} Our pooling layers are invariant under isometries of the input graph %(automorphisms that preserve the metric) 
provided that the edge choice at each iteration is deterministic whenever the edge scores coincide. This is a consequence of the following result.

\begin{restatable}{thm}{invariance}\label[theorem]{thm:invariance}
    For \emph{isomorphic} graphs \(G_1, G_2\), we have  \(\text{Mag}(G_1) = \text{Mag}(G_2)\) and \(\text{Sp}(G_1) = \text{Sp}(G_2)\).  
\end{restatable}
%\vspace*{-12pt}
%\begin{proof}  This follows from the isometry invariance of the distance metrics. See \Cref{app:proofs_mag}.\end{proof}

%\kat{Let's condense the results below.}
%\Ly{I will fix the notation and rephrase in the end once I collect necessary results.}

%Edge contraction as used in our algorithm %as used by our pooling methods (detailed in \Cref{sec:mag_pooling}) %.
%We can consider 
\paragraph{Edge contraction on graphs.} 
Edge contraction is the main operation of our pooling methods and 
can be considered as a map \(f\) % : (G_1,d_1) \rightarrow (G_2,d_2)$ 
between graphs. %over the same vertex set $V$. %Note that in general $f$ is only a morphism of graphs. 
The following result provides a sufficient condition to ensure that $f$ remains compatible with the metric structure.

\begin{restatable}{thm}{magmonotonicity}\label[theorem]{thm:mag_monotonicity}
    Consider an edge-contraction map $f\colon(G_1,d_1) \rightarrow (G_2,d_2)$  between positive definite metric graphs. 
    If the map is $1$-Lipschitz, i.e.\ $d_2(f(v_1),f(v_2)) \leq d_1(v_1, v_2) \ \forall v_1,v_2 \in X_1$, then $\text{Mag}(G_2) \leq \text{Mag}(G_1)$.   
\end{restatable}
%\vspace*{-12pt}
%\begin{proof}
%    The proof follows from \citet{leinster2013magnitude}. See \Cref{app:proofs_mag}.
%\end{proof} 

%\begin{cor}
%    Consider two metric graphs $(G_1,d)$ and $(G_2,d)$ endowed with the shortest path distance. Then, for any edge-contraction mapping $f: (G_1,d) \rightarrow (G_2,d)$, we have $\text{Mag}(G_2) \leq \text{Mag}(G_2)$.  
%\end{cor}

%Consider an edge-contraction $f$. 
Starting from an initial metric graph $G = (X,E)$, a 1-Lipschitz edge-contraction map \(f\) yields a sequence of graphs $\lbrace G_i\rbrace_{i=1}^{k}$, where \(k\) is the number of edges that have been contracted. %Further assume that $f$ induces a distance decreasing surjection on the vertex sets. 
%Let $X^{(k)}$ be the vertex set resulting from the $k^{th}$ edge contraction 
This sequence can be constructed as described by our algorithms in \Cref{sec:mag_pooling}. We  
%However, it  is not guaranteed that MagEdgePool and SpreadEdgePool yield the same sequence of hierarchically pooled graphs. For this reason, we will 
refer to the graphs resulting from the $k^{th}$ edge-contraction using MagEdgePool %with magnitude.  %and spread 
as $G^{(k)}$. %and $\widetilde{G}^{(k)}$, respectively. 
Let $\Delta^{(k)}\text{Mag}(G) = |\text{Mag}(G^{(k-1)}) - \text{Mag}(G^{(k)})|$ and let $\Delta^{(k)}\text{Sp}(G) = |\text{Sp}(G^{(k-1)}) - \text{Sp}(G^{(k)})|$. % \\

\paragraph{Bounding magnitude by spread.} 
We  
track the difference in magnitude and spread throughout the edge contraction process detailed above. %used by our algorithm to demonstrate the %close 
This allows us to propose %, under suitable assumptions, 
an inequality that describes the relation between the difference of magnitude and spread during pooling. The bound then demonstrate the close conceptual relationship between SpreadEdgePool and MagEdgePool. 
\begin{restatable}{thm}{magspreadbound}\label{thm:mag_spread_Diff_bound}
    Consider a positive definite metric graph \(G\) with positive weights. Assume that the edge-contraction maps describing MagEdgePool and SpreadEdgePool induce distance decreasing surjections on the vertex sets. If $| \text{Mag}(G^{(k-1)}) - \text{Sp}(G^{(k)} )| \leq C \Delta^{(k)}\text{Sp}(G)$, then
    \begin{equation*}
        \Delta^{(k)}\text{Mag}(G) \leq 3C\Delta^{(k)}\text{Sp}(G).
    \end{equation*}
\end{restatable}

%\textcolor{blue}{
\paragraph{Expressivity.} Studying the expressive power of GNNs and their ability for %testing graph isomorphisms 
distinguishing non-isomorphic graphs 
offers theoretical insights for understanding the theoretical capabilities of pooling operators. 
\citet{bianchi2024expressive} state sufficient condition for a pooling layer to preserve the expressive power of the preceding message-passing (MP) layers. As demonstrated in \Cref{app:expressivity}, MagEdgePool and SpreadEdgePool satisfy these conditions ensuring expressivity.  
%}

\paragraph{Computational Complexity.} 
The time complexity of our pooling methods is independent of the GNN and is dominated by the cost of computing the edge scores in \Cref{eq:scores}. Given a graph \(G=(X,E)\), magnitude has time complexity \(O(|X|^3)\). In comparison, spread has time complexity \(O(|X|^2)\) and can be more efficiently approximated %via suitable subsets \(S \subset X\) in \(O(|S| \times |X|)\) time 
\citep{dunne2024efficiently}. %Assuming distance computations take time complexity 
Spread thus offers a considerably faster alternative. Now, let \(O(C_d)\) be the time complexity of computing the metric on \(G\). 
The time complexity of SpreadEdgePool is dominated by \(O(|E|(C_d + |X|^2 + \text{log}|E| ))\) and MagEdgePool by \(O(|E| (C_d +|X|^3 + \text{log}|E| ))\). %\(O(C_d|X| + C_d|E|\times(|X|-1)^2 + |E|log(|E|))\).
We note that on large graphs, it is possible to speed up the distance computations further to ensure scalability. 
See \Cref{app:proofs_complexity} for a full description of computational costs and \Cref{app:empirical_efficiency} for an empirical evaluation, which shows that our algorithm performs on-par with existing pooling methods across the datasets evaluated throughout this work.

%\begin{itemize}
%    \item Distance computations - 
%    \item Spread computations - \(O(|X|^2)\)
%    \item Magnitude computations - \(O(|X|^3)\)
%\end{itemize}

%Computational complexity \citep{ying2024boosting}. \kat{Highlight how much faster spread is than magnitude.} 
%\kat{We also need to explain  how our method scales with dataset sizes to adhere to the paper checklist.}

%\kat{Actually, can we show / highlight some type of stability for spread that we cannot show for magnitude?}

%\newpage
\section{Experimental Results}

 %Overview of contributions

 %\begin{itemize}
 %    \item MagEdgePool and SpreadEdgePool perform 
 %    \item We demonstrate that MagEdgePool and SpreadEdgePool better preserve structural properties than other pooling methods.
 %    \item 
 %\end{itemize}

Across our experiments, we address four key tasks, namely
 \begin{inparaenum}[(i)]
     \item graph classification performance,
     \item graph structure preservation during pooling, 
     \item performance across varying pooling ratios, and
     \item  performance at graph property regression.
 \end{inparaenum}

%\begin{table}
%  \caption{Graph classification accuracy of different pooling methods.}
%  \label{sample-table}
%  \centering
%  \begin{tabular}{lll}
%    \toprule
%    \multicolumn{2}{c}{Part}                   \\
%    \cmidrule(r){1-2}
%    Name     & Description     & Size ($\mu$m) \\
%    \midrule
%    Dendrite & Input terminal  & $\sim$100     \\
%    Axon     & Output terminal & $\sim$10      \\
%    Soma     & Cell body       & up to $10^6$  \\
%    \bottomrule
%  \end{tabular}
%\end{table}

%\subsection{Experimental Setup}
\subsection{Graph Classification}
\label{sec:classification}

The primary aim of using graph pooling layers is to preserve task-relevant information while reducing computational costs. %
In particular, useful pooling layers guarantee good performance %at downstream tasks such as graph classification 
across a wide range of different datasets, thus capturing essential information for the task at hand. 
We thus investigate how well our aim to preserve structural diversity during edge pooling translates to practical performance at graph classification tasks. %\todo{Clarify SOTA results and goals of comparing pooling methods.} 
Note that our goal is not to reach state-of-the-art accuracies on all tasks, but to benchmark the performance gain or loss of different pooling operators.
%set up a graph classification benchmar

%\Ly{Add diagram for the pipeline ?} \kat{If we have the space let's  add it in the end.}
\paragraph{Experimental Setup.} We evaluate %classification performance on %the performance of our proposed pooling layers MagEdgePool and SpreadEdgePool across 
8 different graph datasets,
%6 datasets taken from biological applications %\citep{sutherland2003spline, borgwardt2005protein, schomburg2004brenda, dobson2003distinguishing, shervashidze2011weisfeiler}
% and 2 representing social networks, %\citep{yanardag2015deep} 
as detailed in \Cref{app:datasets}. %\kat{Describe and reference the datasets to align with the paper checklist.} %Specifically, we compare none-trainable pooling methods
Whenever node features are not available, we use node degree as an input feature. 
%\paragraph{Model Architecture.} 
To ensure a fair comparison with alternative pooling methods, we follow the experimental setup by \citet{grattarola2022understanding} and guidance by \citet{errica2020fair} for fair model comparison. Specifically, we plug in each pooling layer into the model architecture specified by \citet{grattarola2022understanding}, which is of the following form: 
\[\text{MLP}(\mathbf{X}) \rightarrow \text{GNN}(\mathbf{X}, \mathbf{A}) \rightarrow \text{POOL}(\mathbf{X}, \mathbf{A}) \rightarrow \text{GNN}(\mathbf{X}, \mathbf{A}) \rightarrow \text{GlobalSum}(\mathbf{X}) \rightarrow \text{MLP}(\mathbf{X})
\]
%MLP(X)-GNN(X,A)-POOL(X,A)-GNN(X,A)-GLOBALSUM(X)-MLP(X). 
The model includes pre-processing and post-processing MLPs with 2 layers, 256 hidden units, ReLU activation, and batch normalization. 
\(\text{GNN}(\mathbf{X}, \mathbf{A}) \) refers to a graph neural network layer, more specifically a general convolutional layer \citep{you2020design} with  parameters chosen according to the best results achieved by \citet{you2020design}. %This general-purpose architecture is selected as it gives good performance on a range of graph and node classification tasks \citep{grattarola2022understanding}. 
%\paragraph{Graph Pooling.} 
As an intermediate layer, \(\text{POOL}(\mathbf{X}, \mathbf{A})\) corresponds to a specific pooling layer. %During our experiments, we vary the selection of this pooling layers. 
All pooling layers are configured to pool each graph to around 50\% of nodes. We also compare with `No Pooling,' the same model architecture without any pooling layers. 
%\paragraph{Model Training.} %When available, we evaluate each model across predefined training, test and validation splits (MolHIV, …) and evaluate their performance across 5 random seeds. Otherwise, w
We use 10-fold stratified cross-validation and further partition the training data into 90\% training and 10\% validation data while keeping the labels balanced between splits. Finally, we report the best test accuracy of each model trained using Adam with a cross-entropy loss~(batch size 32, learning rate
$0.0005$, and early stopping based on the validation loss with a patience of 50 epochs). Further details are described in \Cref{app:classification_details}. %For MolHIV we further report AUC-ROC as a performance metric rather than accuracy because this performance score is recommended for this very imbalanced dataset.

%\subsection{Graph Classification Results}
% Required packages:
% Beamer presentation requires \usepackage{colortbl} instead

\begin{table}[t]
  \caption{Mean and standard deviation of the graph classification accuracy. %For each dataset, t
  The best-performing model is marked in bold. %and %the %three top 
  %accuracies of the top performing methods 
  %the accuracy of the model 
  All models that did not perform significantly different from the best model %for each dataset 
  are coloured 
  \textcolor[HTML]{00B050}{green}. The rightmost column shows the mean rank of each pooling method across datasets.
  }
  \label{classification}
  \centering
  \resizebox{\columnwidth}{!}{
  \begin{tabular}{lccccccccc}
    \toprule
    \textbf{Method} & \textbf{ENZYMES} & \textbf{PROTEINS} & \textbf{Mutagenicity} & \textbf{DHFR} & \textbf{IMDB-B} & \textbf{IMDB-M} & \textbf{NCI1} & \textbf{NCI109} & \textbf{Mean Rank} \\
    \midrule
    \textbf{No Pooling} & {\color[HTML]{747474} 87.3 ± 2.5} & {\color[HTML]{747474} 73.8 ± 0.8}& {\color[HTML]{747474} 80.1 ± 1.3} & {\color[HTML]{747474} 71.4 ± 1.9} & {\color[HTML]{747474} 69.7 ± 0.7} & {\color[HTML]{747474} 46.0 ± 0.7} & {\color[HTML]{747474} 76.5 ± 1.8} & {\color[HTML]{747474} 74.3 ± 2.0} & {\color[HTML]{747474}-} \\\midrule
    \textbf{MagEdge} & {\color[HTML]{00B050} 91.5 ± 3.2} & {\color[HTML]{00B050} 76.4 ± 3.9} & {\color[HTML]{00B050} \textbf{77.5 ± 2.7}} & {\color[HTML]{00B050} 88.0 ± 3.8} & {\color[HTML]{00B050} 72.4 ± 1.7} & {\color[HTML]{00B050} 47.4 ± 1.7} & {\color[HTML]{00B050} 72.7 ± 2.4} & {\color[HTML]{00B050} 73.0 ± 3.3} & {\color[HTML]{00B050} \textbf{2.4}} \\
    \textbf{SpreadEdge} & {\color[HTML]{00B050} \textbf{92.8 ± 1.6}} & {\color[HTML]{00B050} 75.1 ± 3.1} & {\color[HTML]{00B050} 76.0 ± 4.0} & {\color[HTML]{00B050} \textbf{90.7 ± 3.8}} & {\color[HTML]{00B050} 71.8 ± 1.5} & {\color[HTML]{00B050} 47.3 ± 1.7} & {\color[HTML]{00B050} 73.4 ± 2.5} & {\color[HTML]{00B050} 71.8 ± 1.8} & {\color[HTML]{00B050} \textbf{3.0}} \\\midrule
    \textbf{NDP} & {\color[HTML]{00B050} 92.2 ± 1.6} & {\color[HTML]{00B050} 73.7 ± 3.9} & 73.4 ± 3.1 & 79.6 ± 4.4 & {\color[HTML]{00B050} \textbf{73.3 ± 2.0}} & {\color[HTML]{00B050} 47.3 ± 2.5} & 70.6 ± 2.2 & 70.0 ± 2.2 & 3.6 \\
    \textbf{Graclus} & {\color[HTML]{00B050} 91.3 ± 3.7} & {\color[HTML]{00B050} \textbf{76.6 ± 3.7}} & 72.5 ± 2.0 & 64.4 ± 5.8 & {\color[HTML]{00B050} 71.9 ± 1.5} & {\color[HTML]{00B050} \textbf{49.3 ± 2.4}} & 68.8 ± 1.4 & 69.5 ± 2.1 & 5.6 \\
    \textbf{NMF} & 78.6 ± 8.0 & {\color[HTML]{00B050} 73.0 ± 8.1} & 71.0 ± 4.7 & 66.5 ± 7.7 & 69.4 ± 2.5 & 43.3 ± 1.7 & {\color[HTML]{00B050} 71.0 ± 3.7} & {\color[HTML]{00B050} 72.0 ± 5.3} & 6.0 \\
    \textbf{TopK} & 82.2 ± 7.5 & {\color[HTML]{00B050} 73.2 ± 1.4} & {\color[HTML]{00B050} 75.8 ± 4.7} & 68.9 ± 3.0 & 68.9 ± 1.5 & 45.6 ± 1.0 & {\color[HTML]{00B050} \textbf{75.3 ± 2.4}} & {\color[HTML]{00B050} 73.9 ± 3.3} & 4.9 \\
    \textbf{SAGPool} & 82.4 ± 4.5 & {\color[HTML]{00B050} 73.8 ± 1.3} & {\color[HTML]{00B050} 76.0 ± 2.6} & 69.9 ± 3.0 & 69.1 ± 0.6 & 45.7 ± 0.5 & {\color[HTML]{00B050} 74.3 ± 2.8} & {\color[HTML]{00B050} \textbf{74.0 ± 2.3}} & 4.7 \\
    \textbf{DiffPool} & 74.0 ± 5.7 & 68.9 ± 2.0 & 68.4 ± 1.9 & 79.8 ± 3.2 & 68.3 ± 0.8 & 44.4 ± 0.8 & 68.9 ± 1.0 & 68.3 ± 1.9 & 7.6 \\
    \textbf{MinCut} & 80.2 ± 6.6 & {\color[HTML]{00B050} 75.6 ± 1.3} & 70.9 ± 1.5 & 63.8 ± 3.7 & 69.3 ± 0.7 & 46.1 ± 0.8 & 66.7 ± 1.4 & 66.9 ± 2.0 & 7.4 \\
    \bottomrule
  \end{tabular}}
\end{table}

\paragraph{Classification Results.} \Cref{classification} reports the mean and standard deviation of the test accuracy achieved by different pooling methods. We furthermore highlight which methods do \emph{not} perform statistically 
significantly different from the best model~(using pairwise Wilcoxon signed-rank tests applied to the accuracy scores and employing Holm--Bonferroni correction at a significance threshold of $p = 0.05$), thus permitting us to identify pooling methods that achieve top performance.  
Notably, both MagEdgePool and SpreadEdgePool achieve the \emph{best mean ranks} across datasets in terms of their accuracy. Further, they are always among the top-performing methods across all evaluated datasets. Altogether, both their ranking and their individual accuracy scores thus demonstrate superior and consistently high performance across graph classification tasks. 
The performance benefits of our proposed pooling methods are most pronounced on DHFR %\todo[]{add something about low performance gains overall - pt3 Reviewer 7pUv}
\citep{sutherland2003spline}, where
%, a dataset of 756 chemical compounds, where %despite all pooling layers using the same GNN architecture, 
%our approach shows and advantage on reducing the graphs, 
our methods surpass even the GNN without pooling layer by around 17 percentage points. This provides evidence that the regularising effects of our pooling approach can help reduce overfitting, especially for small datasets and geometrically-rich graphs. For other biological datasets~(Mutagenicity, NCI1 and NC109), our methods show competitive performance with trainable layers, indicating that the introduction of additional trainable components into the pooling layer is \emph{not} necessary to guarantee high task performance. On ENZYMES, DHFR, IMDB-BINARY, and IMDB-MULTI, non-trainable methods generally outperform trainable pooling layers, with  MagEdgePool and SpreadEdgePool consistently reaching high accuracy. 

Comparing pooling methods to using no pooling, we observe that MagEdgePool and SpreadEdgePool reach similar or even higher performance across datasets. For six datasets, our diversity-guided pooling methods improve mean accuracy, indicating that pooling retains task-relevant information while aiding the generalisation capabilities of the GNN. %\kat{Maybe make a comment on runtime / parameter complexity here.} 
MagEdgePool and SpreadEdgePool act as interpretable and expressive pooling transformations (see \Cref{app:node_features}), %\todo[]{refer to sect 3.2 for expressivity}
%, 
which  reduce the computational costs %of GNNs by 
making GNNs learn from graphs' coarsened geometry.   
%while focusing the models to learn from graphs' coarsened geometry. 
%
As reported in \Cref{classification}, MagEdgePool and SpreadEdgePool achieve very similar accuracies across datasets reaching top accuracies compared to  alternative pooling layers. %This confirms that spread constitutes a valid and faster alternative to magnitude. 
In practice, especially for large graphs, we recommend using SpreadEdgePool due to its high predictive performance and superior computational efficiency. %Overall, we still note that both MagEdgePool and SpreadEdgePool 
%
%\textcolor{blue}{Mention something about large graphs and the use of spread. Something like, experimental results on small graphs motivate the use of spread for larger ones to reduce the time complexity while preserving performance.} 

 %\begin{itemize}
 %    \item Highest mean rank.
 %    \item High performance across a range of datasets.
%\item Magnitude and spread very similar.
%\item Pooling performs similarity to or beats no pooling
 %\end{itemize}

% A key consideration for interpreting the performance of different pooling layers is to disentangle the effect of preserving task-relevant features v.s. preserving task-relevant graph structures. Our pooling method aims to address both considerations. Edge collapse gives a trackable aggregation of the features, while the goal of minimising the decrease in magnitude during pooling effectively preserves the graph's geometric structure.  

\subsection{Magnitude and Graph Structure Preservation}
\label{sec:structure}

Motivated by the visual comparison of graphs pooled using different pooling layers from \Cref{fig:overview_graphs}, we next set out to investigate the link between structure preservation and task performance. 
%We further investigate the link between structure preservation and task performance via the 
Specifically, we choose NCI1, a dataset of 4,110 graphs corresponding to chemical compounds, because it has been shown to possess both informative features and task-relevant graph structures \citep{coupette2025no}. %\todo{Repeat for further datasets- Q2 Reviewer 7pUv.} %Useful pooling layers should therefore accurately encode the relevant structural information in these data. 
We follow the same classification procedure as before and extract the pooled graph representations after training the GNNs described in \Cref{sec:classification}. Three pooling layers, MinCut, DiffPool and NMF, were removed from further comparison because they showed notably worse qualitative results for the motivating examples in \Cref{fig:overview_graphs} and classification performance in \Cref{classification}.  NDP and Graclus are evaluated across fewer pooling ratios than more adaptive methods, because they pool graphs to around half their size at every step. To assess graph structure preservation, we use the spectral distance defined as  \(\sqrt{ \sum_{k=1}^{K} (\lambda_k - \lambda_k')^2}\), i.e.\ the \(l_2\)-norm between the eigenspectra of the normalised Laplacians of the original and the pooled graphs \citep{wills2020metrics}. We also report the magnitude difference between graphs to evaluate the preservation of structural diversity. %and contrast how this goal translates into spectral preservation. %% \citep{bianchi2020hierarchical} 

MagEdgePool and SpreadEdgePool exhibit small spectral distances across pooling ratios as visualised in \Cref{fig:pooling_ratio_structure_main}. This indicates that 
contracting edges guided by structural diversity \emph{preserves} key spectral properties. %, which ensures that we  successfully encode to coarsened geometry of the input graph during pooling.  
%our proposed pooling methods 
Our methods also demonstrate low magnitude differences confirming that they perform as intended. 
In fact, the structure preservation scores for MagEdgePool and SpreadEdgePool coincide almost perfectly, giving empirical evidence that spread offers an alternative to magnitude. 
%However, we note that there is a higher variability in this spectral difference across graphs indicating that in this case the relative change in magnitude is a more expressive measure of structural differences. in this structure-aware manner
%The confirm that our methods do indeed perform as intended and preserve graphs' metric structure as measured by their difference in magnitude. %as indended. However, as well as preserving structural diversity,  
%This is supported by the spectral difference between the normalised graph Laplacians visualised in \Cref{fig:pooling_ratio_structure} (b).
%Furthermore, 
%key structural information during pooling.  
%Comparing trends in spectral distance with the difference in magnitude across 
\Cref{fig:pooling_ratio_structure_main} indicates that preserving magnitude corresponds to lower spectral distances and better retention of spectral properties. This link supports our motivation of guiding pooling by magnitude.

Alternative pooling layers fail to effectively preserve graphs' structural properties during both qualitative and quantitative comparisons to varying extents. 
Node decimal pooling (NDP) \citep{bianchi2020hierarchical} was specifically designed to preserve spectral properties during pooling. However, it still reaches both higher spectral distances and higher magnitude differences than MagEdgePool on average across pooling ratios, as visualised in \Cref{fig:pooling_ratio_structure_main}. %which is why we chose %... \kat{elaborate here or earlier in the paper}
Finally, the sparse pooling layers Graclus, TopKPool, and SAGPool, all show higher spectral distances than our approach. This difference is even more pronounced in terms of magnitude differences, where 
all these three methods demonstrate high distortion of the underlying metric space diversity. These findings are repeated for further datasets~(see \Cref{fig:pooling_ratio_structure_ablation} in the appendix) and agree with the qualitative comparisons between graphs pooled using different pooling layers as visualised in \Cref{fig:overview_graphs}. %Magnitude thus effectively encodes the already visualised in \Cref{fig:overview_graphs}.%, especially at high pooling ratios. In comparison, both TopK and SagPool distort the graph's diversity strongly for low pooling ratios, but as each graph 
We thus conclude that MagEdgePool and SpreadEdgePool successfully encode graphs' coarsened geometry during pooling, surpassing alternative pooling methods. %and do so in an interpretable manner as determined by edge collapse, thus

\begin{figure}
  \centering
   \includegraphics[width=1\linewidth]{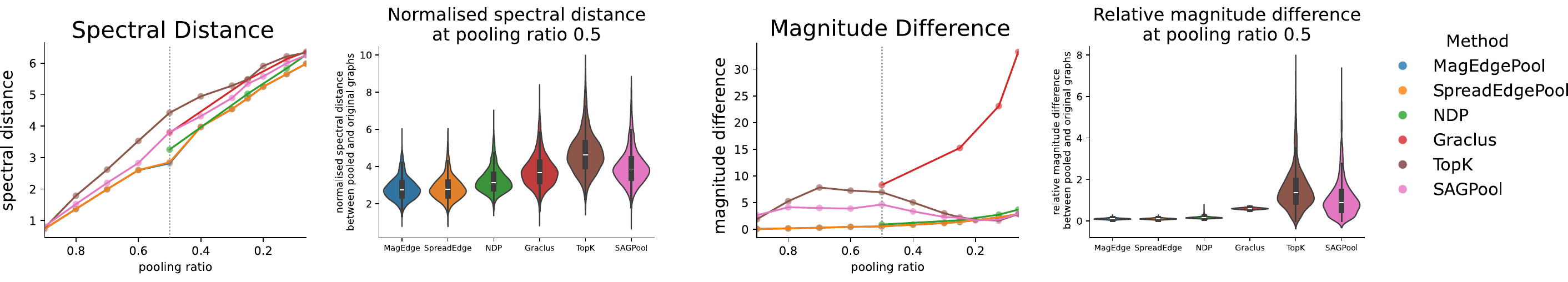}
    \caption{Structure preservation for all graphs in the NCI1 dataset across varying pooling ratios. Left: The spectral distance between the normalised Laplacians of the original and the pooled graphs. Right: The relative difference in magnitude, which summarises the proportional difference in structural diversity after pooling. Violin plots show the variability %in these scores 
    across graphs at pooling ratio 0.5. %Bold lines show the mean values of each score across graphs and thin lines the 10\%, 25\%, 75\% and 90\% quantiles.
  }\label{fig:pooling_ratio_structure_main}
\end{figure}

\subsection{Pooling Ratio and Task Performance}
\label{sec:ratio_acc}

\begin{figure}
  \centering
  \includegraphics[trim={0 0 4cm 0},clip,width=0.21\linewidth]{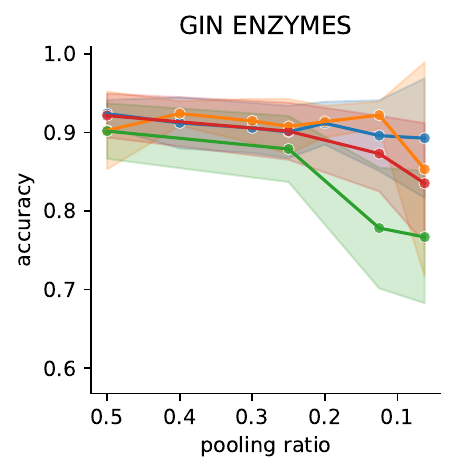}
   \includegraphics[trim={0 0 4cm 0},clip,width=0.21\linewidth]{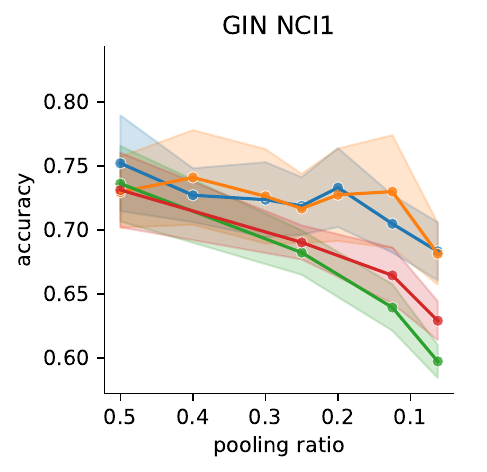}
  \includegraphics[trim={0 0 4cm 0},clip,width=0.21\linewidth]{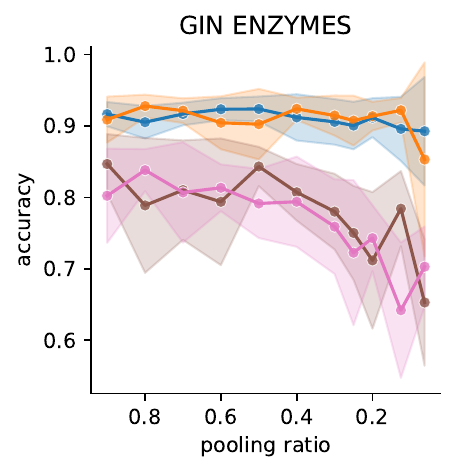}
   \includegraphics[trim={0 0 4cm 0},clip,width=0.21\linewidth]{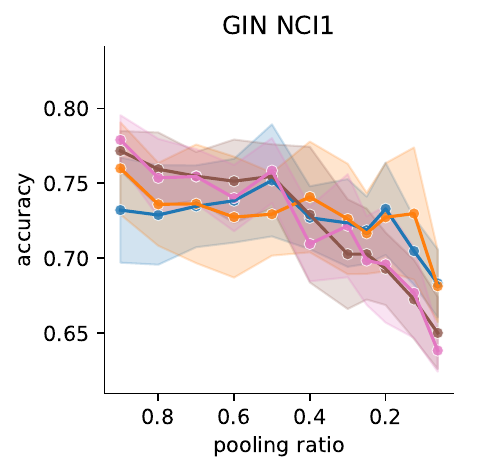}
   \raisebox{0.55\height}{\includegraphics[ width=0.12\linewidth]{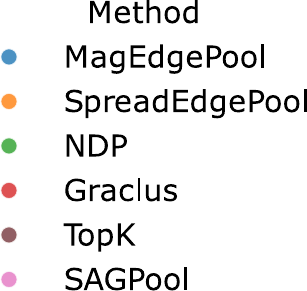}}
  \caption{Classification performance across varying the pooling ratio %from 0.9 to 0.0165 
  for different pooling layers. Pooling is applied as part of a GIN architecture. Results are shown for the ENZYME and NCI1 datasets. Lines show the mean and shaded areas the standard deviation of the test accuracy.}\label{fig:pooling_ratio}
\end{figure}

From \Cref{classification}, we observe that it is possible to reach very high performance on benchmark datasets even while pooling each graph to half its size. Based on this, we further investigate how %increasing 
the pooling ratio influences pooling layers and their classification performance. 
%Here, we take two datasets from our benchmark for further investigation. 
We consider two datasets, NCI1 and ENZYMES, which contains 600 graphs representing protein tertiary structures from 6 classes of enzymes and is selected as an example of a multi-class prediction task. We keep the same experimental setup described in \Cref{sec:classification}, but use GIN layers instead of general convolutional layers to further assess whether the trends in performance differ across model architectures. %, where we hope to see a good pooling method performing well preserving features, which have shown to be most informative for the task \citet{coupette2025no}. 
 \Cref{fig:pooling_ratio} reports the accuracy achieved by each pooling layer for varying pooling ratios. 
 Notably, we observe that MagEdgePool and SpreadEdgePool consistently reach very high test accuracies even at low pooling rations. Meanwhile, the performance of other non-trainable methods drops notably more when graphs are pooled to up to 6.25\% 
 %less than 1\% 
 of their original size 
 %pooling ratios of up to around 6.25\% 
 showing that they fail to preserve task-relevant information %during pooling 
 for both ENZYMES and NCI1. We note that the trainable pooling layers, TopK and SAGPool, reach higher or comparable performance on NCI1 for pooling ratios above 50\%, but decrease in accuracy for more extreme pooling ratios. They consistently perform worse on ENZYMES, indicating that they distort important graph features or key graph structures during pooling. MagEdgePool and SpreadEdgePool in comparison reach top performance and lower decreases in accuracy across varying pooling ratios demonstrating their potential to offer reliable, interpretable and stable pooling operations. 
 Overall, the reported accuracies in \Cref{fig:pooling_ratio} agree with results in \Cref{classification} indicating that our observations hold for varying choices of GNN layers. 
 We thus find across experiments that MagEdgePool and SpreadEdgePool constitute useful general-purpose pooling approaches, demonstrating their capability to faithfully encoding graphs' geometry, which ensures stable performance across pooling ratios, datasets and GNN architectures.

% Required packages:
% \usepackage{booktabs}
%\begin{table}[t]
\begin{wraptable}[11]{r}{0.4\textwidth}
    \vspace{-1\baselineskip}
  \caption{RMSE on the test data for different pooling methods.}
  \label{regression-table}
   \centering
  \resizebox{0.4\columnwidth}{!}{
  \begin{tabular}{lccc}
    \toprule
    %\toprule
    \textbf{Method} & \textbf{MolEsol} & \textbf{MolFreeSolv} & \textbf{MolLipo} \\
    \midrule
    \textbf{No Pooling} & {\color[HTML]{9B9B9B}1.44 ± 0.10} & {\color[HTML]{9B9B9B}2.98 ± 0.78} & {\color[HTML]{9B9B9B}0.98 ± 0.30} \\ \midrule
    \textbf{MagEdge} & {\color[HTML]{00B050}1.47 ± 0.11} & {\color[HTML]{00B050}2.81 ± 0.30} & {\color[HTML]{00B050}0.91 ± 0.23} \\
    \textbf{SpreadEdge} & {\color[HTML]{00B050}1.58 ± 0.15} & {\color[HTML]{00B050}2.83 ± 0.24} & {\color[HTML]{00B050}0.91 ± 0.20} \\ \midrule
    \textbf{NDP} & {\color[HTML]{00B050}1.54 ± 0.18} & 3.17 ± 0.29 & {\color[HTML]{00B050}0.82 ± 0.09} \\
    \textbf{Graclus} & {\color[HTML]{00B050}1.47 ± 0.13} & 2.99 ± 0.27 & {\color[HTML]{00B050}0.85 ± 0.17} \\
    \textbf{NMF} & 2.37 ± 0.77 & 15.06 ± 8.83 & {\color[HTML]{00B050}0.85 ± 0.04} \\
    \textbf{TopK} & 1.68 ± 0.10 & 3.25 ± 0.54 & 1.07 ± 0.25 \\
    \textbf{SAGPool} & 1.66 ± 0.18 & {\color[HTML]{00B050}2.77 ± 0.18} & 1.03 ± 0.20 \\
    \textbf{DiffPool} & 1.71 ± 0.14 & 5.78 ± 4.55 & 2.35 ± 2.57 \\
    \textbf{MinCut} & 1.88 ± 0.51 & 4.28 ± 0.32 & 1.22 ± 0.19 \\
    \bottomrule
  \end{tabular}}
%\end{table}
\end{wraptable}

\subsection{Graph Regression}

We further apply our pooling methods to graph regression tasks. We modify the GNN architecture as described in \Cref{app:reg} and plug in varying pooling layers. \Cref{regression-table} reports the RMSE on the test dataset across pre-defined data splits and ten random seeds for three datasets from OGB~\citep{hu2020open}. Overall, these results confirm that SpreadEdgePool and MagEdgePool constitute useful pooling operations and reach  comparatively low RMSEs across these three regression tasks. Our pooling methods successfully retain or improve task performance, indicating their ability to preserve task-relevant information.

\subsection{Computational Efficiency}

Finally, we evaluate the computational efficiency and scalability of pooling operators~(see \Cref{app:empirical_efficiency} for details).
Our approaches scale reasonably well to the datasets evaluated in this study, both in terms of runtime and memory requirements, thus remaining applicable for graphs with up to few hundreds of nodes.
Experiments confirm that SpreadEdgePool is notably faster than MagEdgePool and that distance approximations can be used to speed up the edge score computations used for pooling. 
Further, we %benchmark the training times and memory costs on the datasets from \Cref{classification}. %Results show that p
find that pre-computing allows for notably lower GNN training costs compared to trainable methods; %~(\Cref{tab:training_time}). 
when comparing our methods to EdgePool, our methods exhibit  superior computational efficiency during training
% (see \Cref{app:empirical_efficiency})
while ensuring similarly high classification performance.
%(see \Cref{app:comparison_EdgePool}).
%Hence, our pooling methods successfully improve the scalability of alternative edge pooling approaches.

\section{Discussion}
Despite its advantageous properties, our method exhibits certain \emph{limitations}: We implicitly require redundancy and homophily in the graph %and its features 
representation to make it amenable to  geometry-guided pooling. Our methods further assume that preserving graph structure is \emph{beneficial} to the learning task. Otherwise, edge pooling still aggregates features faithfully (as evaluated in \Cref{app:node_features}), but the geometric objective might not be necessary %for the learning task. %
to ensure high performance~(\Cref{app:random} provides an extended discussion and ablation study on the importance of preserving graph structure or preserving expressivity during pooling).
%\todo{Extend the discussion on expressivity vs structure-preservation.} %although there also is evidence that this doesn't necessarily hold in practice. 
%In terms of computational complexity, we  assume that 
Moreover, our algorithm relies on efficient distance computations, and %scale to large graphs
%, which is not guaranteed for all graph metrics 
we only explore one specific case of diffusion distances, which could be sped up~(see \Cref{app:empirical_efficiency}) or generalised further in future work.  
%\kat{Mention the distance approximations.} 
Although the non-trainable nature of our pooling methods can be limiting, our experiments nevertheless demonstrate that trainable pooling layers do \emph{not}
%necessary %to %guiding edge pooling to preserve  graphs' structural diversity %during edge contraction pooling 
%successfully %guides the pooling process to 
%encodes key graph properties and 
guarantee %s %stable and interpretable 
higher performance on standard benchmark datasets for graph classification or regression. 

Across experiments, we thus find that MagEdgePool and SpreadEdgePool constitute useful general-purpose pooling approaches: They are competitive compared to state-of-the-art pooling layers for graph classification or regression tasks, and perform well for a wide range of datasets and pooling ratios. 
%benefits GNNs  computational complexity by greatly reducing the number of trainable parameters. 
%While the none-trainable nature of our pooling methods is a limitation, w
Guiding edge pooling to preserve  graphs' structural diversity %during edge contraction pooling 
successfully %guides the pooling process to 
encodes key graph properties, ensuring stable and interpretable performance. %that cannot be guaranteed by trainable methods. 
Further, we overcome a major limitation of computing magnitude on large graphs by proposing spread as a substantially faster and closely-related alternative, which has the potential to aid research in geometric deep learning and efficient diversity evaluation beyond the scope of this paper. %We conclude that MagPool and SpreadPool are good default choices for interpretable and structure-aware graph pooling. 
For future work, we plan on applying our methods with alternative graph distances, to scale computations to large-scale graphs, or to determine an ideal pooling ratio automatically. Our methods are available as a Python package on GitHub.\footnote{\url{https://github.com/aidos-lab/mag\_edge\_pool} available under a BSD 3-Clause License.} %It is further of interest to study the benefit of extending our framework to include trainable components and investigate for which types of datasets and task, this could have the biggest benefits. 

%\kat{We might also want to draft a seperate limitations section.}

%By introducing spread as an alternative to magnitude we greatly improve computational performance of computing a geometric summary of the size of a metric space while avoiding expensive operations such as matrix inversion. We believe that this approach has promising extensions to a range of alternative applications, such as more scalable diversity evaluation \citep{limbeck2024metric, ospanov2024towards}, or ... .

%Further, our strategy can be used for cases when 
\newpage
\section*{Acknowledgements}

The authors are grateful for the stimulating discussions with the anonymous reviewers and the area
chair, who believed in the merits of this work. 
%Part of/This work was funded/supported by the Helmholtz International Lab Causal Cell Dynamics (InterLabs-0029) - Grant support from the Initiative and Networking Fund of the Hermann von Helmholtz-Association Deutscher Forschungszentren e.V..
K.L. is supported by the %\kat{check how to phrase the funding statement for the Helmholtz-Mila collaboration.}, and by the 
Helmholtz Association under the joint research school `Munich School for Data Science~(MUDS).'

\section*{Funding Disclosure}
This work was partially funded by the Helmholtz International Lab Causal Cell Dynamics InterLabs-0029 grant by the Initiative and Networking Fund of the Hermann von Helmholtz-Association Deutscher Forschungszentren e.V. [K.L.], Mila EDI scholarships [L.M.], Humboldt Research Fellowship,
CIFAR AI Chair, NSERC Discovery grant 03267, FRQNT grant 343567, and NSF grant DMS-2327211 [G.W.].
This work has received funding from the Swiss State Secretariat for
Education, Research, and Innovation~(SERI).
%\textcolor{red}{Add any further funding statements.} 
The content provided here is solely the responsibility of the authors and does
not necessarily represent the views of the funding agencies. 

%\subsection{Math}
%Note that display math in bare TeX commands will not create correct line numbers for submission. Please use LaTeX (or AMSTeX) commands for unnumbered display math. (You really shouldn't be using \$\$ anyway; see \url{https://tex.stackexchange.com/questions/503/why-is-preferable-to} and \url{https://tex.stackexchange.com/questions/40492/what-are-the-differences-between-align-equation-and-displaymath} for more information.)

\bibliographystyle{abbrvnat}
\bibliography{neurips_2025}

\clearpage
\appendix

\clearpage
\section*{NeurIPS Paper Checklist}

\begin{enumerate}

\item {\bf Claims}
    \item[] Question: Do the main claims made in the abstract and introduction accurately reflect the paper's contributions and scope?
    \item[] Answer: \answerYes{} % Replace by \answerYes{}, \answerNo{}, or \answerNA{}.
    \item[] Justification: We support all claims made in the introduction by empirical or theoretical evidence. %For example, we highlight the high performance of our pooling method on a range of benchmark datasets and expect it to reflect to settings beyond the scope of this paper. %\justificationTODO{}

\item {\bf Limitations}
    \item[] Question: Does the paper discuss the limitations of the work performed by the authors?
    \item[] Answer: \answerYes{} % Replace by \answerYes{}, \answerNo{}, or \answerNA{}.
    \item[] Justification: We discuss the computational complexity, and the non-trainable nature of our pooling approach. We further state our main assumptions on which edges in the underlying graph are deemed redundant and merged during pooling throughout the paper and clearly state any assumptions made for our theoretical contributions. Further, we clearly describe how empirical results were derived. %\justificationYes{}

\item {\bf Theory assumptions and proofs}
    \item[] Question: For each theoretical result, does the paper provide the full set of assumptions and a complete (and correct) proof?
    \item[] Answer: \answerYes{} % Replace by \answerYes{}, \answerNo{}, or \answerNA{}.
    \item[] Justification: All proofs reference relevant mathematical literature, are stated or referenced in the main text or included in the supplementary materials. %\justificationTODO{}

    \item {\bf Experimental result reproducibility}
    \item[] Question: Does the paper fully disclose all the information needed to reproduce the main experimental results of the paper to the extent that it affects the main claims and/or conclusions of the paper (regardless of whether the code and data are provided or not)?
    \item[] Answer: \answerYes{} % Replace by \answerYes{}, \answerNo{}, or \answerNA{}.
    \item[] Justification: We follow an established benchmark setting by \citet{grattarola2022understanding} to evaluate our pooling methods, and explain all further design choices we make throughout the paper, such as the use of stratified cross-validation instead of random sampling. The benchmark setup we use is therefore easily reproducible. Further, we provide the code for our algorithm. Further, all benchmark datasets used for evaluation are publicly available under the TUDataset project \citep{morris2020tudataset}. %\justificationTODO{}

\item {\bf Open access to data and code}
    \item[] Question: Does the paper provide open access to the data and code, with sufficient instructions to faithfully reproduce the main experimental results, as described in supplemental material?
    \item[] Answer: \answerYes{}%{} % Replace by \answerYes{}, \answerNo{}, or \answerNA{}.
    \item[] Justification: We provide the code to reproduce the main experiments in the supplementary materials. %\justificationTODO{}

\item {\bf Experimental setting/details}
    \item[] Question: Does the paper specify all the training and test details (e.g., data splits, hyperparameters, how they were chosen, type of optimizer, etc.) necessary to understand the results?
    \item[] Answer: \answerYes{} % Replace by \answerYes{}, \answerNo{}, or \answerNA{}.
    \item[] Justification: We follow the benchmark setting by \citet{grattarola2022understanding} for our experiments and further explain the model architecture, all parameter choices and method choices throughout the paper. Design choices for our proposed pooling method, such as the choice of graph distance, are explained in detail. %\justificationTODO{}

\item {\bf Experiment statistical significance}
    \item[] Question: Does the paper report error bars suitably and correctly defined or other appropriate information about the statistical significance of the experiments?
    \item[] Answer: \answerYes{} % Replace by \answerYes{}, \answerNo{}, or \answerNA{}.
    \item[] Justification: We report standard deviations for our classification results and further use statistical testing to determine which models performed on-par to the best performing models across datasets. The structure preservation experiment further includes violin plots to plot the distributions of structure preservation scores. Finally, when varying the pooling ratio, we report the standard deviation of each model's accuracy. %\justificationTODO{}

\item {\bf Experiments compute resources}
    \item[] Question: For each experiment, does the paper provide sufficient information on the computer resources (type of compute workers, memory, time of execution) needed to reproduce the experiments?
    \item[] Answer: \answerYes{} % Replace by \answerYes{}, \answerNo{}, or \answerNA{}.
    \item[] Justification: We detail the computational resources used for our experiments in the supplementary materials. %\justificationTODO{}
    
\item {\bf Code of ethics}
    \item[] Question: Does the research conducted in the paper conform, in every respect, with the NeurIPS Code of Ethics \url{https://neurips.cc/public/EthicsGuidelines}?
    \item[] Answer: \answerYes{} % Replace by \answerYes{}, \answerNo{}, or \answerNA{}.
    \item[] Justification: We have read and follow the NeurIPS Code of Ethics. %\justificationTODO{}

\item {\bf Broader impacts}
    \item[] Question: Does the paper discuss both potential positive societal impacts and negative societal impacts of the work performed?
    \item[] Answer: \answerYes{} % Replace by \answerYes{}, \answerNo{}, or \answerNA{}.
    \item[] Justification: We find that the proposal of novel pooling methods for GNNs is unlikely to have direct negative societal impacts. We further aimed to design interpretable and reliable pooling operators, which should mitigate the risk of unintentionally providing misleading results due to the pooling layer. %\justificationTODO{}
    
\item {\bf Safeguards}
    \item[] Question: Does the paper describe safeguards that have been put in place for responsible release of data or models that have a high risk for misuse (e.g., pretrained language models, image generators, or scraped datasets)?
    \item[] Answer: \answerNA{} % Replace by \answerYes{}, \answerNo{}, or \answerNA{}.
    \item[] Justification: Our paper poses no such risks.%\justificationTODO{}

\item {\bf Licenses for existing assets}
    \item[] Question: Are the creators or original owners of assets (e.g., code, data, models), used in the paper, properly credited and are the license and terms of use explicitly mentioned and properly respected?
    \item[] Answer: \answerYes{} % Replace by \answerYes{}, \answerNo{}, or \answerNA{}.
    \item[] Justification: We are careful to reference the code and dataset used throughout this work. Further details are included in the supplementary materials. %\justificationTODO{}

\item {\bf New assets}
    \item[] Question: Are new assets introduced in the paper well documented and is the documentation provided alongside the assets?
    \item[] Answer: \answerYes{} % Replace by \answerYes{}, \answerNo{}, or \answerNA{}.
    \item[] Justification: We provide a minimal implementation of our pooling methods in Python that can be easily integrated into existing workflows. %\justificationTODO{}

\item {\bf Crowdsourcing and research with human subjects}
    \item[] Question: For crowdsourcing experiments and research with human subjects, does the paper include the full text of instructions given to participants and screenshots, if applicable, as well as details about compensation (if any)? 
    \item[] Answer: \answerNA{} % Replace by \answerYes{}, \answerNo{}, or \answerNA{}.
    \item[] Justification: The paper does not involve crowdsourcing nor research with human subjects.

\item {\bf Institutional review board (IRB) approvals or equivalent for research with human subjects}
    \item[] Question: Does the paper describe potential risks incurred by study participants, whether such risks were disclosed to the subjects, and whether Institutional Review Board (IRB) approvals (or an equivalent approval/review based on the requirements of your country or institution) were obtained?
    \item[] Answer: \answerNA{} % Replace by \answerYes{}, \answerNo{}, or \answerNA{}.
    \item[] Justification: The paper does not involve crowdsourcing nor research with human subjects.%\justificationTODO{}

\item {\bf Declaration of LLM usage}
    \item[] Question: Does the paper describe the usage of LLMs if it is an important, original, or non-standard component of the core methods in this research? Note that if the LLM is used only for writing, editing, or formatting purposes and does not impact the core methodology, scientific rigorousness, or originality of the research, declaration is not required.
    %this research? 
    \item[] Answer: \answerNA{} % Replace by \answerYes{}, \answerNo{}, or \answerNA{}.
    \item[] Justification: The core method development in this research does not involve LLMs as any important, original, or non-standard components. %\justificationTODO{}

\end{enumerate}
%\fi

\clearpage

\crefalias{section}{appendix}
\crefalias{subsection}{appendix}
\crefalias{subsubsection}{appendix}

% Reset figure counts for the supplement with a nice prefix. 
\counterwithin*{figure}{part}
\stepcounter{part}
\renewcommand{\thefigure}{S.\arabic{figure}}

% Reset table counts for the supplement with a nice prefix. 
\counterwithin*{table}{part}
\stepcounter{part}
\renewcommand{\thetable}{S.\arabic{table}}

\startcontents
\printcontents{}{1}{{%
    \vskip10pt\hrule
    \large\textbf{Appendix~(Supplementary Materials)}\vskip3pt\hrule\vskip5pt}
}
\clearpage

\section{Technical Appendices and Supplementary Material}
%Technical appendices with additional results, figures, graphs and proofs may be submitted with the paper submission before the full submission deadline (see above), or as a separate PDF in the ZIP file below before the supplementary material deadline. There is no page limit for the technical appendices.

To elaborate on the results reported in our main paper, we first detail extended theoretical results and proofs for our theoretical contributions. In particular, we investigate computational complexity as well as the relationship between magnitude and spread. Next, we detail our experimental evaluation, the assets used for our experiments, and the algorithm describing our pooling methods. Finally, we report extended results on the experiments included in our main paper.

\section{Theoretical Analysis}
\label{app:theoretical_analyis}

This section details full proofs and extended explanations for the mathematical theory introduced in \Cref{sec:background} and the theoretical analysis of our pooling methods described in \Cref{sec:theory}.

\subsection{Diffusion Distances}
\label{app:proofs_diffusion}

As detailed in \Cref{sec:diffusion}, the \emph{diffusion distance} is defined by
\begin{equation}
    d(x,y) = \norm{\Phi(x)-\Phi(y)}_{2} \text{ for } x,y \in X.
\end{equation}
\diffdist*
\begin{proof}
    By definition of the diffusion distance, the map \(\Phi: X \rightarrow \mathbb{R}^{N-1}\) in \Cref{eq:Phi} defines an isometry \((X,d) \hookrightarrow l_{2}^{N-1}:=(\mathbb{R}^{N-1}, d_{2})\), where \(d_2\) is the metric induced by the \(l_{2}\)-norm. Finally, by Theorem 2.5.3 in \citet{leinster2013magnitude}, subsets of Euclidean space $l_2^{N-1}$ are positive definite.     
    %of a finite set $X$ is the metric induced by the $l_2$-norm on the image of the mapping $\Phi : X \rightarrow \mathbb{R}^{N-1}$. Finally, by Theorem 2.5.3 in \cite{leinster2013magnitude}, subsets of Euclidean space $l_2^{N-1}$ are positive definite.   
\end{proof}

\subsection{Computational Complexity}
\label{app:proofs_complexity}

%\kat{The text need to be edited still and a full expression for the computational complexity should be discussed.}

We next analyse  the computational complexity of our pooling methods, which are described in \Cref{app:algo} and \Cref{sec:mag_pooling}. Specifically, we expand on the statements in \Cref{sec:theory} by detailing the computational complexity of the pooling process. 
Given a graph \(G=(X,E)\), let \(k=\lfloor (1-r)|X|\rceil\) be the number of nodes that should be contracted as determined by the pooling ratio \(r\).
The time complexity of our pooling approach can be split up into the following steps:
%. 
%\kat{Elaborate on the computational complexity.}

\begin{description}
\item[Computing magnitude or spread.] Magnitude has time complexity \(O(|X|^3)\) and can further be approximated via iterative normalisation in \(O(i \times|S_i| \times |X|^2)\) time assuming \(G\) has a positive weighting where \(i\) is the number of iterations and \(S_i \subset X\) \citep{andreeva2024approximating}. 
Spread computations have time complexity \(O(|X|^2)\), which is a notable improvement to magnitude. It is possible to approximate spread computations via subsets \citep{dunne2024efficiently} or iterative optimisation using mini-batching \citep{andreeva2024approximating}. For \(i\) iterations on subsets \(S_i \subset X\), the time complexity of approximating spread reduces to \(O(i \times |S_i| \times |X|)\)  \citep{dunne2024efficiently, andreeva2024approximating}. %Assuming distance computations take time complexity 
Spread thus offers a much faster alternative to magnitude and can scale to large graphs considerably more efficiently.

\item[Computing distances and similarities.] For large datasets, it is key to speed up the distance calculations. Diffusion distances have time complexity \(O(|X|^3)\), but can be reduced to \(O(k|X|^2)\) when restricting the computations to the top k eigenvectors. Diffusion maps can further be approximated via low-rank approximations. To reduce the cost of repeated distance computations, it is possible to approximate the metric on the reduced graph \(G/e\) by directly updating the distances for \(G\). %, which can notably reduce the costs. %, \(O(C_d)\). %In fact, such a update could be applied naively by simply setting the distance between any node and the collapsed node to be the minimum distance on the original graph while leaving the remaining distances unchanged. To get a more precise estimate, it could be advisable to approximate the updated diffusion maps via more principled approaches, which we leave for future investigation. 
Given a distance matrix, computing the similarity matrix then has linear time complexity in the number of entries. Denote the time complexity of computing the distances and similarities by \(O(C_d)\).

%\paragraph{}
%To ensure fast and efficient performance in practice, the  computational efficiency of SpreadEdgePool can then be improved in multiple ways. 

\item[Edge contraction.] To get \(G'=G/e\), contracting an edge \(e\in E\) takes \(O(|X|)\) time.

\item[Edge score computations.] For each edge, its edge score is computed by applying the edge contraction and computing the magnitude or spread of the reduced graph, which takes \(O(|X|+C_d+C_S)\) time, where \(C_S\) refers to the cost of computing either magnitude or spread as detailed above. Note that the computation of these edge scores is independent across edges and can be parallelised.

\item[Edge score sorting.] The edge scores can be sorted from lowest to highest in \(O(|E|\log|E|)\) time.

\item[Feature aggregation.] The node features, \(\mathbf{F} \subseteq \mathbb{R}^{|X|,F}\), can be aggregated in \(O(|X| \times f)\) time.
\end{description}

Putting this all together, we get that the cost of our pooling algorithms can be described by a worst-case time complexity of
\[
O(|E|(|X|+C_d + C_S + log|E|) + |X| (f + k))
\]
if \(k \leq 0.5|X|\) and the graph is not pooled to less than half its size. Otherwise, if \(k > 0.5|X|\), we re-compute the edge scores whenever no valid edges are left as described in \Cref{sec:mag_pooling}. In this scenario, the first term of the complexity expression is repeated, corresponding to re-computations on successively smaller graphs. 
In summary, the overall time complexity of our pooling method is dominated by the cost of calculating and sorting the edge scores. This cost is independent of the choice of GNN architecture, ensuring that the training costs remain stable and do not escalate with model complexity. In practice, as further explored in \Cref{app:empirical_efficiency}, we thus find that our pooling algorithms perform on par with alternative pooling layers in terms of computational efficiency.

\subsection{Magnitude and Spread}
\label{app:proofs_mag}
%\kat{Let's finish up and include all poofs for the results from the main text. Also, let's correct the theorem numbers to match the main text.}

\subsubsection{Additivity for disjoint graphs} 
%\textcolor{red}{Add something about diffusion distance + shortest path satisfying conditions in proof}.

As a measure of the effective size, one appealing property of magnitude is that it behaves akin to cardinality. In fact, magnitude is additive when taking the disjoint union of multiple metric spaces. 

\disjoint*
\begin{proof} 
Let $G=(X,E)$ together with the metric $d$ be a positive definite metric graph. Assume $G$ is a disjoint union of two metric graphs $(G_1,d_1)$ and $(G_2,d_2)$ over the vertex sets $X_1$ and $X_2$, respectively, such that $d_{|_{X_1}} = d_1$ and $d_{|_{X_2}} = d_2$ and $d(x_1,x_2)=\infty$ for all $x_1 \in X_1, \ x_2 \in X_2$. Then, $\zeta_{G} = \zeta_{G_1} \oplus \zeta_{G_2}$ and $\text{Mag}(G) = \text{Mag}(G_1) + \text{Mag}(G_2)$.
%Proposition 2.3.2, \citep{leinster2013magnitude}    
\end{proof}

This result applies to graphs equipped with the shortest-path distance or the diffusion distance considered in this paper, because the distance between two nodes depends only on the connected component they belong to and is infinite if there is no path between them. Therefore, we can naturally see the similarity matrix \(\zeta_G\) as block-diagonal and compute the magnitude of \(G\) by summing up the magnitude of its disconnected subgraphs. %Note that the same additivity property holds for the spread of a metric space.

\subsubsection{Isomorphism invariance}

A key property of magnitude and spread is that they are isometry invariants of metric spaces. Note that by \emph{graph isometry} we mean an isometry on the underlying vertex set equipped with a metric. 

\invariance*
\begin{proof}
    Let $f\colon (X_1,d_1) \rightarrow (X_2,d_2)$ be a bijective isometry between the metric graphs $G_1$ and $G_2$, respectively. Then, for all $x,y \in X_1$ we have that $d_2(f(x), f(y)) = d_1(x,y)$. A consequence of the bijectivity of $f$ is that the distance matrices coincide (up to permutations) and that $\zeta_1 = \zeta_2$. This implies in turn that $\text{Mag}(G_1) = \text{Mag}(G_2)$ and that $\text{Sp}(G_1) = \text{Sp}(G_2)$. 
\end{proof}

%\kat{add a statement on how this relates to the pooling algorithms being invariant to isomorphisms}
Based on this property for magnitude and spread, we can show that isometry invariance also holds for our proposed pooling algorithm further detailed in \Cref{app:algo}.

\begin{cor}
    MagEdgePool and SpreadEdgePool are isometry-invariant if applied to isomorphic graphs provided the choice of edges to contract at each iteration is deterministic whenever the edge scores coincide.
\end{cor}
\begin{proof}
    Let $f\colon (G_1,d_1) \rightarrow (G_2,d_2)$ be a graph isomorphism and an isometry. From \Cref{thm:invariance}, we know that the edge scores as defined in \Cref{eq:scores} for \(e_1 \in E_1\) and \(f(e_1) \in E_2\) will coincide, i.e. \(s(e_1)=s(f(e_1))\). Because the choice of edges to contract at each iteration is further assumed to be deterministic if edge scores coincide, it follows that every edge to contract in \(G_2\) corresponds to the image $f(e)$ of an edge $e$ to contract in $G_1$ and vice versa. Hence, the pooled graphs output by our algorithm are isomorphic. % i.e. \(f(G_1')=G'_2\). 
\end{proof}

\subsubsection{Edge contraction on graphs} 
\label{app:proofs_edges}
We first recall important results about magnitude in the context of (strictly) positive definite finite metric spaces and refer the interested reader to \citet{leinster2013magnitude} for further details.  

\begin{prop}[Proposition 2.4.3, \citet{leinster2013magnitude}]\label[proposition]{Prop:mag_variational_def}
    Let $(X,d)$ be a positive definite metric space with finite cardinality $|X|=n$. Then
    \begin{equation}
        \text{Mag}(X) = \sup_{v \in \mathbb{R}^{n}\backslash \lbrace0 \rbrace} \frac{(\sum_{i=0}^{n}v_i)^2}{v^{t} \zeta_{X}v}.
    \end{equation}
\end{prop}
\Cref{Prop:mag_variational_def} implies that the magnitude of positive definite metric spaces is always positive. Another consequence of this result is a \emph{monotonicity} property on subsets of these metric spaces.   

\begin{cor}[Corollary 2.4.4, \citet{leinster2013magnitude}]
    Let $(X,d)$ be a positive definite finite metric space and consider a subset $Y \subset X$ (endowed with the induced metric). Then 
    \begin{equation}
        \text{Mag}(Y) \leq \text{Mag}(X). 
    \end{equation}
\end{cor}

We will now show an analogous result for graphs constructed via edge contraction.  

\magmonotonicity*
\begin{proof} 
    Let $f: (G_1,d_1) \rightarrow(G_2, d_2)$ be an edge-contraction map and let $n:=|X_1|$ and $m:=|X_2|$. We will identify $\mathbb{R}^{m}$ with a subset of $\mathbb{R}^{n} = \mathbb{R}^{m} \oplus \mathbb{R}^{n-m}$ using the map $(v_1, \cdots,v_m) \in \mathbb{R}^{m} \hookrightarrow (v_1, \cdots, v_m, 0, \cdots, 0) \in \mathbb{R}^{n}$. Then,
\begin{equation*}
    \begin{split}
        v^{t}\zeta_{X_1}v &= \sum_{i,j} v_i\zeta_{X_1}[i,j]v_j \\ & = \sum_{i,j} v_i \big (e^{-d_1(x_i,x_j)} \big )v_j \\ & \leq \sum_{i,j} v_i \zeta_{X_2}[i,j]v_j, 
    \end{split}
\end{equation*}
and 
\begin{equation} \label{ineq:zeta_2_less_zeta_1}
    \frac{1}{v^{t}\zeta_{X_2} v} \leq \frac{1}{v^{t}\zeta_{X_1} v} \ \forall \ v \neq 0.   
\end{equation}

Finally, by \Cref{Prop:mag_variational_def} and \Cref{ineq:zeta_2_less_zeta_1}, we get that  
\begin{equation*}
    \begin{split}
        \text{Mag}(G_1) &= \sup_{v \in \mathbb{R}^{n}\backslash \lbrace0 \rbrace} \frac{(\sum_{i=0}^{n}v_i)^2}{v^{t} \zeta_{X_1}v}\\
        &\geq  \sup_{v \in \mathbb{R}^{m}\backslash \lbrace0 \rbrace}\frac{(\sum_{i=0}^{n}v_i)^2}{v^{t} \zeta_{X_1}v}\\
        & \geq \sup_{v \in \mathbb{R}^{m}\backslash \lbrace0 \rbrace} \frac{(\sum_{i=0}^{m}v_i)^2}{v^{t} \zeta_{X_2}v}\\
        &= \text{Mag}(G_2)
    \end{split}
\end{equation*} 
\end{proof}

\begin{comment}
    
\begin{thm}
    Consider an edge-contraction map $f:(G_1,d_1) \rightarrow (G_2,d_2)$  between metric graphs. 
    If the map is Lipschitz, then $\text{Sp}(G_2) \leq \text{Sp}(G_1)$.  
\end{thm}
\begin{proof}
    Let $f: (G_1,d_1) \rightarrow(G_2, d_2)$ be an edge-contraction map and let $n:=|X_1|$ and $m:=|X_2|$. We will identify $\mathbb{R}^{m}$ with a subset of $\mathbb{R}^{n} = \mathbb{R}^{m} \oplus \mathbb{R}^{n-m}$ using the map $(v_1, \cdots,v_m) \in \mathbb{R}^{m} \hookrightarrow (v_1, \cdots, v_m, 0, \cdots, 0) \in \mathbb{R}^{n}$. 
    Then, for any \(v_i, v_j \in \{v_1,..,v_m\}\) we have \(e^{-d_1(v_i, v_j)} \leq e^{-d_2(f(v_i), f(v_j))}\). %Otherwise, for any vertex that is contracted, we set \(e^{-d_2(0, \cdot )}=e^{-d_2(\cdot,0)}=0\). 
    Hence, we have that 
    \[
    \text{Sp}(G_2) = \sum_{i=1}^m\frac{1}{\sum_{j=1}^m e^{-d_2(f(v_i), f(v_j))}} \leq \sum_{i=1}^m\frac{1}{\sum_{j=1}^m e^{-d_1(v_i, v_j)}} \leq \text{Sp}(G_1)\]
    %\[\leq \sum_{i=1}^m\frac{1}{\sum_{j=1}^m e^{-d_1(v_i, v_j)}}+ \sum_{m+1}^n\frac{1}{\sum_{j=1}^n e^{-d_1(v_i, v_j)}} \leq \text{Sp}(G_1).\]
    if we assume that
     \[\sum_{i=1}^m\frac{1}{\sum_{j=1}^m e^{-d_1(v_i, v_j)}} - \sum_{i=1}^m\frac{1}{\sum_{j=1}^n e^{-d_1(v_i, v_j)}} \leq \sum_{m+1}^n\frac{1}{\sum_{j=1}^n e^{-d_1(v_i, v_j)}}\]
    %\[\sum_{i=1}^m\frac{1}{\sum_{j=1}^m e^{-d_1(v_i, v_j)}} \leq \sum_{i=1}^m\frac{1}{\sum_{j=1}^n e^{-d_1(v_i, v_j)}}+ \sum_{m+1}^n\frac{1}{\sum_{j=1}^n e^{-d_1(v_i, v_j)}} = \text{Sp}(G_1).\]
    
\end{proof}

\end{comment}

\subsubsection{Bounding magnitude by spread}

Recall that for positive definite metric spaces, magnitude is known to be an upper bound for spread. 

\begin{thm}[Theorem 2.2.\ from \citet{willerton2015spread}] \label{thm:bound}
    Suppose that \(X\) is a finite metric space. If \(X\) is positive definite then 
    \[
    \text{Sp}(X) \leq \text{Mag}(X).
    \]
\end{thm}
%\textcolor{red}{Add result about spread < mag}.\\

We will now use this bound as well as the results in \Cref{app:proofs_edges} to investigate the relationship between MagEdgePool and SpreadEdgePool. Through a process of iterated edge contraction, our pooling algorithm produces a sequence of hierarchically pooled graphs (as described in \Cref{sec:mag_pooling} and \Cref{app:algo}). Note that it is not guaranteed that MagEdgePool and SpreadEdgePool yield the same sequence. For this reason, we will refer to the graphs resulting from the $k^{th}$ edge-contraction with MagEdgePool and SpreadEdgePool by $G^{(k)}$ and $\widetilde{G}^{(k)}$ respectively. 

For each $k$, the edge contraction map $G^{(k)} \rightarrow G^{(k+1)}$ is a surjection on the underlying vertex sets $X^{(k)}$ and $X^{(k+1)}$ respectively, i.e $X^{(k+1)} \subset X^{(k)}$. Moreover, we will assume that for any $k$ this map is distance-decreasing. That is, $d^{(k+1)}(f(x_i), f(x_j)) \leq d^{(k)}(x_i, x_j)$ for all $x_i,x_j \in X^{(k+1)}$.

Recall that for any $k$, $\text{Mag}(G^{(k)})$ is the magnitude of a finite positive definite metric space $(X^{(k)},d^{(k)})$. Then, by \Cref{thm:mag_monotonicity}, we deduce that for any $k$, 
\begin{equation}\label{Ineq:mag_pool_k_monotonicity}
    \text{Mag}(G^{(k+1)}) \leq \text{Mag}(G^{(k)}) 
    %\text{ and }  \text{Sp}(\widetilde{G}^{(k+1)}) \leq \text{Sp}(\widetilde{G}^{(k)}).
\end{equation}

%By construction, 
%\begin{equation*}
%    X^{(k)} = \text{argmin}_{Y \subset X ,\ |Y|+1= |X^{(k-1)}| } | \text{Mag}(X^{(k-1)}) - \text{Mag}(Y)|.
%\end{equation*}
%Then by Inequality \ref{Ineq:mag_pool_k_monotonicity}, 
%\begin{equation*}
%    X^{(k)} = \text{argmax}_{Y \subset X^{(k-1)}, |Y|+1 =|X|} \text{Mag}(Y) = \text{argmax}_{Y \subset X^{(k-1)}} \text{Mag}(Y).
%\end{equation*}
Note that by construction, scoring the edges in \Cref{alg:pool} translates into the following: 
\begin{equation}
    X^{(k)} = \text{argmin}_{Y \subset X^{(k-1)}, |Y|+1=|X^{(k-1)}|}|\text{Mag}(X^{(k-1)}) - \text{Mag}(Y)|
\end{equation}
and,
\begin{equation}
    \widetilde{X}^{(k)} = \text{argmin}_{Y \subset \widetilde{X}^{(k-1)}, |Y|+1=|\widetilde{X}^{(k-1)}|}|\text{Sp}(\widetilde{X}^{(k-1)}) - \text{Sp}(Y)|.
\end{equation}
Then, by the monotonicity of magnitude, i.e.\ \Cref{Ineq:mag_pool_k_monotonicity},
\begin{equation*}
    X^{(k)} = \text{argmax}_{Y \subset X^{(k-1)}, |Y|+1 =|X^{(k-1)}|} \text{Mag}(Y) = \text{argmax}_{Y \subset X^{(k-1)}} \text{Mag}(Y).
\end{equation*}
Let $\Delta^{(k)}\text{Mag}(G) = |\text{Mag}(G^{(k-1)}) - \text{Mag}(G^{(k)})|$ and let $\Delta^{(k)}\text{Sp}(G) = |\text{Sp}(G^{(k-1)}) - \text{Sp}(G^{(k)})|$.  
For the following result, we assume that scores are only computed once and that they are the only criterion for edge contraction. Furthermore, we will assume that spread is monotonically decreasing. 

\magspreadbound*

\begin{proof}
    For any $k$ we have the following inequality:
\begin{equation*}
    \begin{split}
        |\text{Mag}(G^{(k-1)}) - \text{Mag}(G^{(k)})| & \leq | \text{Sp}(G^{(k-1)}) - \text{Sp}(G^{(k)})|+ |\text{Mag}(G^{(k)})-\text{Sp}(G^{(k)})| \\
        & + | \text{Mag}(G^{(k-1)}) - \text{Sp}(G^{(k-1)})|. 
    \end{split}
\end{equation*}
Assume that $| \text{Mag}(G^{(k-1)}) - \text{Sp}(G^{(k)} )| \leq C \Delta^{(k)}\text{Sp}(G)$ for some constant $C>0$. By the monotonicity of magnitude~(\cref{thm:mag_monotonicity}), we get that $ \text{Mag}(G^{(k)}) \leq  \text{Mag}(G^{(k-1)})$ and, 
\begin{equation*}
    \text{Mag}(G^{(k)})-\text{Sp}(G^{(k)}) \leq \text{Mag}(G^{(k-1)})-\text{Sp}(G^{(k)}) \leq C \Delta^{(k)}\text{Sp}(G).
\end{equation*}
Similarly, assuming monotonicity of spread yields $\text{Sp}(G^{(k)}) \leq \text{Sp}(G^{(k-1)})$ and,
\begin{equation*}
    \text{Mag}(G^{(k-1)}) - \text{Sp}(G^{(k-1)}) \leq \text{Mag}(G^{(k-1)}) - \text{Sp}(G^{(k)}) \leq C \Delta^{(k)}\text{Sp}(G).
\end{equation*}
Since $ 0 \leq \text{Sp}(G^{(k-1)})-\text{Sp}(G^{(k)}) \leq \text{Mag}(G^{(k-1)}) - \text{Sp}(G^{(k)})$, the constant $C$ must be greater than or equal to 1. We conclude that
\begin{equation*}
    \begin{split}
        \Delta^{(k)}\text{Mag}(G) \leq 3C\Delta^{(k)}\text{Sp}(G). 
    \end{split}
\end{equation*}
 
\end{proof}

%---------------------------------------------------------------
\subsection{Expressivity}
\label{app:expressivity}

%\textcolor{blue}{
%\paragraph{Expressivity.}  %\kat{See \citet{bianchi2024expressive}, \citet{lachi2025expressive}.} 

While studying the expressive power of GNNs, we aim to evaluate their ability to generate different outputs for non-isomorphic graphs. Here, we will analyse the expressive power of our pooling methods MagEdgePool and SpreadEdgePool within the \emph{Select-Reduce-Connect} framework introduced by \citet{grattarola2022understanding} for describing pooling operators.  
%For the following discussion on the expressivity on pooling operators, we briefly introduce the \emph{Select-Reduce-Connect framework} introduced in \citet{grattarola2022understanding} to evaluate and compare graph pooling layers. 
Let $G=(X,E)$ be a graph %with vertex set $X$ and edge set $E$. Let 
and $\mathbf{F}\in \mathbb{R}^{|X|\times f}$ be the node features associated to the nodes in $G$. Then, a graph pooling operator is regarded as a function $\text{\textbf{POOL}}\colon (\mathbf{F},G) \mapsto (\mathbf{F}_P,G_P)$ where $\mathbf{F}_P$ denotes the pooled node features and $G_P=(X_P,E_P)$ the pooled graph with $|X_P| \leq |X|$. Pooling is described as a combination of three elementary functions: \emph{selection} (\textbf{SEL}), \emph{reduction} (\textbf{RED}), and \emph{connection} (\textbf{CON}). %\todo{explain which functions they are} %We will denote by $G^{L}$ the graph resulting from a block of $L$ MP layers and by $\mathbf{F}^{L} \in \mathbb{R}^{n \times f}$ the corresponding feature matrix. %\todo{Check if the definition of a graph and of the node features is consistent across the manuscript.}     
The selection function clusters the nodes of the input graph into super-nodes, so that $\text{\textbf{SEL}}\colon G\mapsto\mathcal{S}=\{\mathcal{S}_j\}_{j=1}^{|X_P|}$ where $\mathcal{S}_j=\{S_{ij}\}_{i=1}^{|X|}$ and $S_{ij}$ is the membership score of the node $i$ to super-node $j$. Node selection can be represented as the assignment matrix $S \in \mathbb{R}^{|X|\times |X_P|}$ with entries $S_{ij}$. Based on this selection, the reduction function aggregates node features of all nodes that are assigned to the same super-node, i.e. $\textbf{RED}\colon (\mathbf{F},S) \mapsto \mathbf{F}_P$. Then, the connection function, \textbf{CON}, generates the edges and determines the connectivities between super-nodes in the pooled graph. Finally, to study expressivity, we note that hierarchical graph pooling is typically applied in GNN architectures after some initial message passing layers. Let  $G^{L}$ denote the graph resulting from a block of $L$ MP layers and  $\mathbf{F}^{L} \in \mathbb{R}^{|X| \times f}$ the corresponding feature matrix \citep{bianchi2024expressive}. 

\begin{thm}(Theorem 1 from \citet{bianchi2024expressive}) \label[theorem]{thm:expressivity_Bianchi_Lachi}
  Let $G_1=(X_1,E_1)$ and $G_2=(X_2,E_2)$ be two graphs %defined by the vertex-edge pairs $(X_i,E_i)$ and 
  equipped with node features $\mathbf{F}_i \in \mathbb{R}^{|X_i|\times f}$ for $i=1,2$. %and let $\mathbf{F}_1$ and $\mathbf{F}_2$ denote the corresponding node features. 
  Assume that $G_1 \neq_{\text{WL}} G_2$, i.e. that $G_1$ and $G_2$ are distinguishable by the Weisfeiler-Leman isomorphism test. Apply a block of L MP layers to get $G^{L}_1$ and $G^{L}_2$ as well as $\mathbf{F}^{L}_1$ and $\mathbf{F}^{L}_2$. Let \text{\textbf{POOL}} be a pooling operator placed after these MP layers to get $G_{1_P} = \text{\textbf{POOL}}(G^{L}_1)$ and $G_{2_P}=\text{\textbf{POOL}}(G^{L}_2)$ associated with the node features $\mathbf{F}_{1_P} $ and $\mathbf{F}_{2_P}$ in $\mathbb{R}^{k\times f}$. %Let $x_{P_j}$ and $y_{P_j}$ be the feature vectors of super-node $j$ in the graphs $G_{1_P}$ and $G_{2_P}$ respectively. T
  Then, $G_{1_P}$ and $G_{2_P}$ will have different node features (up to permutation) %and $G_{1_P} \neq_{WL} G_{2_P}$ 
  provided the following conditions hold:
  \begin{enumerate}[noitemsep]
      \item %$ \sum_i^{|X_1|} x_{i}^{L} \neq \sum_i^{|X_2|} y_{i}^{L}$, 
      $ \sum_{i=1}^{|X_1|} \mathbf{F}^L_{{1}_{[i,:]}} \neq \sum_{i=1}^{|X_2|} \mathbf{F}^L_{{2}_{[i,:]}}$,
      \item The memberships generated by \textbf{SEL} satisfy  $\sum_{j=1}^k S_{ij} = \lambda$, with $\lambda >0$ for each node $i$, i.e., the cluster assignment matrix S is a right stochastic matrix up to the global constant $\lambda$,
      \item The reduction function satisfies $\text{\textbf{RED}}\colon (\mathbf{F}^{L}, S) \mapsto \mathbf{F}_P = S^{T}\mathbf{F}^{L}$.
  \end{enumerate}
\end{thm}
%By analysing our pooling layers using the \emph{Select-Reduce-Connect framework}, we obtain the following result about their expressivity power. 
%\begin{restatable}{thm}{Expressivity}\label{thm:Expressivity}
%    The pooling operators MagEdgePool and SpreadEdgePool are expressive.
%\end{restatable}
%} Due to the injectiveness of the coloring function of the WL
%algorithm, two graphs with different multisets of node features will be classified as non-isomorphic by
%the WL test a
The WL test will identify that two graphs with different multisets of node features are non-isomorphic based on the injectivity of the colouring function of the WL algorithm. \Cref{thm:expressivity_Bianchi_Lachi} then guarantees that $G_{1_P} \neq_{WL} G_{2_P}$ thus ensuring that the pooling operation \textbf{POOL} preserves  expressivity.

%\Expressivity*
\begin{cor}
    MagEdgePool and SpreadEdgePool satisfy the sufficient conditions outlined in \Cref{thm:expressivity_Bianchi_Lachi} %. %. %for preserving %the expressive power of message-passing (MP) 
    %expressivity 
    %(see \citet{bianchi2024expressive}) 
    %
    when using sum aggregation for pooling the node features. 
\end{cor}
\begin{proof}
    %We will show that   
    Condition 1 is independent to the choice of pooling layer and instead relates to the expressivity of MP layers. It is  guaranteed %by the \emph{Universal Approximation Theorem} 
    for any MP
layer that is as powerful as the 1-WL test %and the existence of a function that makes the sum over a multiset of countable node features injective 
    \citep{%xu2018powerful, 
    bianchi2024expressive, feng2024graph}. Conditions 2 and 3 hold trivially by construction. %At each step of the algorithm, a 
    Each super-node is the result of edge contractions and node features are aggregated via summation. %and %as such %it contains only two vertices 
    Hence, each vertex is assigned to a unique super-node %Node features are %aggregated via summation 
    %averaged, so 
    and the selection matrix is constructed as $S_{ij} = 1$ %\frac{1}{|S_j|}$ %1$ 
     if the node $i$ is contained in super-node $j$ %, %i.e. $x_i \in S_j$, 
    %$x_i \in S_j \subseteq X_{\dot}$ 
    and $S_{ij}=0$ otherwise.  This ensures that $\sum_{j=1}^k S_{ij} = 1$ for every node $i$ fulfilling condition 2. 
    The resulting \textbf{SEL} function can be represented using the \emph{cluster assignment matrix} $S$ obtained as a product of these elementary (contraction) operations represented by matrices. %A proper normalisation 
    This yields the \textbf{RED} function %(weighted sum) 
    described in our algorithm in \Cref{alg:pool} as a map $\mathbf{F}^{L} \mapsto S^T \mathbf{F}^{L}$ that respects condition 3.   
\end{proof}

%-------------------- PROOFS END -----------------------------

\newpage
\section{Extended Methods}

\subsection{Hardware and Software}
\label{app:ware}

The experiments reported in our study were implemented using \texttt{spektral 1.3.1} \citep{grattarola2021graph}\footnote{\url{https://graphneural.network/} available under an MIT license.}, and \texttt{tensorflow 2.16.2} \citep{tensorflow2015-whitepaper}\footnote{\url{https://pypi.org/project/tensorflow/2.16.2/} available under the Apache Software License (Apache 2.0).}. As further detailed in \Cref{app:classification_details} and \Cref{sec:classification}, we base our graph classification experiments on the benchmark setup and code by \citet{grattarola2022understanding}\footnote{\url{https://github.com/danielegrattarola/SRC} available to the research community (\citet{grattarola2022understanding}).}, which also include implementations for the pooling layers compared across our study. By relying on this existing framework, we aim to ensure the reproducibility of our results. 

Further, to calculate magnitude and spread we rely on \texttt{magnipy}, a Python package by \citet{limbeck2024metric} for magnitude and diversity computations.\footnote{\url{https://github.com/aidos-lab/magnipy} available under a BSD 3-Clause License.}
Further, as a novel contribution of this paper, we extend the computation of magnitude to graph data and novel graph metrics. Specifically, we modify the computations, so that the magnitude of disconnected subgraphs is computed separately (based on \Cref{thm:disjoint}) using the \texttt{NetworkX}\footnote{\url{https://github.com/networkx/networkx} available under a BSD 3-Clause License.} package. We also implement graph distances that have not previously been used to compute magnitude, such as the diffusion distances detailed in \Cref{sec:diffusion}. %, such as the diffusion distances implemented by \citet{coupette2025no}\footnote{}
Further details on the code for implementing our proposed pooling methods and our experiments can be found in our supplementary code submission as well as on GitHub. Finally, we publish a reproducible \texttt{PyTorch} implementation of our pooling methods as \texttt{mag\_edge\_pool}\footnote{\url{https://github.com/aidos-lab/mag\_edge\_pool} available under a BSD 3-Clause License.} on GitHub.

All experiments were conducted on 
a high-performance cluster with  hardware specifications as detailed in \Cref{tab:compute}. In particular, all experiment were run requesting a single GPU with 32~GB video memory or less.

\begin{table}[h]
\caption{Summary of the compute resources used for our experiments.} \label{tab:compute}
\centering
%\resizebox{0.8\linewidth}{!}{
\begin{tabular}{ll}
\toprule
\textbf{Inventory} & \textbf{Models}                                                                           \\
\midrule
Available CPUs     & Intel Xeon (Gold 6128, 6130, 6134, 6136, 6142, 6240, 6248R)                                \\
                   & Intel Xeon Platinum (8280L, 8480+, 8468, 8562Y+)                                          \\
                   & Intel Xeon (E7-4850, E5620, 4114, 6126)                                                   \\
                   & AMD EPYC (7262, 7413, 7513, 7713, 7742)                                                   \\
                   & AMD Opteron (6128, 6164 HE, 6234, 6272, 6376 x2)                                          \\
\midrule
Available GPUs     & NVIDIA Tesla (K80, P100, V100)                                                            \\
                   & NVIDIA A100 (20GB, 40GB, 80GB PCIe)                                                       \\
                   & NVIDIA H100 (80GB PCIe)                                                                   \\
                   & NVIDIA Quadro RTX 8000                                                                    \\
\bottomrule
\end{tabular}%}
\end{table}

\subsection{Datasets}
\label{app:datasets}
%\kat{Describe and reference all assets}

We briefly describe the graph datasets analysed throughout our work. 
Simulated graphs, as used for \Cref{fig:overview_graphs}, are created using either \texttt{PyGSP} \footnote{\url{https://pygsp.readthedocs.io/en/stable/} available under a BSD-3-Clause license.} %\texttt{PyTorch Geometric}\footnote{\hyperlink{https://pyg.org/}{https://pyg.org/} available under an MIT license.} 
or \texttt{NetworkX}\footnote{\url{https://github.com/networkx/networkx} available under a BSD 3-Clause License.} and all example graphs were created to consist of 64 nodes.

%Reference code and datasets used for the experiments
%Provide URLS and copyright information

For our main graph classification experiments, we analyse six graph datasets taken from biological or chemical applications \citep{sutherland2003spline, borgwardt2005protein, schomburg2004brenda, dobson2003distinguishing, shervashidze2011weisfeiler}, and two datasets which represent social networks \citep{yanardag2015deep}. All datasets are taken either from the \texttt{TUDataset}\footnote{\url{https://chrsmrrs.github.io/datasets/} available under a CC-BY-4.0 license.} benchmark \citep{morris2020tudataset} or the Open Graph Benchmark\footnote{\url{https://ogb.stanford.edu/} available under an MIT licence.}. 

More specifically, the results in \Cref{classification}, \Cref{regression-table}, and \Cref{extended-class} analyse the following graph  datasets described in \Cref{tab:data}. Note that we only consider node and not edge features for our experiments.

\begin{table}[ht]
\caption{Summary of the graph datasets considered for our experiments.}\label{tab:data}
\resizebox{1\columnwidth}{!}{
\begin{tabular}{lllllll}
\toprule
\textbf{dataset}        & \textbf{library} & \textbf{\# classes} & \multicolumn{1}{l}{\textbf{\# node features}} & \textbf{\#  graphs} & \textbf{avg \# nodes} & \textbf{avg \# edges} \\
\midrule
\textbf{MUTAG}          & TUDataset        & 2                  & no                                         & 187               & 18                   & 40                   \\
\textbf{Enzymes}        & TUDataset        & 6                  & 18                                         & 600               & 33                   & 62                   \\
\textbf{COX2}           & TUDataset        & 2                  & 3                                          & 467               & 41                   & 44                   \\
\textbf{DHFR}           & TUDataset        & 2                  & 3                                          & 756               & 42                   & 45                   \\
\textbf{IMDB-B}         & TUDataset        & 2                  & no                                         & 1000              & 20                   & 97                   \\
\textbf{IMDB-M}         & TUDataset        & 3                  & no                                         & 1500              & 13                   & 65                   \\
\textbf{AIDS}           & TUDataset        & 2                  & 4                                          & 2000              & 15                   & 16                   \\
\textbf{Proteins} & TUDataset        & 2                  & 29                                         & 1113              & 39                   & 72                   \\
\textbf{Mutagenicity}   & TUDataset        & 2                  & no                                         & 4337              & 30                   & 31                   \\
\textbf{NCI1}           & TUDataset        & 2                  & no                                         & 4110              & 30                   & 32                   \\
\textbf{NCI109}         & TUDataset        & 2                  & no                                         & 4127              & 30                   & 32                   \\
\textbf{OGBG-MOLHIV}    & OGB              & 2                  & 9                                          & 41127             & 25                   & 27                   \\
\textbf{BZR}            & TUDataset        & 2                  & 3                                          & 405               & 36                   & 38                   \\
\textbf{BZR\_MD}        & TUDataset        & 2                  & no                                         & 306               & 21                   & 225                  \\
\textbf{COX2\_MD}       & TUDataset        & 2                  & no                                         & 303               & 26                   & 335                  \\
\textbf{DHFR\_MD}       & TUDataset        & 2                  & no                                         & 393               & 24                   & 283                  \\
\textbf{ER\_MD}         & TUDataset        & 2                  & no                                         & 446               & 21                   & 235 \\
\textbf{OGBG-MOlESOL}    & OGB              & regression                  & 9                                          & 1128             & 13                   & 14                   \\
\textbf{OGBG-MOlFREESOLV}    & OGB              & regression                  & 9                                          & 642             & 9                   & 8                   \\
\textbf{OGBG-MOLLIPO}    & OGB              & regression                  & 9                                          & 4200             & 27                   & 30                   \\
\midrule
\end{tabular}}
\end{table}

\subsection{Magnitude and Spread Computations}

Across our experiments, we compute magnitude and spread as outlined in the main text, implemented in our code submission, and further described in \Cref{app:ware}. Elaborating on these descriptions, we now aim to give an extended explanation of practical and theoretical considerations for computing the magnitude of graphs in practice.

\paragraph{Defining the magnitude of a graph.} In mathematical literature, the magnitude of graphs is often studied with the shortest path metric  \citep{leinster2019magnitude}. 
However, shortest path distances are not guaranteed to be of negative type, thus leading to scenarios and well-known examples for which the similarity matrix is not invertible and magnitude based on this metric cannot be computed \citep{leinster2013magnitude}. In comparison, resistance distances, diffusion distances, or Euclidean distance always permit the computation of magnitude. Because of this difference in the choice of distance metric, we note that our definition of the magnitude of a graph in \Cref{sec:mag_spread} differs from the definition used by \citet{leinster2019magnitude}. While we choose to investigate diffusion distances, we note that the distance metric can easily be replaced if needed to explore alternative geometries.

\paragraph{Diffusion geometry. } For further research, we believe that the usage of diffusion distances offers the chance to leverage a rich theory on approximation methods via landmarks \citep{long2019landmark}, or localised diffusion computations \citep{david2012hierarchical}, which can lead to further computational improvements and extension of our methods. 

\paragraph{Magnitude and spread as multi-scale functions.} Note that magnitude and spread can also be defined as multi-scale functions i.e.\ \(t \mapsto \text{Mag}((X,t\cdot d))\) for a metric space \((X, d)\) and a scale parameter \(t \in \mathbb{R}^{+}\). This parameter \(t\) can be likened to choosing a kernel bandwidth or the scale of distances or similarity determining when observations are considered to be distinct. In practical applications, it is advisable to carefully consider the choice of scaling factor \(t\) or the type of normalisation used to compare distances \citep{limbeck2024metric}. \citet{limbeck2024metric} propose a heuristic that %determine the most interesting scales for comparison by comparing magnitude functions from \(t=0\) until a value of \(t\) such that magnitude has reached a certain proportion of the cardinality. However, their approach uses 
uses root-finding to find a suitably large \(t\). However, this requires repeated computations of magnitude and  increases the computational costs. A faster and more desirable default choice of \(t\) would be based solely on the distance metric. 
%...
For distances that are not otherwise scaled or normalised, we therefore recommend the usage of faster heuristics, such as the median heuristic for choosing the kernel-bandwidth i.e. the scale parameter \(t\)
\citep{garreau2017large}.  For diffusion distances, we find that setting \(t=1\) is sufficient for our goals. This is because, as discussed in \Cref{sec:diffusion}, diffusion distances are computed from the normalised graph Laplacians and are inherently comparable across graphs. %Throughout our experiments, we choose to compute magnitude and spread at the one scale \(t=1\) to ensure computational efficiency. 
Nevertheless, investigating magnitude and spread as multi-scale functions on graphs remains an interesting extension for further work. %as such functions have the chance to be even more expressive graph invariant.

%\kat{Let's shorten the paragraph / the intro in the contribution section for the final draft:}

\paragraph{Magnitude and spread as diversity measures.} We extensively discuss the relationship between magnitude and spread throughout our work. However, our main paper does not have the space to fully explain the theoretical motivations behind the formulations of magnitude, spread, and other generalised measures of diversity. 
%\citet{leinster2021entropy} 
%we would  like to highlight extended references, which clarify the relationship between magnitude, spread, and diversity. 
For a more complete discussion %on the relationship between magnitude and spread, 
we refer the interested reader to \citet{leinster2021entropy}, an extensive reference work on the mathematical motivation behind entropy and diversity. Furthermore, \citet{willerton2015spread} specifically discusses the spread of a metric space, and %\citet{leinster2021entropy}, which also 
\citet{limbeck2024metric} describe the practical usage of magnitude as a diversity measure in ML. These works also give descriptions on how and why the magnitude or spread of a metric space can be interpreted as an effective size i.e. as the effective number of distinct points in a metric space or the number of dissimilar nodes in a graph.

%\paragraph{Similarity-based graph pooling. } SimPool uses the cosine similarity between rows in the adjacency matrix to propose a structure-based modification of DiffPool. The usage of diffusion distances in conjunction to magnitude is a novel approach to structural similarity for graph pooling, which fulfills ...

%\newpage
\subsection{Pooling Algorithm}
\label{app:algo}
We now detail our pooling algorithm introduced in \Cref{sec:mag_pooling} by describing a pseudocode implementation. Note that to describe the algorithm we assume we have pre-selected a distance metric for computing either magnitude or spread. %Further, we keep the choice of distance metric and feature aggregation fixed throughout our study and choose the average the features. Note however, that the feature reduction in line 20 of the algorithm could be easily replaced by summation or taking the maximum of features values, which could be of interest for further investigation. 

\begin{algorithm}[ht]
\caption{Graph Pooling Methods: SpreadEdgePool and MagEdgePool}
\begin{algorithmic}[1]
\REQUIRE input graph \( G = (X, E) \), node features \( \mathbf{F} \in \mathbb{R}^{|X| \times f} \), pooling ratio \( r \in (0, 1] \), diversity measure \( \text{Mag}(G) \text{ or } \text{Sp}(G) \)
\ENSURE Pooled graph \( G' =(X', E')\), pooled features \( \mathbf{F}' \)
\STATE Initialise the super-node set \( \mathcal{S}(x) \gets x \) for all \( x \in X \)
\STATE Initialise the set of edges adjacent to a contracted edge \( E_c \gets \emptyset \)
\STATE Initialise the pooled graph \(G' \gets G\)
\STATE Compute initial edge scores:
\[
s(e) = \left| \text{Mag}(G) - \text{Mag}(G / e) \right| \quad \forall e \in E
\]
\WHILE{\(|X'| \neq \lfloor r|X|\rceil \) AND \(|E'| \neq \emptyset\)}
    \STATE Select edge \( e = (x,y) = \arg\min_{e \in E \setminus E_c} s(e) \)
    \IF{\( e \) is not adjacent to any previously contracted edge in \( E_c \)}
        \STATE Contract edge \( e \), update \( G' \leftarrow G' / e \)
        \STATE Add \( e \) and any edges adjacent to \( e \) to \( E_c \)
        \STATE Update the node selection: merge \( \mathcal{S}(x) \) and \( \mathcal{S}(y) \)
    \ENDIF
    \IF{no more valid edges AND pooling ratio not reached}
        \STATE Recompute the edge scores \(s(e)\) on the updated graph \(G'\)
        \STATE Reset \( E_c \gets \emptyset \)
    \ENDIF
\ENDWHILE
\STATE Initialize \( \mathbf{F}' \gets \emptyset \)
\FOR{each super-node representative \( w \in \mathcal{S} \)}
    \STATE Let \( S_w = \{x \in X \mid \mathcal{S}(x) = w \} \)
    \STATE Compute the aggregated features:
    \[
    \mathbf{F}'_{w,:} = \frac{1}{|S_w|} \sum_{x \in S_w} \mathbf{F}_{x,:}
    \]
    \STATE Append \( \mathbf{F}'_w \) to \( \mathbf{F}' \)
\ENDFOR
\RETURN Pooled graph \( G' \), pooled features \( \mathbf{F}' \)
\end{algorithmic}
\label{alg:pool}
\end{algorithm}

%\subsubsection{Distinctions and Similarities to alternative pooling methods}
%Specifically, we would like to note that SimPool uses the cosine similarity between rows in the adjacency matrix to propose a modification of DiffPool. In contrast, we use for edge 

\newpage

\subsection{Extended Experimental Details}

We briefly describe extended details on the experimental setup used for our main experiments. 

\subsubsection{Overview Experiment}
\label{app:overview_experiment}

To create the illustration in \Cref{fig:overview_graphs} and visually compare the outputs of different pooling methods, we follow the experimental setup by \citet{grattarola2022understanding} on understanding structure preservation in graph pooling layers. Specifically, we  simulate a ring graph with 64 nodes, a barbell graph with 20 nodes on each side connected by 24 nodes in the middle, and a sensor graph with 64 nodes. Graphs are then pooled to a pooling ratio of approximately $50\%$. Because some pooling methods (e.g. Graclus or NDP) do not give exact control over the number of vertices, but pool graphs to approximately half of their original size, the number of nodes visualised in \Cref{fig:overview_graphs} can vary across methods. 
All trainable pooling layers were then trained in a self-supervised manner to optimise the following spectral loss between the original graph \(G\) and the pooled graph \(G'\):
\begin{equation}
    \mathcal{L}(G, G') = \sum_{i=0}^{f} F^\top_{:,i} L F_{:,i} - F'^\top_{:,i} L' F'_{:,i}
\end{equation}
where \(L, L'\) are the corresponding Laplacian matrices. The features \(F\) are taken to be the top 10 eigenvectors of \(L\) concatenated with the coordinates of the nodes in \(G\). \(F'\) is the reduced version of \(F\) after pooling. 
Note that this type of spectral loss is one particular proposal on structure preservation and alternative objectives could be investigated. 

\subsubsection{Structure Preservation Experiment and Pooling Ratios}

For reporting the structure preservation results described in \Cref{sec:structure}, we considered multiple different proposals for what it means to preserve graph structure during pooling. In the end, we settled to compare the spectral distance between the symmetrically normalised graph Laplacians as an established measure of spectral property preservation, and investigated the relative difference in magnitude between the original graph \(G\) and the pooled graph \(G'\) computed from diffusion distances. Specifically, the reported relative magnitude difference is calculated as \begin{equation}
    \text{MagDiff}(G, G') = \frac{|\text{Mag}(G) - \text{Mag}(G')|}{\text{Mag(G)}}.
\end{equation}
%Explaining the structure preservation measures

For this experiment we further vary the pooling ratios across different pooling methods. However, some of the pooling layers considered in our study (NDP, Graclus and NMF) were configured to always pool graphs to around half their size. To allow us to compare these methods across increasing pooling ratios, we choose to reapply these pooling operations repeatedly, which is why these three pooling methods are evaluated at pooling ratios that are powers of 0.5.

\subsubsection{Graph Classification Experiment}
\label{app:classification_details}

Our graph classification architecture follows the experimental setup described in \Cref{sec:classification} and is based the benchmark by  \citet{grattarola2022understanding}. 
Across our main classification experiment detailed in \Cref{sec:classification}, different pooling layers are configured to reduce each input graph to around 50\% of the number of nodes in the original graphs. 
Depending on the pooling method, this is chosen so each graph is reduced to 50\% of its original size \(k=\lfloor0.5*N\rceil\), (for NDP, Graclus, MagEdgePool, SpreadEdgePool, TopKPool, and SAGPool), or to %pool all graphs to 
50\% the average size of all graphs in the training dataset \(k=\lfloor0.5*\bar{N}\rfloor\) (for DiffPool, and MinCutPool). Interpreting the experimental results in \Cref{classification} it is thus of interest that the sizes of the pooled graphs can vary across pooling layers, which might explain some of the difference in performance between the fixed-size methods DiffPool and MinCUT compared to more adaptive pooling methods.

%Explaining the GNN architecture

\subsubsection{Graph Regression}
\label{app:reg}
As a further ablation study, we aim to assess whether the results reported in \Cref{classification} for graph classification tasks remain consistent for further graph regression tasks. To this end, we use three molecular datasets from the OGB benchmark with their predefined test, training, and validation splits~\citep{hu2020open}. Further, we adjust the GNN architecture described in \Cref{sec:classification} and \Cref{app:classification_details} to use the MSE as a training loss, the RMSE for performance evaluation, a linear final activation for the readout MLP, an early stopping patience of 100 epochs, %for early stopping, %. %, and a more hierarchical GNN architecture. The new GNN architecture consists of three blocks of two general convolutional layers with two intermediate pooling layers that sequentially pool the graphs to approximately half their size. 
%F%urther, we adjust the GNN architecture to use 
and blocks of two convolutional layers instead of single layers. Similar to before, we only use node features as inputs. 
Note that stronger, domain-specific, and purpose-built models exist for molecular regression tasks that utilise both atom and bond information \citep{hu2020open}. Hence, our goal is not to reach state of the art performance. Rather, we aim to compare the performance of different pooling layers and evaluate the information loss due to the pooling operations themselves. Finally, \Cref{regression-table} reports the RMSE on the test dataset across ten repeats using varying random seeds. %Overall, these results confirm that SpreadEgdePool and MagEdgePool constitute useful pooling operations reaching comparatively low RMSEs across these three regression tasks. Thus, the performance of our methods is better or comparable to using no pooling layer indicating their ability to preserve task-relevant information during pooling. 

%\newpage

\section{Extended Results}

Finally, we summarise extended experimental results beyond the scope of our main paper.

\subsection{Correlation between Magnitude and Spread}
\label{app:mag_spread_corr}

As stated in \Cref{sec:mag_spread}, the magnitude and spread of a metric space are closely related with magnitude giving an upper bound for spread when computed from the same positive definite metric space. Across our experiments on real graph datasets, we further find that this bound in practice can be very tight and magnitude and spread measure very related notions of effective size. More specifically, when computing both magnitude and spread from the diffusion distances detailed in \Cref{sec:diffusion}, we observe that magnitude and spread almost coincide for all graphs from the NCI1, ENZYME or IMDB-Multi datasets as illustrated in \Cref{fig:corr}. In fact, magnitude and spread correlate almost perfectly across these three graph datasets  (Pearson correlation \(r^2 \geq 0.99\)). Further, we confirm that across these examples, magnitude is generally greater or equal to spread by a relatively low multiplicative factor close to 1. We therefore find empirical evidence for the fact that spread offers a valid and highly related alternative to magnitude in practice supporting our theoretical analysis of the relationship between spread and magnitude during pooling (\Cref{app:proofs_mag}) as well as our observations on the similar performance of MagEdgePool and SpreadEdgePool. 

\begin{figure}[ht]
  \centering
  \includegraphics[trim={0 0 0cm 0},clip, width=0.3\textwidth]{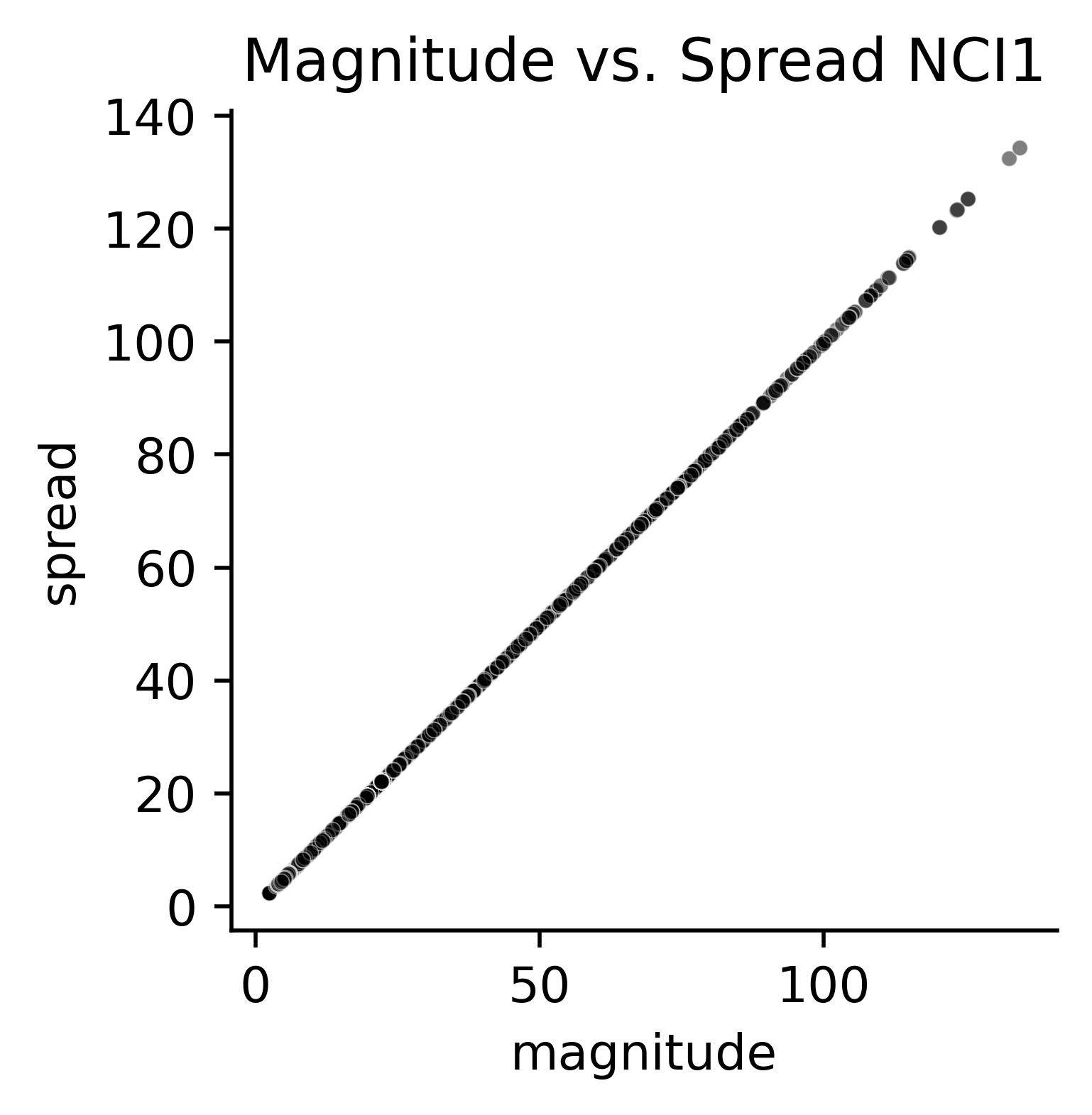}
  \includegraphics[trim={0 0 0cm 0},clip, width=0.315\textwidth]{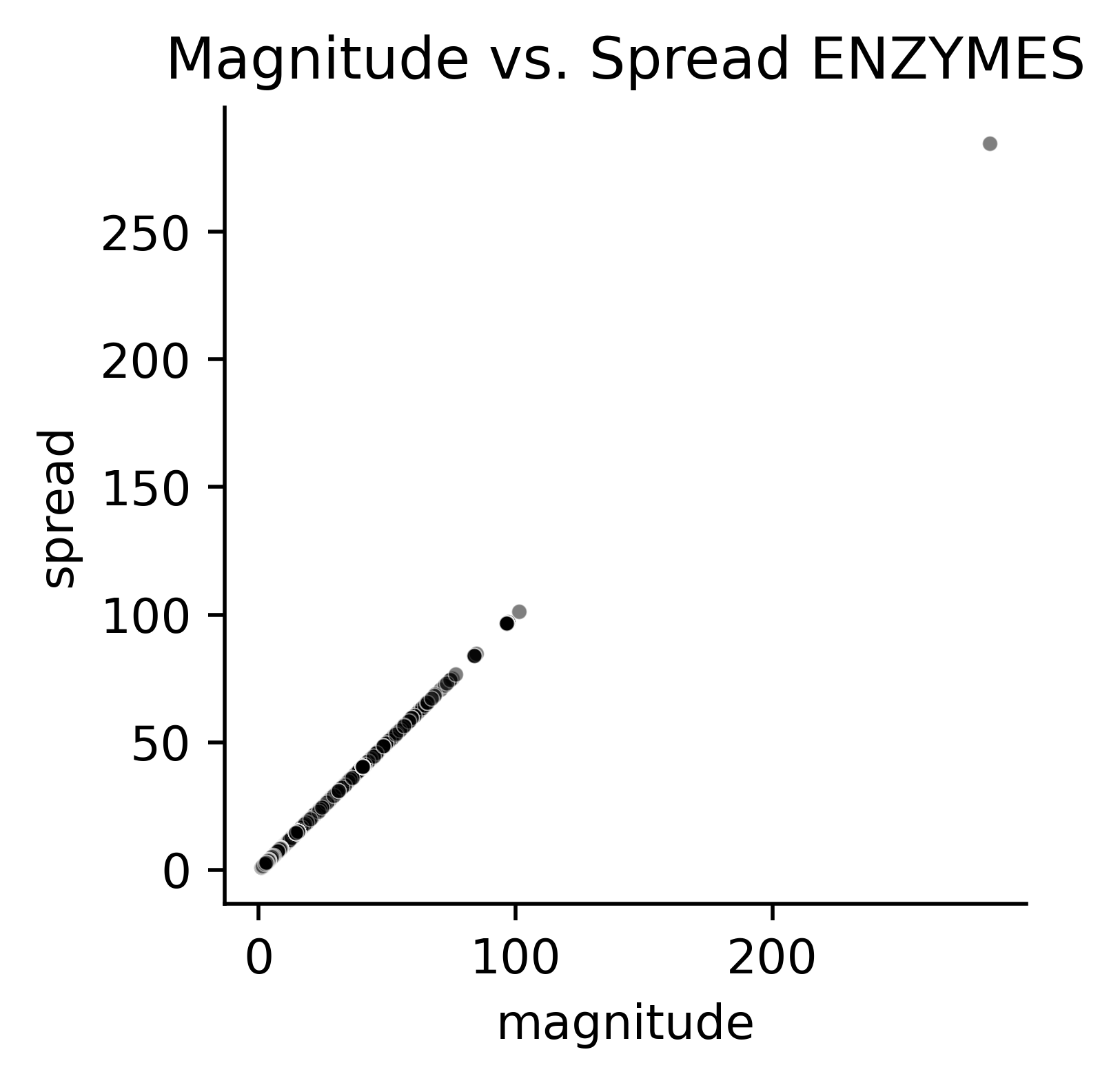}
  \includegraphics[trim={0 0 0cm 0},clip, width=0.325\textwidth]{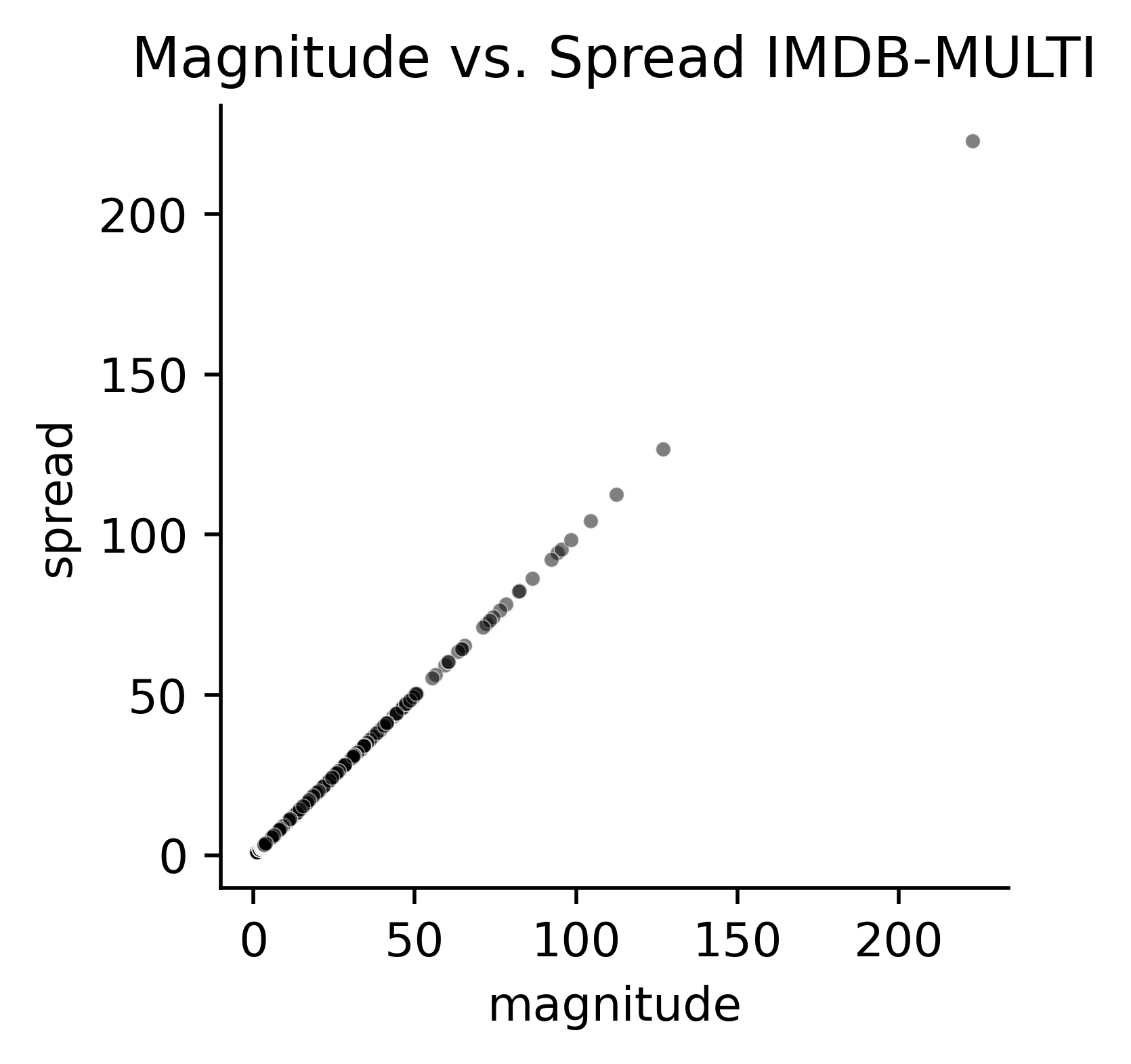}
  \includegraphics[trim={0 0 0cm 0},clip, width=0.3\textwidth]{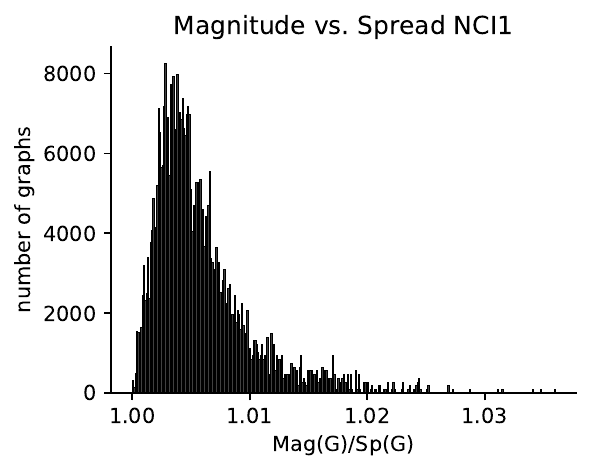}
\includegraphics[trim={0 0 0cm 0},clip, width=0.3\textwidth]{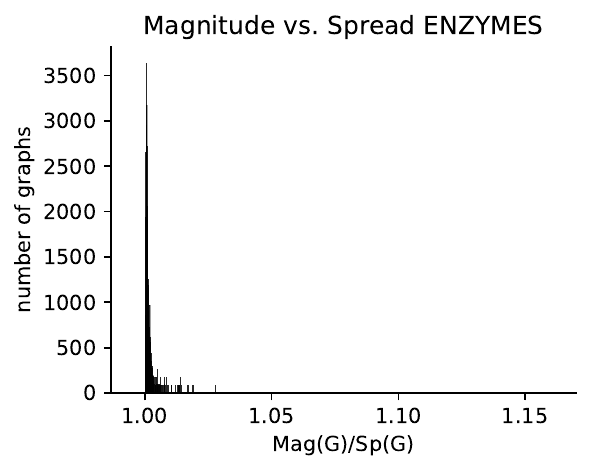}
  \includegraphics[trim={0 0 0cm 0},clip, width=0.3\textwidth]{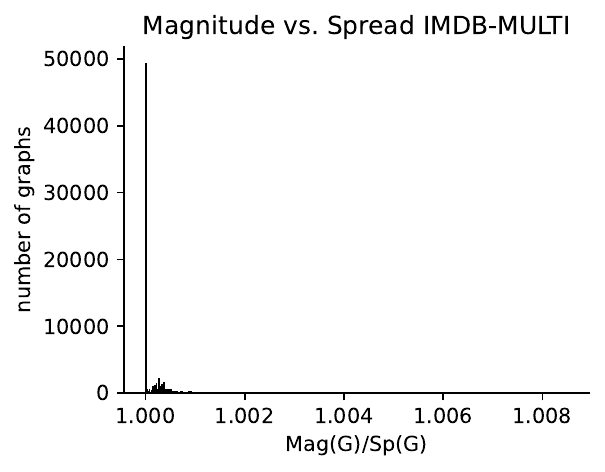}
  \caption{Comparison between magnitude and spread computed from diffusion distances for all graphs in three graph datasets, NCI1, ENZYMES and IMDB-Multi. }\label{fig:corr}
\end{figure}
\vspace{20pt}

\subsection{Evaluating Computational Efficiency}
\label{app:empirical_efficiency}

As detailed in \Cref{app:proofs_complexity}, the computational costs of our algorithm are determined by the costs of computing the edge scores used for pooling. To illustrate how this theoretical discussion translates into practice, we now investigate computational costs empirically in comparison to alternative pooling methods considered throughout this study. 

\subsubsection{Training Costs}
We first compare the runtime (in seconds) and memory usage (in MB per cross-validation run) of our pooling methods (MagEdgePool and SpreadEdgePool) to trainable pooling methods %(EdgePool, TopKPool, SAGPool, DiffPool and MinCutPool) 
in \Cref{tab:training_time}. 
Specifically, we train the GNN architecture specified in \Cref{app:classification_details} and \Cref{sec:classification} using GIN layers across 200 epochs using 10-fold stratified cross-validation. We compare our edge pooling methods (MagEdgePool, SpreadEdgePool) to trainable pooling methods from \texttt{torch\_geometric}\footnote{\url{https://pytorch-geometric.readthedocs.io/en/latest/} available under an MIT license~\citep{Fey2019, Fey2025}.}
(EdgePool, TopKPool, SAGPool) or from \texttt{torch-geometric-pool}\footnote{\url{https://github.com/tgp-team/torch-geometric-pool} available under an MIT license.} (DiffPool, MinCutPool). We record the mean and standard deviation of the runtimes in seconds in \Cref{tab:training_time} and GPU memory usage in MB per cross-validation run in \Cref{tab:memory}.
We note, in particular, that our proposed edge pooling methods, MagEdgePool and SpreadEdgePool, allow for significantly more efficient GNN training than EdgePool as highlighted in \Cref{tab:training_time}. Similarly, our methods generally improve on the runtimes for dense pooling methods such as DiffPool and MinCutPool. To generalise this runtime comparison to other datasets sizes and experimental setups, we note that our algorithm will scale with dataset size as described in \Cref{app:proofs_complexity}.

\subsubsection{Pre-training Costs }
%and to non-trainable ones (NDP, NMF and Graclus) in \Cref{tab:runtime_non_trainable} considering only pretraining in this case.
Having observed that pre-computed edge-pooling speeds up GNN training times, we further investigate the computational costs of this preprocessing step by comparing the runtime and memory costs of our methods against other non-trainable pooling operators (NDP, NMF and Graclus) prior to training. \Cref{tab:runtime_non_trainable} and \Cref{tab:memory_non_trainable} report the computational costs of computing the pooling assignment for all graphs in the datasets. 
We observe that SpreadEdgePool is generally more efficient than MagEdgePool. Beyond exact computations of our pooling methods (described in \Cref{app:algo}), \Cref{tab:runtime_non_trainable} and \Cref{tab:memory_non_trainable} also present approximate versions that reduce the cost of distance computations. %the results obtained by using 
Specifically, these approximate versions, referred to as  \textit{MagEdgePool*} and \textit{SpreadEdgePool*}, use 
the minimum distance to the original nodes to update the (diffusion) distances during edge contraction. This leads to a considerable improvement in runtime and memory usage during pre-training. The results reported in this section highlight one of the limitations of our edge pooling approach, namely, it scales in the number of edges as well as in the number of nodes as detailed in \Cref{app:proofs_complexity}. %\todo{what does size of the graph mean ? vertex set ?}. 
Nevertheless, our proposed pooling methods still outperform EdgePool in terms of computational efficiency when considering both pre-processing and training costs.

Furthermore, the pre-training costs of our method, range in a number of seconds, need to be computed only once per dataset and the memory requirements remain below what is required by GNN training. 
Hence, while the scalability to very large graphs is a limitation, we find that our proposed pooling methods scale sufficiently well to standard graph datasets. In practice, based on the computational complexity (\Cref{app:proofs_complexity}), we  recommend that our pooling method is particularly suitable for small to medium graphs that show a certain degree of sparsity rather than being fully connected. As visualised in \Cref{fig:graph_size}, graphs with up to a few hundred nodes can feasibly be processed with our method in a matter of seconds. For future work, we believe that there is a strong potential for adapting our methods to scale on large graphs. For instance, edge score calculations could be parallised,  sampling heuristics could restrict the edge score computations to a subset of candidate edges, or edge scores could be estimated from local subgraphs to improve the computations. %would constitute interesting avenues to explore.   

\begin{table}[h!]
    \caption{Training times in seconds compared across pooling methods. The fastest methods are marked in bold.}
    \label{tab:training_time}
    \centering
    \resizebox{1\linewidth}{!}{
    \begin{tabular}{lrrrrrr}
        \toprule
         \textbf{Method} & 
         \multicolumn{1}{c}{\textbf{DHFR}} &
        \multicolumn{1}{c}{\textbf{ENZYMES}} &
        \multicolumn{1}{c}{\textbf{NCI109}} &
        \multicolumn{1}{c}{\textbf{Mutagenicity}} &
        \multicolumn{1}{c}{\textbf{IMDB-BINARY}} &
        \multicolumn{1}{c}{\textbf{IMDB-MULTI}} \\
         \midrule
         \textbf{MagEdge} &39.8 ± 6.0 & \textbf{22.0 ± 0.2}& 186.0 ± 9.1 & \textbf{180.3 ± 29.0} & 75.4 ± 1.4 & \textbf{89.6 ± 6.1} \\
         \textbf{SpreadEdge } & \textbf{ 34.8 ± 1.0} & 23.0 ± 0.7 & \textbf{181.3 ± 29.4}& 206.7 ± 20.6& \textbf{62.7 ± 0.9}& 99.7 ± 4.3\\
         \midrule
         \textbf{EdgePool} & 184.1 ± 4.0 & 209.3 ± 4.1 & 807.4 ± 22.2 & 790.7 ± 22.1 & 362.3 ± 3.4 & 392.7 ± 11.9 \\
         \textbf{TopKPool} & 50.7 ± 0.8 & 24.1 ± 0.6 & 226.1 ± 21.3 & 243.8 ± 15.4 & 88.2 ± 1.1 & 110.9 ± 2.2\\
         \textbf{SAGPool} & 54.4 ± 1.3 & 25.7 ± 0.4 & 250.4 ± 25.7 & 253.0 ± 20.0 & 92.6 ± 2.1 &131.1 ± 1.3\\ \textbf{DiffPool}& 52.2 ± 6.8 & 36.4 ± 3.3 & 242.4 ± 12.5 & 274.0 ± 21.5 & 106.7 ± 5.5 & \textbf{74.9 ± 2.0}\\
         \textbf{MinCut}& \textbf{33.8 ± 2.5} & 28.2 ± 1.7 & 210.3 ± 34.4 & 201.0 ± 3.7 & 93.2 ± 1.6 & 132.8 ± 3.2\\
         \bottomrule
    \end{tabular}}
\end{table}

\begin{table}[h!]
\caption{Memory usage in MB compared across pooling methods. The most efficient methods are marked in bold.}
\label{tab:memory}
\centering
\resizebox{1\linewidth}{!}{
\begin{tabular}{lrrrrrr}
\toprule
    \textbf{Method} & 
    \multicolumn{1}{c}{\textbf{DHFR}} &
    \multicolumn{1}{c}{\textbf{ENZYMES}} &
    \multicolumn{1}{c}{\textbf{NCI109}} &
    \multicolumn{1}{c}{\textbf{Mutagenicity}} &
    \multicolumn{1}{c}{\textbf{IMDB-BINARY}} &
    \multicolumn{1}{c}{\textbf{IMDB-MULTI}} \\
\midrule
\textbf{MagEdge} & \textbf{84.0 ± 0.1} & \textbf{91.2 ± 2.5} & \textbf{97.0 ± 1.1} & \textbf{94.8 ± 1.0} & \textbf{283.4 ± 28.0} & \textbf{197.6 ± 17.8} \\ 
\textbf{SpreadEdge} & \textbf{84.0 ± 0.1} & \textbf{90.6 ± 2.3} & \textbf{96.8 ± 1.0} & \textbf{95.2 ± 1.0} & \textbf{283.4 ± 28.0} & \textbf{197.4 ± 17.8} \\ 
\midrule
\textbf{EdgePool} & 125.6 ± 10.4 & 148.4 ± 15.0 & 118.6 ± 7.7 & 118.0 ± 7.4 & 666.6 ± 106.6 & 526.2 ± 89.2 \\ 
\textbf{TopKPool} & 89.6 ± 9.0 & 94.4 ± 4.0 & 105.0 ± 1.4 & 105.8 ± 6.5 & \textbf{259.4 ± 31.5} & \textbf{193.8 ± 17.5} \\ 
\textbf{SAGPool} & 104.2 ± 0.6 & 94.2 ± 3.8 & 105.2 ± 1.4 & 107.8 ± 8.1 & \textbf{273.4 ± 30.1} & \textbf{209.0 ± 17.9} \\ 
\textbf{DiffPool} & 90.2 ± 10.2 & 94.6 ± 2.5 & 103.8 ± 1.1 & 299.8 ± 31.9 & \textbf{258.0 ± 40.3} & 240.8 ± 26.4 \\ 
\textbf{MinCut} & 90.6 ± 10.3 & 94.2 ± 2.4 & 103.6 ± 1.0 & 288.4 ± 38.3 & \textbf{258.0 ± 40.3} & \textbf{196.6 ± 26.0} \\ 
\bottomrule
\end{tabular}}
\end{table}

\begin{table}[h!]
\centering
\caption{Pre-training times in seconds compared across non-trainable pooling methods. The fastest method is marked in bold. The fastest approximation of our pooling method is marked in italics.}
\label{tab:runtime_non_trainable}
\resizebox{1\linewidth}{!}{
\begin{tabular}{lrrrrrr}
\toprule
    \textbf{Method} & 
    \multicolumn{1}{c}{\textbf{DHFR}} &
    \multicolumn{1}{c}{\textbf{ENZYMES}} &
    \multicolumn{1}{c}{\textbf{NCI109}} &
    \multicolumn{1}{c}{\textbf{Mutagenicity}} &
    \multicolumn{1}{c}{\textbf{IMDB-BINARY}} &
    \multicolumn{1}{c}{\textbf{IMDB-MULTI}} \\
\midrule
\textbf{MagEdge} & 61.9 & 197.4 & 4765.1 & 678.8 & 801.8 & 480.2 \\ 
\textit{\textbf{MagEdge*}} & 24.6 & 21.0 & 81.3 & 74.9 & 77.4 & 69.9 \\ 
\textbf{SpreadEdge} & 66.7 & 68.3 & 316.1 & 208.8 & 287.2 & 243.8 \\ 
\textit{\textbf{SpreadEdge*}} & \textit{10.9} & \textit{16.0} & \textit{52.3} & \textit{55.6} & \textit{56.0} & \textit{66.0} \\
\midrule
\textbf{NDP} & 5.6 & 4.1 & \textbf{37.1} & 32.6 & 5.8 & 7.0 \\ 
\textbf{NMF} & 7.8 & 10.8 & 39.0 & 32.1 & 8.7 & 9.0 \\ 
\textbf{Graclus} & \textbf{3.8} & \textbf{2.9} & 54.5 & \textbf{18.2} & \textbf{4.5} & \textbf{6.0} \\
\bottomrule
\end{tabular}}
\end{table}

\begin{table}[h!]
\centering
\caption{Pre-training memory usage in MB compared across non-trainable pooling methods. The most efficient method is marked in bold. The most efficient approximation of our pooling method is marked in italics.}
\label{tab:memory_non_trainable}
\resizebox{1\linewidth}{!}{
\begin{tabular}{lrrrrrr}
\toprule
    \textbf{Method} & 
    \multicolumn{1}{c}{\textbf{DHFR}} &
    \multicolumn{1}{c}{\textbf{ENZYMES}} &
    \multicolumn{1}{c}{\textbf{NCI109}} &
    \multicolumn{1}{c}{\textbf{Mutagenicity}} &
    \multicolumn{1}{c}{\textbf{IMDB-BINARY}} &
    \multicolumn{1}{c}{\textbf{IMDB-MULTI}} \\
    \midrule
\textbf{MagEdgePool} & 76.2 & 163.8 & 158.9 & 179.0 & 162.1 & 93.1 \\ 
\textbf{\textit{MagEdgePool*}} & 83.7 & \textit{34.8} & \textit{47.0} & \textit{72.2} & \textit{11.5} & 65.2 \\ 
\textbf{SpreadEdgePool} & 61.9 & 60.1 & 177.9 & 93.1 & 204.1 & 75.8 \\ 
\textbf{\textit{SpreadEdgePool*}} & \textit{33.6} & 54.1 & 60.5 & 72.9 & 53.2 & \textit{23.8} \\ 
\midrule
\textbf{NDP} & \textbf{3.0} & \textbf{5.0} & \textbf{7.0} & 13.4 & 5.1 & 6.5 \\ 
\textbf{NMF} & 3.8 & 5.9 & 39.9 & \textbf{10.5} & 4.2 & 5.9\\
\textbf{Graclus} & 15.8 & 7.5 & 17.5 & 59.5 & \textbf{1.0} & \textbf{1.5}\\
\bottomrule
\end{tabular}}
\end{table}

%\newpage

%Furthermore, the pre-training costs of our method, range in a number of seconds, need to be computed only once per dataset and the memory requirements remain below what is required by GNN training. 
%Hence, while the scalability to very large graphs is a limitation, we find that our proposed pooling methods scale sufficiently well to standard graph datasets. In practice, based on the computational complexity (\Cref{app:proofs_complexity}), we  recommend that our pooling method is particularly suitable for small to medium graphs that show a certain degree of sparsity rather than being fully connected. As visualised in \Cref{fig:graph_size}, graphs with up to a few hundred nodes can feasibly be processed with our method in a matter of seconds. For future work, we believe that there is a strong potential for adapting our methods to scale on large graphs. For instance, edge score calculations could be parallised,  sampling heuristics could restrict the edge score computations to a subset of candidate edges, or edge scores could be estimated from local subgraphs to improve the computations. %would constitute interesting avenues to explore.   

\begin{figure}[t]
    \label{fig:graph_size}
    \centering
    \includegraphics[width=0.48\linewidth]{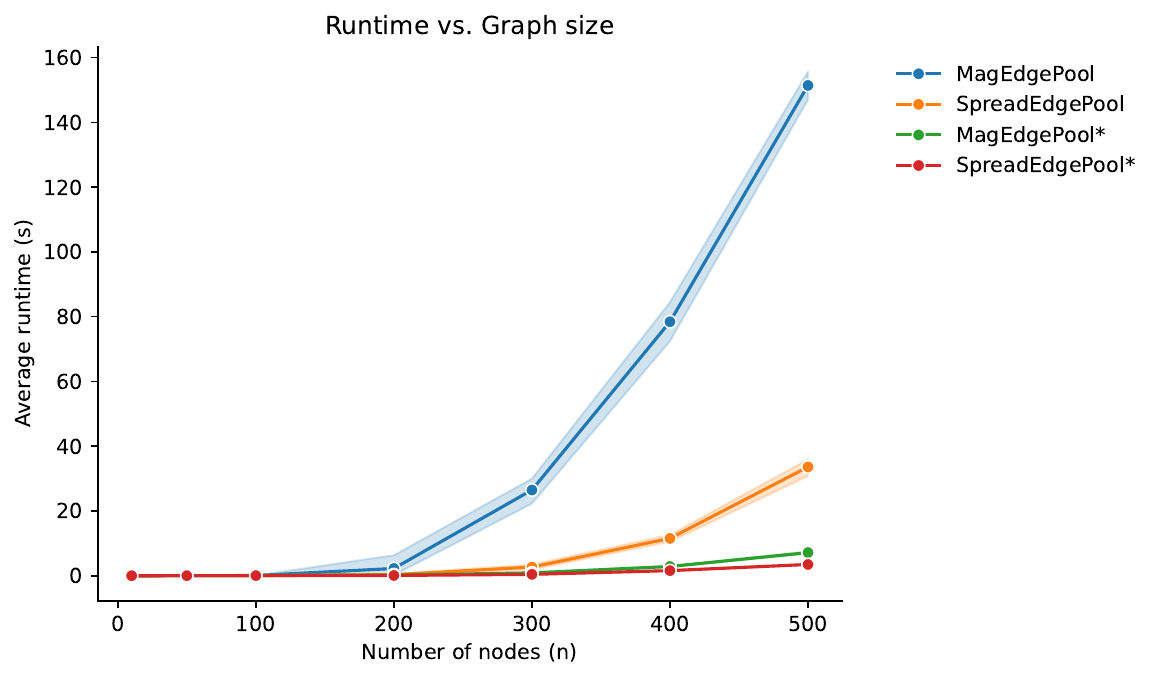}
\includegraphics[width=0.48\linewidth]{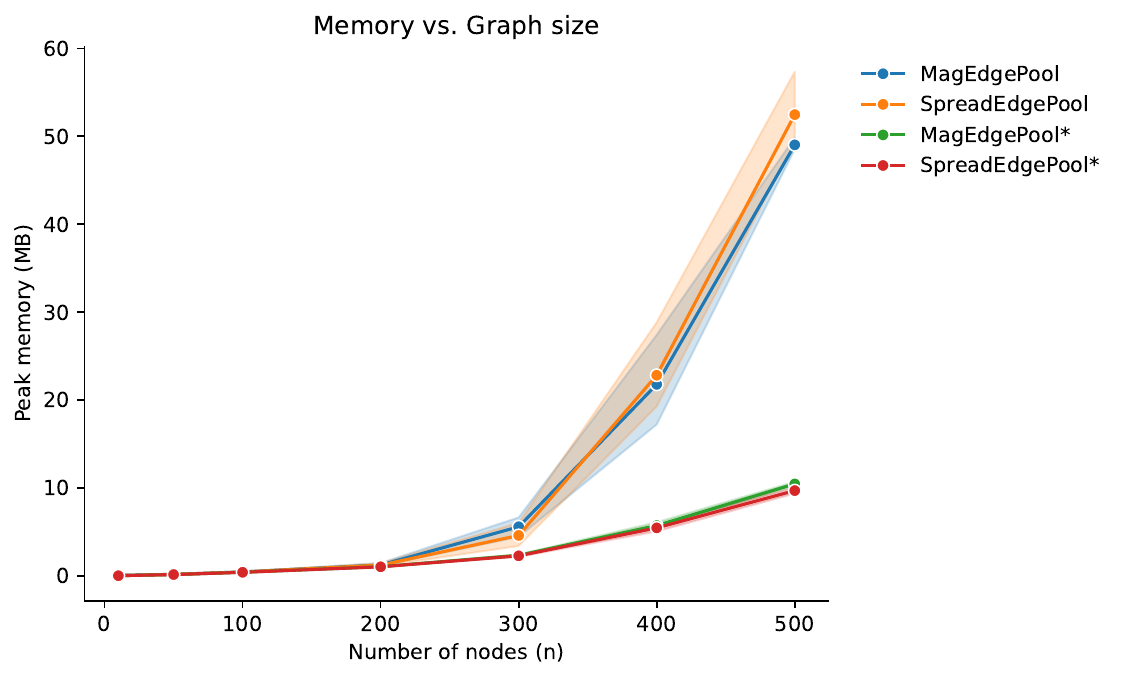}
    \caption{Runtime and memory costs of computing the pooling assignment for Erdős-Rényi  graphs with edge probability $0.005$ for increasing numbers of nodes. Lines show the mean and shaded areas the standard deviation across five repeats. SpreadEdgePool is notably faster than MagEdgePool. Distance approximations (marked with *) further speed up computations.}
\end{figure}

\subsubsection{Runtimes across Pooling Ratios }
Further, we compare the runtimes of training the models used in \Cref{sec:ratio_acc} to compare the accuracy of different pooling methods across varying pooling ratios. \Cref{fig:runtimes} then reports the mean runtime of training the GNN on one CV-fold for different choices of pooling ratios and pooling methods for the NCI1 dataset using either general convolutional layers or GIN layers. All models are trained as specified in \Cref{sec:classification} on a single GPU with 32~GB memory. 
Notably, we observe that  MagEdgePool and SpreadEdgePool overall perform on par with alternative pooling methods in terms of runtimes. SpreadEdgePool has a consistent advantage over MagEdgePool due to the higher computational efficiency of computing spread rather than magnitude. Notice that for increasing pooling ratios, our algorithm re-computes the edge scores repeatedly, leading to a less pronounced decrease in computational costs than alternative methods. Nevertheless, %when pooling no more than 50\%, we see that SpreadEdgePool  performs very efficiently surpassing other non-trainable approaches. 
we conclude that it is generally more efficient to apply SpreadEdgePool than to rely on trainable approaches, such as TopK and SAGPool for this dataset, indicating the computational benefit of non-trainable graph pooling operations.

\begin{figure}[t]
    \label{fig:runtimes}
    \centering
    \includegraphics[trim={0 0 5cm 0},clip, width=0.4\linewidth]{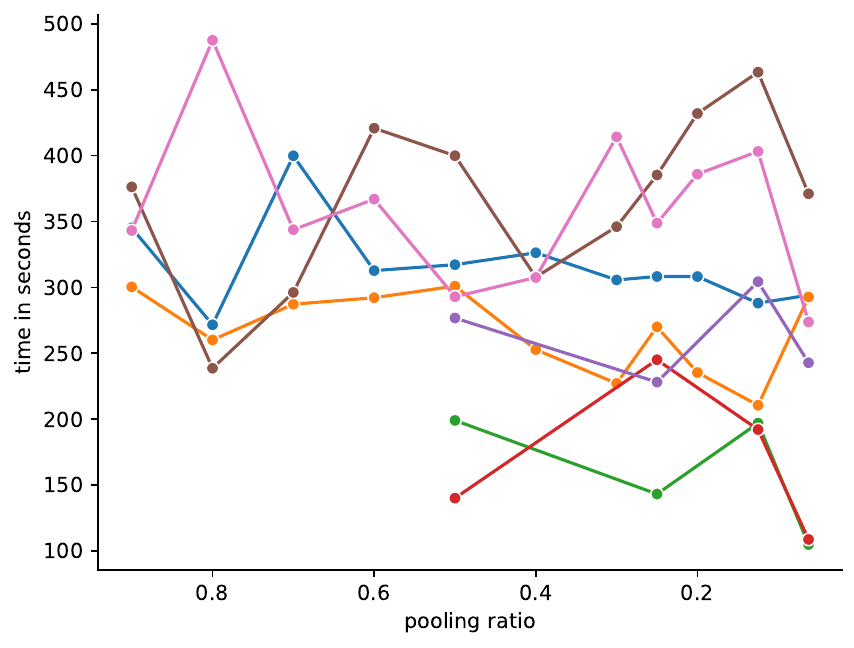}
    \includegraphics[trim={0 0 5cm 0},clip, width=0.4\linewidth]{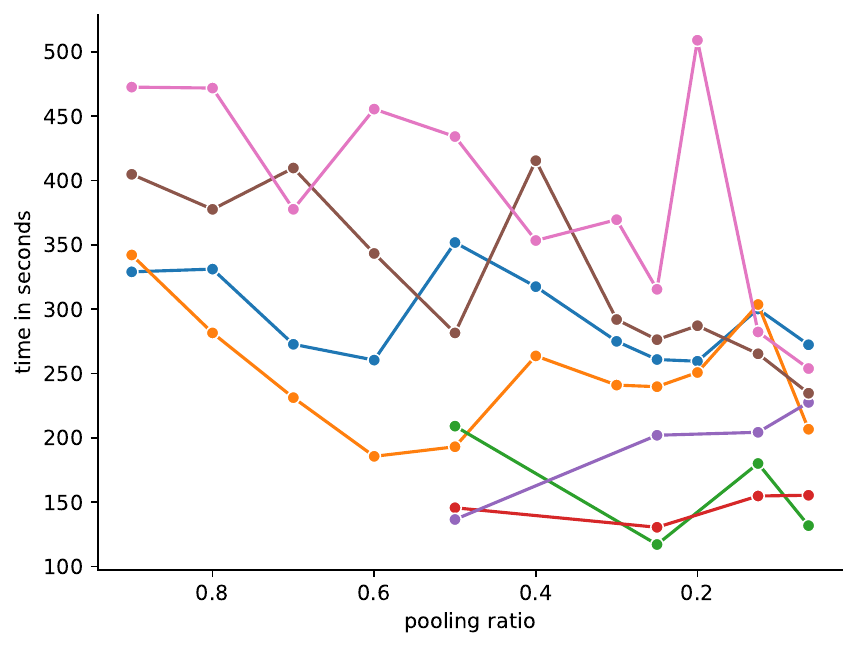}
    \raisebox{0.6\height}{\includegraphics[trim={0 0 0cm 0},clip, width=0.16\linewidth]{plots/structure_properties/mag_spectral_legend.pdf}}
     \caption{Runtime comparison for training the GNNs reported in \Cref{sec:classification} for NCI1 using different pooling layers. Plots show the mean time in seconds per run using general convolutional layers (left) or GIN layers (right).}
\end{figure}

%\newpage

\subsection{Node Feature Preservation and Expressivity}
\label{app:node_features}

%\textcolor{purple}{Lydia : optional/remove ?}

Graph pooling should not only preserve graph structure, but also preserve relevant node feature information during pooling. That is, pooling is frequently used after initial rounds of message passing and data representations learnt by previous layers should be respected and effectively encoded by the pooling procedure \citep{grattarola2022understanding}. One way of investigating node feature retention during pooling, is to evaluate how well a graph can be reconstructed from its pooled version. We follow the experimental setup by \citet{grattarola2022understanding} to investigate. In particular, this experiment uses a model architecture, similar to the model proposed in \Cref{sec:classification}, where each graph gets pooled to around 50\% of nodes after the initial MLP and GNN layer. Then, the pooled graphs are up-scaled again by reversing the node selection step used by each pooling layer. From these unpooled graph representations, a further GNN and post-processing MLP layer are trained and the task is set to output the reconstructed node feature representation. This model is trained on each example graph using Adam to minimize the mean squared error (MSE)
between the input and output node features with a learning rate of $0.0005$ and early stopping on the training
loss with a patience of 1000 epochs and a tolerance of \(10^{-6}\). Each experiment is repeated three times across different random seeds. See \citet{grattarola2022understanding} for further explanations on the model architecture and experimental setup.

For this experiment, we expect SpreadEdgePool to perform comparably well as we specifically designed our pooling algorithm so that features are averaged during pooling. Further restricting the number of times a node can be merged effectively prevents the collapse of entire portions of the graph, which aids reconstruction. SpreadEdgePool pooling thus successfully encodes node information while allowing for a flexible choice of pooling ratio. This is confirmed by the results in \Cref{ae_table}, which highlight that SpreadEdgePool overall performs well at the features reconstruction task, especially for the sensor graph reaching low reconstruction errors. Alternative methods, in particular node drop approaches such as TopK and SAGPool, show notably worse feature preservation during this experiment indicating the benefits of more expressive pooling operations, such as SpreadEdgePool. 

Note that the results in  \Cref{ae_table} capture one specific aspect of feature preservation, namely how well the features of specific example graphs can be reconstructed. This experiment does not assess the generalisation capability of pooling layers. Further, the experimental setup assumes that it is relevant to preserve all node features during pooling, which might not be realistic in practice, where the aim of pooling could be to solely encode task-relevant feature representations. Nevertheless, as discussed above, this extended experiment gives evidence to support that our proposed pooling algorithm, SpreadEdgePool, is capable of outputting expressive feature representations and aggregates node features in a faithful manner, which is likely one of the reasons for its high performance in graph classification tasks.

\begin{table}[t]
\caption{Mean and standard deviation of the reconstruction MSE for reconstructing the original node positions from the pooled graph representations for different example graphs and pooling methods. Our proposed algorithm, SpreradEdgePool, does well at faithfully encoding the feature representations.}
  \label{ae_table}
\resizebox{\columnwidth}{!}{
\begin{tabular}{lllllll}
\toprule
                    & \textbf{Ring}       & \textbf{Sensor}     & \textbf{Barbell}    & \textbf{Community}  & \textbf{Erdős–Rényi} & \textbf{Torus}      \\ \midrule 
\textbf{SpreadEdge} & 5.47e-07 ± 2.63e-07 & 2.78e-05 ± 3.04e-07 & 3.42e-04 ± 1.42e-06 & 3.71e-03 ± 5.80e-05 & 6.49e-07 ± 3.33e-07 & 5.29e-07 ± 1.33e-07 \\ \midrule
\textbf{NDP}        & 3.08e-07 ± 3.57e-07 & 4.07e-05 ± 4.57e-06 & 4.54e-04 ± 2.01e-05 & 2.52e-01 ± 9.50e-06 & 1.46e-06 ± 1.18e-06 & 5.68e-07 ± 1.02e-07 \\
\textbf{Graclus}    & 6.87e-04 ± 7.56e-07 & 2.67e-06 ± 2.31e-06 & 1.82e-03 ± 3.22e-07 & 2.42e+00 ± 2.11e-04 & 4.76e-02 ± 3.75e-07 & 7.10e-07 ± 8.60e-08 \\
\textbf{NMF}        & 4.78e-07 ± 2.95e-07 & 1.96e-05 ± 1.39e-05 & 5.80e-04 ± 4.53e-07 & 6.06e-01 ± 1.43e-04 & 5.04e-07 ± 3.31e-07 & 2.52e-07 ± 2.93e-07 \\
\textbf{TopK}       & 1.21e-01 ± 8.23e-03 & 5.83e-03 ± 2.16e-03 & 1.55e-02 ± 1.10e-02 & 6.03e+00 ± 2.21e+00 & 5.30e-03 ± 7.49e-03 & 1.72e-01 ± 8.51e-03 \\
\textbf{SAGPool}    & 1.45e-01 ± 2.52e-02 & 2.01e-03 ± 2.73e-03 & 4.12e-02 ± 4.18e-02 & 4.76e+00 ± 1.57e+00 & 9.60e-05 ± 1.35e-04 & 1.88e-01 ± 4.63e-02 \\
\textbf{DiffPool}   & 8.63e-06 ± 4.73e-06 & 3.50e-04 ± 8.32e-05 & 6.50e-04 ± 1.01e-06 & 2.14e-01 ± 2.84e-01 & 3.74e-04 ± 1.46e-04 & 5.17e-05 ± 8.90e-06 \\
\textbf{MinCut}     & 2.55e-06 ± 2.68e-06 & 6.56e-06 ± 3.86e-06 & 2.35e-06 ± 1.50e-06 & 1.80e-04 ± 2.01e-04 & 1.44e-06 ± 5.24e-07 & 1.49e-06 ± 9.81e-07 \\ \midrule
\end{tabular}
}
\end{table}

%\newpage

\begin{figure}[t]
  \centering
  \includegraphics[trim={0 0 0cm 0},clip, width=0.32\textwidth]{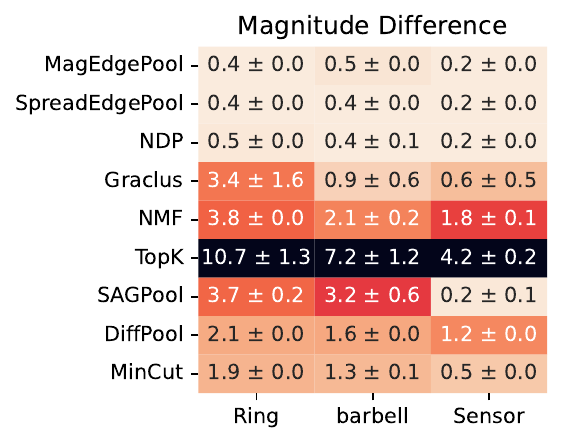}
  \includegraphics[trim={0 0 0cm 0},clip,width=0.32\textwidth]{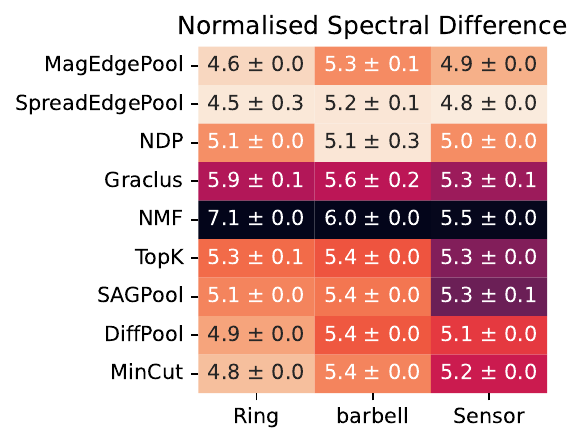}  
  \includegraphics[trim={0 0 0cm 0},clip,width=0.32\textwidth]{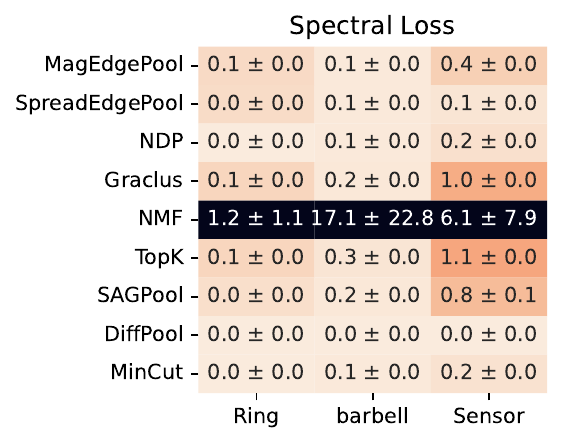}
  \caption{Structure preservation measures for the examples in \Cref{fig:overview_graphs}. Pooling is repeated for three different random seeds and the annotations report the means and standard deviations of the structure preservation scores.} \label{fig:heatmaps}
\end{figure}

\subsection{Overview Experiment}
\label{app:overview_experiment_results}

Expanding on the qualitative comparison between the example graphs in \Cref{fig:overview_graphs}, \Cref{fig:heatmaps} shows quantitative structure preservation measures for all example graphs and pooling methods. Specifically, we summarise the spectral loss by \citet{grattarola2022understanding}, the spectral distance between the normalised graph Laplacians, and the magnitude difference between the pooled and original graphs as further detailed in \Cref{sec:structure}.

\Cref{fig:heatmaps} demonstrates that SpreadEdgePool and MagEdgePool do not only reach low magnitude differences across these three example graphs, they also show comparatively low spectral distances supporting our findings in \Cref{sec:structure}. Further, we observe  that methods that show worse visual preservation of graph structures, such as NMF, TopK, DiffPool, and MinCut, also reach consistently higher magnitude differences and spectral distances supporting our claim that these methods fail to faithfully preserve graph structures during pooling to varying extents.

%\newpage

\subsection{Graph Classification}
\label{app:extended_classification}

\Cref{extended-class} further reports extended classification results on additional datasets extending on the results shown in \Cref{classification}. We chose not to report on these datasets in the main text because they showed fewer and less notable differences between pooling methods. Note that for the open graph benchmark MolHIV dataset, instead of using stratified cross-validation, we evaluate each model across predefined training, test and validation splits and evaluate their performance across 5 random seeds. Further, we report AUROC as the performance metric for MolHIV because is the suggested evaluation metric for this very imbalanced dataset. All other results are reported as in \Cref{classification} via the mean and standard deviation of the test accuracy across 10-fold stratified cross-validation. In agreement with our main results, we observe that MagEdgePool and SpreadEdgePool constitute high performing general-purpose pooling methods that reach top performance across these extended datasets. In particular, for smaller biological datasets, such as MUTAG, BZR, or COX2, pooling via magnitude or spread notably improves on the GNN that uses no pooling layer, which indicates the beneficial effects of structure-aware pooling for graph learning.

In order to further clarify the classification performance comparison between methods we visualise the ranking of each model choice via a critical difference diagram. This visualisation uses the Friedman test, a non-parametric test for the performance difference between multiple classifiers with repeat measurements, followed by the Nemenyi post-hoc test that facilitates the comparison across all classifiers \citep{demvsar2006statistical}. In our case, we use the Python package \texttt{scikit-posthocs} \citep{terpilowski2019scikit} to apply this test to the mean performance scores across datasets. The critical difference plot then links all methods that are found to not be statistically different by a horizontal bar. The results reported in \Cref{fig:critical} confirm that our methods are amongst the best-performing group of pooling layers for the datasets reported in \Cref{classification}. Similarly, \Cref{fig:critical_all} shows the %mean rank and 
critical difference diagram and mean ranks across all tasks included in \Cref{classification} or \Cref{extended-class}. This extended ranking further strengthens our finding that MagEdgePool and SpreadEdgePool reach top performance across classification tasks.

% Required packages:
% Beamer presentation requires \usepackage{colortbl} instead

\begin{table}[t]
  \caption{Classification performance of different pooling layers across multiple datasets. For each dataset, the best performing model is marked in bold and models that do not perform significantly different from the best performing model are coloured green.}
  \label{extended-class}
  \centering
  \resizebox{\columnwidth}{!}{
  \begin{tabular}{lccccccccc}
    \toprule
    \textbf{Method} & \textbf{MolHIV (AUROC)} & \textbf{MUTAG} & \textbf{COX2} & \textbf{BZR} & \textbf{BZR\_MD} & \textbf{COX2\_MD} & \textbf{DHFR\_MD} & \textbf{ER\_MD} & \textbf{AIDS} \\
    \midrule
    \textbf{No Pooling} & {\color[HTML]{747474} 74.4 ± 0.7} & {\color[HTML]{747474} 81.6 ± 6.3} & {\color[HTML]{747474}83.8 ± 3.8} & {\color[HTML]{747474}76.1 ± 9.1} & {\color[HTML]{747474} 73.7 ± 5.3} & {\color[HTML]{747474}70.2 ± 8.1 } & {\color[HTML]{747474}71.3 ± 1.9} & {\color[HTML]{747474}74.9 ± 0.8} & {\color[HTML]{747474}99.0 ± 0.1} \\
    \midrule
    \textbf{MagEdge} & {\color[HTML]{00B050} 64.6 ± 8.7} & {\color[HTML]{00B050} 84.0 ± 7.4} & {\color[HTML]{00B050} 86.0 ± 7.3} & {\color[HTML]{00B050} \textbf{88.1 ± 10.0}} & 65.9 ± 2.2 & {\color[HTML]{00B050} 76.5 ± 9.2} & {\color[HTML]{00B050} \textbf{77.9 ± 9.4}} & {\color[HTML]{00B050} 79.9 ± 7.1} & {\color[HTML]{00B050} \textbf{99.7 ± 0.1}} \\
    \textbf{SpreadEdge} & {\color[HTML]{00B050} 70.1 ± 2.9} & {\color[HTML]{00B050} 85.9 ± 7.4} & {\color[HTML]{00B050} 85.1 ± 5.6} & {\color[HTML]{00B050} 87.4 ± 9.5} & {\color[HTML]{00B050} 69.7 ± 10.0} & {\color[HTML]{00B050} \textbf{77.1 ± 9.0}} & {\color[HTML]{00B050} 74.6 ± 11.1} & {\color[HTML]{00B050} 83.1 ± 4.0} & {\color[HTML]{00B050} \textbf{99.7 ± 0.1}} \\
    \midrule
    \textbf{NDP} & {\color[HTML]{00B050} 65.2 ± 7.3} & {\color[HTML]{00B050} \textbf{91.6 ± 2.4}} & {\color[HTML]{00B050} \textbf{86.1 ± 7.1}} & {\color[HTML]{00B050} 86.3 ± 9.4} & {\color[HTML]{00B050} 72.3 ± 11.3} & {\color[HTML]{00B050} 73.9 ± 10.6} & {\color[HTML]{00B050} 72.2 ± 11.4} & {\color[HTML]{00B050} \textbf{84.2 ± 1.7}} & 99.6 ± 0.1 \\
    \textbf{Graclus} & {\color[HTML]{00B050} 70.1 ± 5.3} & {\color[HTML]{00B050} 87.4 ± 9.5} & {\color[HTML]{00B050} 79.2 ± 12.5} & {\color[HTML]{00B050} 80.1 ± 7.1} & {\color[HTML]{00B050} 72.0 ± 9.6} & {\color[HTML]{00B050} 75.4 ± 9.9} & {\color[HTML]{00B050} 71.1 ± 12.3} & {\color[HTML]{00B050} 81.7 ± 3.9} & 99.5 ± 0.1 \\
    \textbf{NMF} & {\color[HTML]{00B050} 73.2 ± 1.3} & {\color[HTML]{00B050} 84.0 ± 9.4} & {\color[HTML]{00B050} 83.8 ± 8.7} & {\color[HTML]{00B050} 80.8 ± 8.5} & {\color[HTML]{00B050} 69.0 ± 10.2} & {\color[HTML]{00B050} 70.4 ± 6.6} & {\color[HTML]{00B050} 70.8 ± 8.6} & 80.9 ± 2.0 & 97.0 ± 1.7 \\
    \textbf{TopK} & {\color[HTML]{00B050} 72.7 ± 1.8} & 81.9 ± 3.9 & {\color[HTML]{00B050} 76.4 ± 7.2} & {\color[HTML]{00B050} 75.8 ± 8.0} & {\color[HTML]{00B050} 68.4 ± 8.5} & 65.9 ± 7.0 & {\color[HTML]{00B050} 69.5 ± 3.3} & 74.0 ± 1.0 & 99.3 ± 0.1 \\
    \textbf{SAGPool} & {\color[HTML]{00B050} \textbf{74.3 ± 3.2}} & 82.9 ± 2.3 & 76.8 ± 7.8 & {\color[HTML]{00B050} 77.8 ± 5.0} & {\color[HTML]{00B050} 66.8 ± 7.0} & {\color[HTML]{00B050} 67.0 ± 6.4} & {\color[HTML]{00B050} 70.5 ± 2.6} & 75.6 ± 1.2 & 99.0 ± 0.1 \\
    \textbf{DiffPool} & {\color[HTML]{00B050} 72.1 ± 1.0} & 83.3 ± 2.8 & {\color[HTML]{00B050} 76.9 ± 6.5} & {\color[HTML]{00B050} 76.8 ± 10.1} & {\color[HTML]{00B050} 72.3 ± 3.0} & {\color[HTML]{00B050} 69.2 ± 2.1} & {\color[HTML]{00B050} 67.9 ± 2.6} & 73.4 ± 1.1 & 98.8 ± 0.2 \\
    \textbf{MinCut} & {\color[HTML]{00B050} 70.3 ± 3.0} & 80.6 ± 3.7 & {\color[HTML]{00B050} 79.1 ± 4.4} & 70.8 ± 8.6 & {\color[HTML]{00B050} 69.0 ± 5.0} & {\color[HTML]{00B050} 70.0 ± 3.7} & {\color[HTML]{00B050} 69.7 ± 1.9} & 73.3 ± 1.6 & 99.2 ± 0.1 \\
    \bottomrule
  \end{tabular}}
\end{table}

\begin{figure}[h!]
  \centering
  \includegraphics[width=0.9\textwidth]{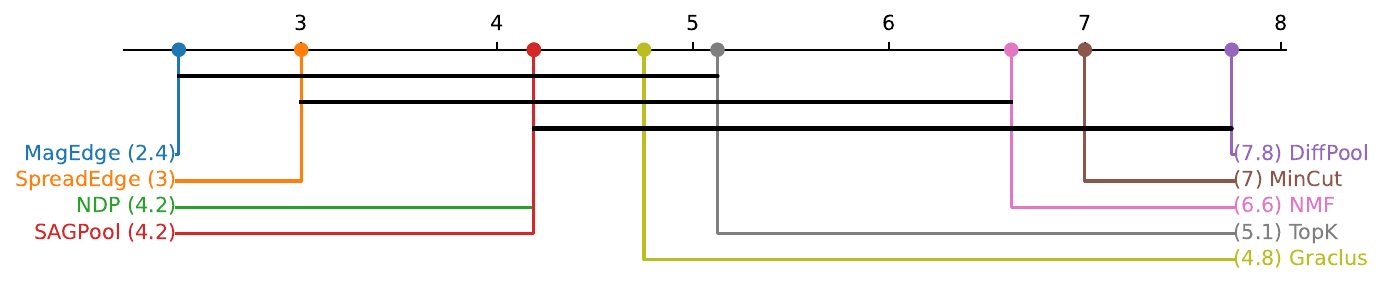}
  \caption{Critical difference diagram for the classification results from \Cref{classification}. Each label corresponds to a choice of pooling method and the method's mean rank across classification tasks. Groups of methods that are found to not be statistically different are linked by horizontal bars.} \label{fig:critical}
\end{figure}

\begin{figure}[h!]
  \centering
  \includegraphics[width=0.9\textwidth]{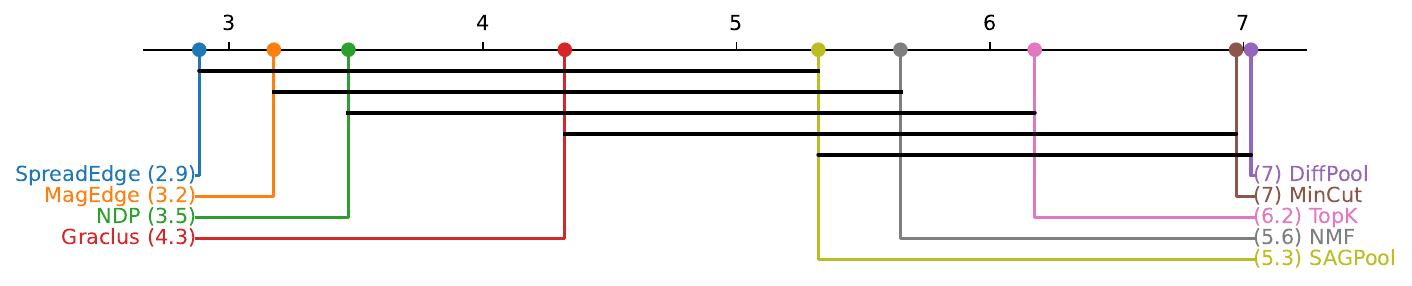}
  \caption{Critical difference diagram for the classification results from \Cref{classification} and \Cref{extended-class}. Each label corresponds to a choice of pooling method and the method's mean rank across classification tasks. Groups of methods that are found to not be statistically different are linked by horizontal bars.} \label{fig:critical_all}
\end{figure}

%\newpage
\subsection{Comparison with EdgePool}
\label{app:comparison_EdgePool}

%\textbf{\textcolor{red}{Add EdgePool results}} \todo{EdgePool results?}

To extend our evaluation, we compare with the \texttt{pytorch\_geometric} implementation of EdgePool~\citep{diehl2019edge, diehl2019towards}. To do so, we implement the GNN architecture specified in \Cref{sec:classification} in \texttt{PyTorch} and repeat our main experiment reported in \Cref{classification}. We find that there is no significant difference in performance between our methods and EdgePool when computed on e.g. IMDB-M, where EdgePool reaches an accuracy of 47.6  ± 3.3, or on IMDB-B, where EdgePool reaches an accuracy of 70.9  ± 4.5. %We expect this same pattern to hold across further datasets. \todo{Add results.}
As further reported in \Cref{classification_edge_ablation}, similar patterns hold across datasets with our proposed pooling methods reaching comparable accuracies to EdgePool across datasets. 
As a major advantage, our methods have notably lower computational costs during training than EdgePool as reported in \Cref{app:empirical_efficiency} and permit choosing flexible pooling ratios. Our methods thus make edge-contraction pooling scalable to larger datasets and enable significantly faster GNN training. On IMDB-B for example, EdgePool takes 362.2 seconds and 666.6 MB for 200 epochs during training, but our method SpreadEdgePool only requires 62.7 seconds and 283.4 MB. 
We thus find that learning feature-based edge scores as done by EdgePool, is not necessary to ensure classification performance, but rather adds a notable computational burden.

\begin{table}[t]
  \caption{Mean and standard deviation of the graph classification accuracy of different pooling methods across datasets.}
  \label{classification_edge_ablation}
  \centering
  \resizebox{0.9\columnwidth}{!}{
  \begin{tabular}{lccccccccc}
    \toprule
    \textbf{Method} & \textbf{ENZYMES} & \textbf{PROTEINS} & \textbf{Mutagenicity} & \textbf{DHFR} & \textbf{IMDB-B} & \textbf{IMDB-M} & \textbf{NCI1} & \textbf{NCI109}  \\
    \midrule
    \textbf{No Pooling} & {\color[HTML]{747474} 87.3 ± 2.5} & {\color[HTML]{747474} 73.8 ± 0.8}& {\color[HTML]{747474} 80.1 ± 1.3} & {\color[HTML]{747474} 71.4 ± 1.9} & {\color[HTML]{747474} 69.7 ± 0.7} & {\color[HTML]{747474} 46.0 ± 0.7} & {\color[HTML]{747474} 76.5 ± 1.8} & {\color[HTML]{747474} 74.3 ± 2.0} \\
    \midrule
    \textbf{MagEdge} & 91.5 ± 3.2 & 76.4 ± 3.9 & 77.5 ± 2.7 & 88.0 ± 3.8 & 72.4 ± 1.7 & 47.4 ± 1.7 & 72.7 ± 2.4 & 73.0 ± 3.3 \\
    \textbf{SpreadEdge} & 92.8 ± 1.6 & 75.1 ± 3.1 & 76.0 ± 4.0 & 90.7 ± 3.8 & 71.8 ± 1.5 & 47.3 ± 1.7 & 73.4 ± 2.5 & 71.8 ± 1.8 \\
    %\midrule
    \textbf{EdgePool} & 92.1 ± 1.4 & %71.7 ± 4.1 & 70.4 ± 2.1 & 72.6 ± 3.5 
    75.8 ± 4.6 & 75.3 ± 3.1 & 83.2 ± 1.0
    & 70.9 ± 4.5 & 47.6 ± 3.3 & 71.9 ± 3.0 & 72.7 ± 2.6 \\%68.6 ± 2.6 
    \textbf{Random} & 88.1 ± 2.2 & 73.7 ± 1.1 &  73.1 ± 2.1 & 82.1 ± 3.3 & 70.1 ± 0.7 & 45.6 ± 0.8 & 71.1 ± 2.2 &  68.0 ± 3.0 
    \\
    \bottomrule
  \end{tabular}}
\end{table}

\subsection{Comparison with Randomised Pooling}
\label{app:random}

Research by \citet{mesquita2020rethinking} has demonstrated that randomised ablations of Graclus and DiffPool using random cluster assignment can reach competitive performance to established pooling methods on frequently used graph classification datasets, such as IMDB-B, PROTEINS, NCI109, DD and MOLHIV. Questioning standard assumptions of hierarchical pooling methods,  \citet{mesquita2020rethinking} thus question the utility of localised graph pooling.  
We find that this discussion and the observations by \citet{mesquita2020rethinking} are closely connected to the following challenges: 
\begin{itemize}[noitemsep, leftmargin=1em, topsep=0pt]
    \item Graph learning tasks do not necessarily need the graph structure to reach high performance~\citep{coupette2025no, speicher2024graph}. 
\item Evaluating the performance of pooling layers is dependent on the specific GNN architecture and benchmarking practices used \citep{errica2020fair, mesquita2020rethinking}. 
\item Randomised modifications of expressive pooling operators can retain the
expressivity of the underlying pooling operation 
\citep{bianchi2024expressive}. 
\end{itemize}
Hence, if all task-relevant information is already captured by the learnt features before pooling, randomised baselines can reach competitive performance. Nevertheless, destroying the structure of graphs during pooling can reduce the effectiveness of preceding message-passing layers \citep{bianchi2024expressive}. Our pooling approach aims to mitigate this risk by presenting
geometry-aware and expressive edge-pooling methods. 

As an ablation study, we add randomised edge pooling to our main experiment~(cf.\ \Cref{classification}) by randomly merging pairs of nodes into super-nodes to pre-compute the pooling assignment and averaging their features during training. %. \todo{Explain how it is randomised.}
We find that for NCI109 randomised pooling reaches an accuracy of 68.0 ± 3.0 as compared to 71.8 ± 1.8 for SpreadEdgePool. On NCI1 random pooling reaches an accuracy of 71.1 ± 2.2 compared to 73.4 ± 2.5 for SpreadEdgePool. Thus, the 
random baseline performs worse than any alternative pooling method on NCI109
and decreases the accuracy on NCI1 slightly. We hypothesise that this decrease is due to graph structure being relevant for these learning tasks \citep{coupette2025no} and our pooling method better preserves structural properties. 
In contrast, we observe fewer performance differences for other standard datasets, such as ENZYMES, PROTEINS, IMDB-B, or IMDB-M. This confirms the findings by \citet{coupette2025no} that on these datasets GNNs can reach competitive accuracies even when using completely randomised adjacencies as inputs. %Thus, 
%We are evaluating the performance of hierarchical pooling layers later on in the GNN architecture. 
%For these specific datasets, w
We thus find that an expressive aggregation of the node features during pooling can suffice~\citep{bianchi2024expressive} and preserving graph structure might not be necessary %as the input adjacencies are not needed 
for solving these tasks to begin with. %Based on the fundamental challenges listed above, these results thus point towards the need for beyond the scope of our study. 
%\todo{Add more datasets.}
%Nevertheless, %in cases structure preservation is beneficial for the task at hand, there are notable advantages of using our proposed pooling methods.

%\input{regression_table}

\subsection{Preserving Graph Structure}
\label{app:struct}
%Extended explanation and further results on magnitude / spectral distance experiment.
Extending on the results reported in \Cref{fig:pooling_ratio_structure_main} and \Cref{sec:structure}, we further report the distribution of structure preservation measures across pooling ratios. Specifically, \Cref{fig:pooling_ratio_structure} shows line plots that summarise the mean magnitude difference relative to the original graph and the normalised spectral distance between all original and pooled graphs from the NCI1 dataset in the leftmost column. The remaining plots illustrate the quantiles of the same measures split up per pooling method. Overall, these individual plots support our assessment that MagEdgePool and SpreadEdgePool consistently reach low magnitude differences and comparably low spectral distances, with these trends being more pronounced in terms of the relative difference in magnitude after pooling.
Further, we repeat the experiment reported in \Cref{fig:pooling_ratio_structure_main} for further datasets, specifically for DHFR, PROTEINS, and ENZYMES and summarise the results in \Cref{fig:pooling_ratio_structure_ablation}. We observe consistent trends across these examples that support the results reported for NCI1 in \Cref{fig:pooling_ratio_structure_main}. Specifically we find that our proposed pooling methods, MagEdgePool and SPreadEdgePool, reach the lowest magnitude differences across pooling ratios and datasets, and that this corresponds to low spectral distances between the pooled and original graphs.

\begin{figure}[tbh]
  \centering
  \includegraphics[width=0.22\linewidth]{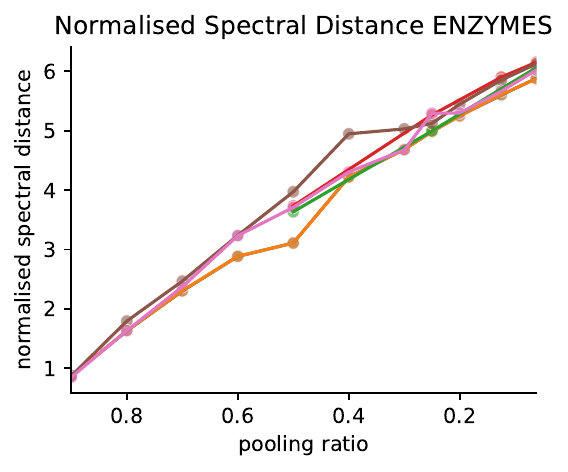} 
   \raisebox{6pt}{\includegraphics[width=.24\linewidth]{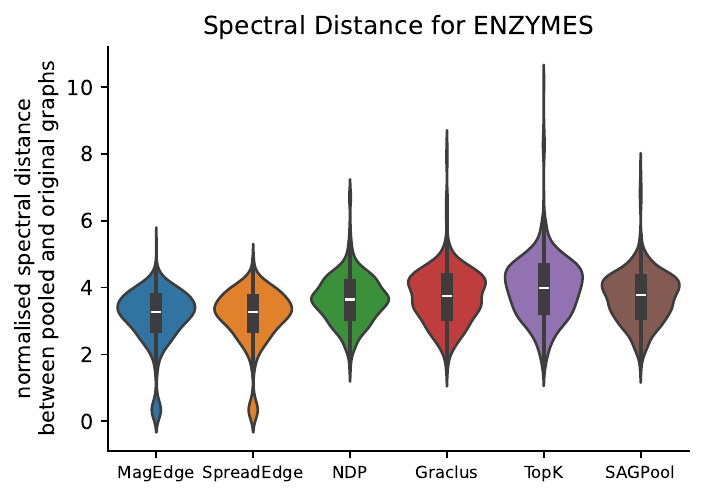}}
   \includegraphics[width=0.22\linewidth]{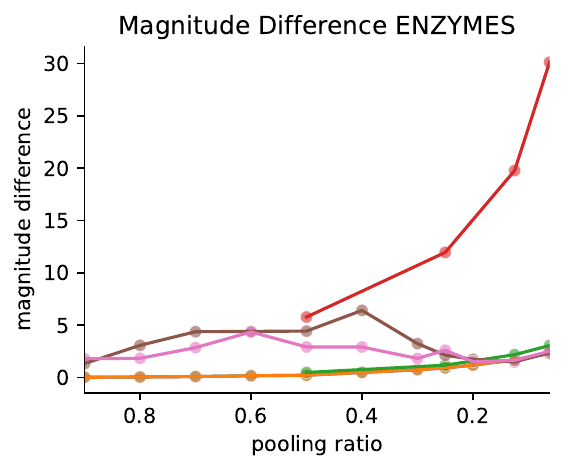}
   \raisebox{6pt}{\includegraphics[width=0.24\linewidth]{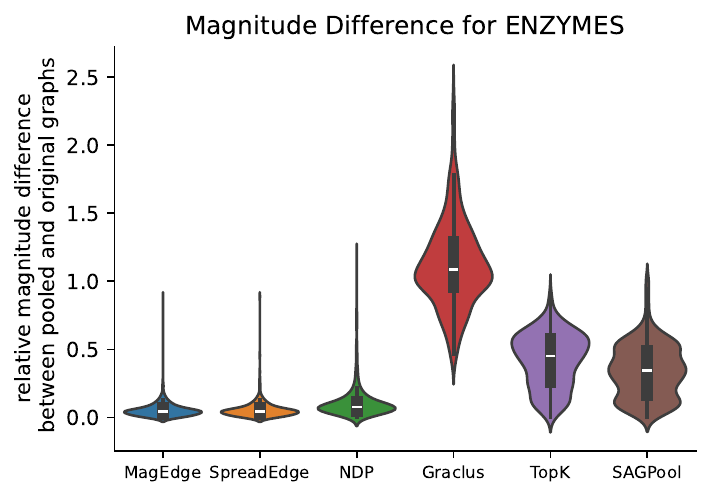}}
   
   \includegraphics[width=0.22\linewidth]{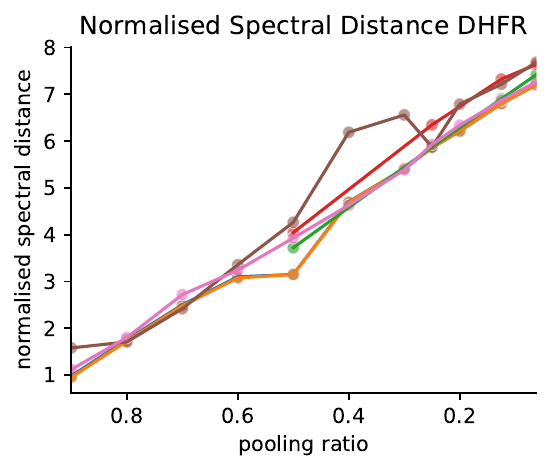} 
   \raisebox{6pt}{\includegraphics[width=.24\linewidth]{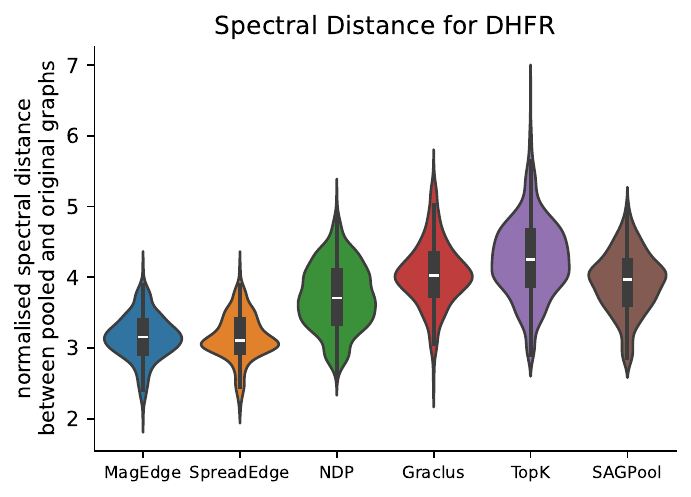}}
   \includegraphics[width=0.22\linewidth]{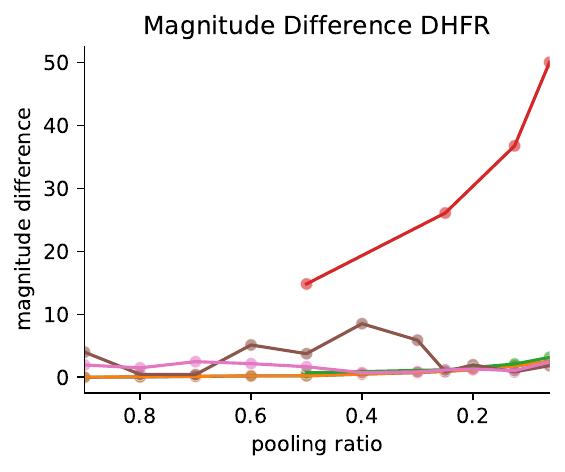}
   \raisebox{6pt}{\includegraphics[width=0.24\linewidth]{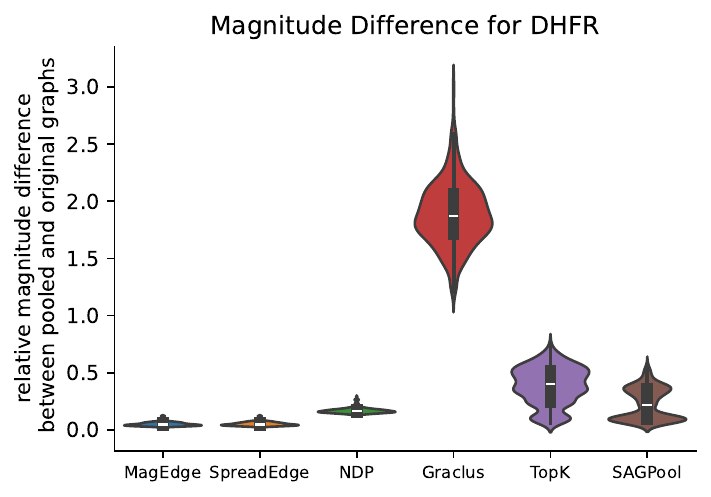}}

   \includegraphics[width=0.22\linewidth]{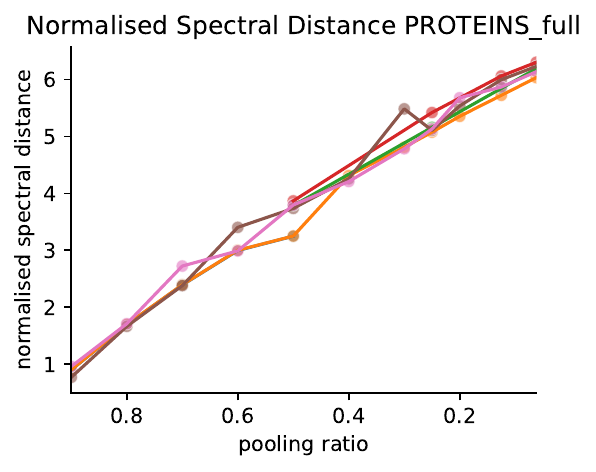} 
   \raisebox{6pt}{\includegraphics[width=.24\linewidth]{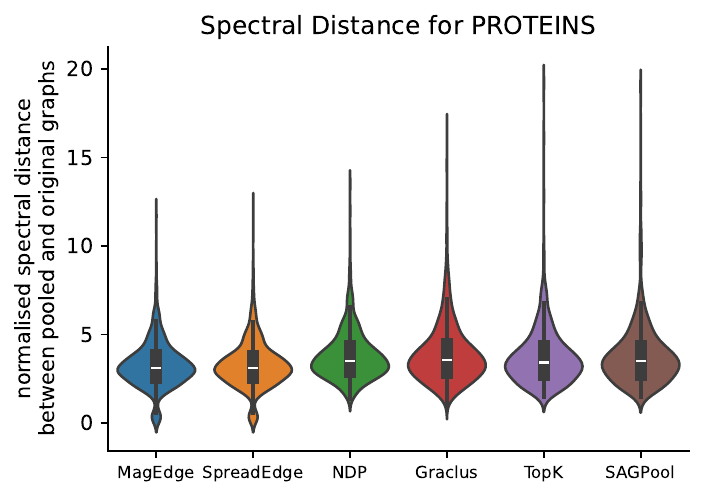}}
   \includegraphics[width=0.22\linewidth]{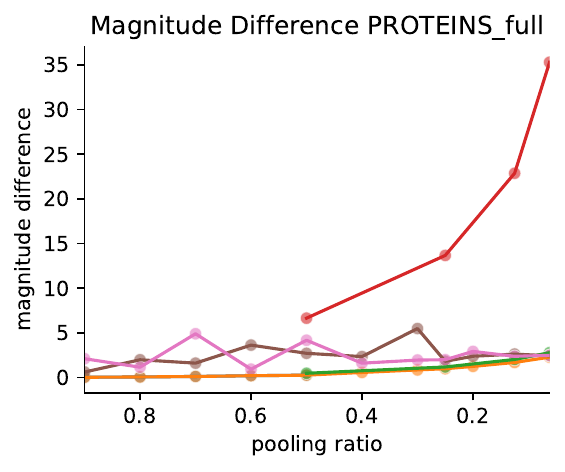}
   \raisebox{6pt}{\includegraphics[width=0.24\linewidth]{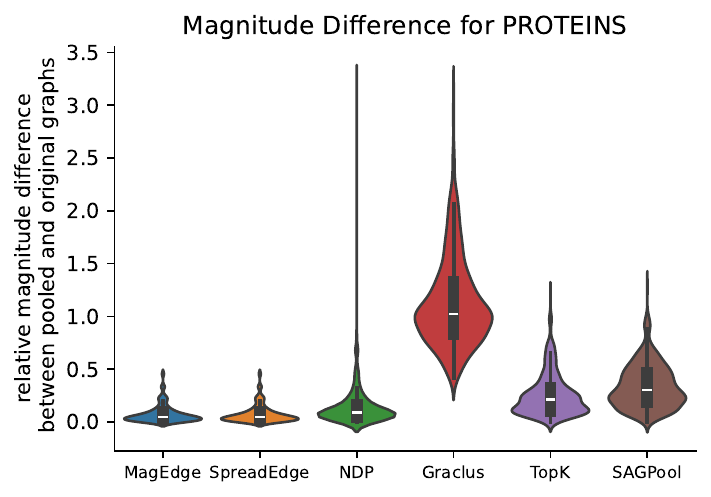}}

    \caption{Structure preservation for all graphs in the ENZYMES, DHFR, and PROTEINS datasets across pooling ratios. Left: The spectral distance between the normalised Laplacians of the original and the pooled graphs. Right: The relative difference in magnitude, summarising proportional differences in structural diversity after pooling. Violin plots show the variability %in these scores 
    across graphs at pooling ratio 0.5. %Bold lines show the mean values of each score across graphs and thin lines the 10\%, 25\%, 75\% and 90\% quantiles.
  }\label{fig:pooling_ratio_structure_ablation}
\end{figure}

\begin{figure}[tbp]
  \centering
  \includegraphics[width=0.9\linewidth]{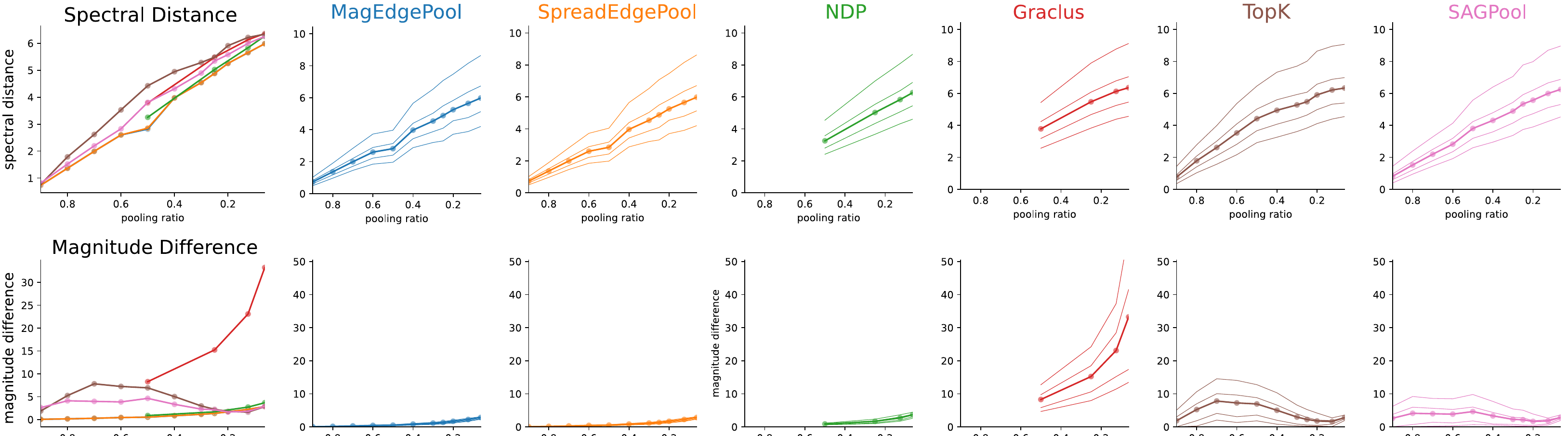}
  \caption{Structure preservation for the NCI1 dataset across pooling ratios. Top row: The spectral distance between the original and pooled graphs. Bottom row: The relative difference in magnitude. Bold lines show the mean values of each score across graphs and thin lines the 10\%, 25\%, 75\% and 90\% quantiles.}\label{fig:pooling_ratio_structure}
\end{figure}

\end{document}